\documentclass{article}

\usepackage{microtype}
\usepackage{graphicx}
\usepackage{subcaption}
\usepackage{booktabs} % for professional tables

\usepackage{hyperref}

\usepackage[accepted]{icml2026}

\usepackage{amsmath}
\usepackage{amssymb}
\usepackage{mathtools}
\usepackage{amsthm}

\usepackage[capitalize,noabbrev]{cleveref}

\usepackage{wrapfig} 

\definecolor{OrangeRed}{RGB}{255, 69, 0}

\usepackage{multirow}
\usepackage[table]{xcolor}
\usepackage{colortbl}
\usepackage{enumitem}
\usepackage{etoc}

\theoremstyle{plain}
\newtheorem{theorem}{Theorem}[section]
\newtheorem{proposition}[theorem]{Proposition}
\newtheorem{lemma}[theorem]{Lemma}
\newtheorem{corollary}[theorem]{Corollary}
\theoremstyle{definition}

\newtheorem{assumption}[theorem]{Assumption}

\newtheorem{innercustomprop}{Proposition}
\newenvironment{customprop}[1]
{\renewcommand\theinnercustomprop{#1}\innercustomprop}{\endinnercustomprop}

\usepackage[textsize=tiny]{todonotes}

\icmltitlerunning{RSPO: Regularized Self-Play Alignment of Large Language Models}

\begin{document}

\twocolumn[
  \icmltitle{RSPO: Regularized Self-Play Alignment of Large Language Models}



  \icmlsetsymbol{equal}{*}

  \begin{icmlauthorlist}
    \icmlauthor{Xiaohang Tang}{equal,1}
    \icmlauthor{Sangwoong Yoon}{equal,2}
    \icmlauthor{Seongho Son}{1}
    \icmlauthor{Huizhuo Yuan}{3}
    \icmlauthor{Quanquan Gu}{3}
    \icmlauthor{Ilija Bogunovic}{1,4}
  \end{icmlauthorlist}

  \icmlaffiliation{1}{University College London}
  \icmlaffiliation{2}{Ulsan National Institute of Science \& Technology}
  \icmlaffiliation{3}{University of California, Los Angeles}
  \icmlaffiliation{4}{University of Basel}

  \icmlcorrespondingauthor{Ilija Bogunovic}{i.bogunovic@ucl.ac.uk}

  \icmlkeywords{Machine Learning, ICML}

  \vskip 0.3in
]



\printAffiliationsAndNotice{\icmlEqualContribution}

\begin{abstract}
Self-play-based policy optimization has emerged as an effective approach for fine-tuning large language models (LLMs), formulating preference optimization as a two-player game. However, the regularization with respect to the reference policy, which is crucial for mitigating over-optimization, has been insufficiently investigated in self-play alignment. To study the impact of different regularization strategies, we propose \textbf{Regularized Self-Play Policy Optimization (RSPO)}, a novel framework that unifies prior methods and enables simple plug-and-play regularizers, meanwhile preserving convergence to Nash equilibrium of the corresponding regularized game. 
We empirically show that RSPO with appropriate regularizers can substantially improve the length-controlled win rate (LCWR) on AlpacaEval-2 across a range of base models, while also achieving consistently superior performance on Arena-Hard, MT-Bench, ArmoRM, and response diversity. In particular, RSPO improves unregularized self-play baseline (SPPO) on AlpacaEval-2 LCWR from $28.5\%$ to $ 35.4\%$ with base model Mistral-7B, from $38.77\%$ to $43.66\%$ with LLaMA-8B, and from $50.54\%$ to $51.83\%$ with Gemma-2B.
Combining simplicity, convergence guarantees, and significant empirical gains, RSPO offers a strong foundation for exploring regularized self-play in alignment. Code is available at \url{https://github.com/xiaohangt/RSPO}
\end{abstract}

\section{Introduction}

Self-play is a line of work conducting iterative self-competition of models, which has been demonstrated as an effective approach for improving AI systems \citep{goodfellow2020generative,wang2022diffusion}, particularly in strategic decision-making problems  \citep{silver2016mastering,heinrich2016deep,pinto2017robust,brown2018superhuman}. In the human alignment of LLMs, self-play recently started to be used and has shown superior empirical performance than other iterative Reinforcement Learning from Human Feedback (RLHF) methods on popular benchmarks \citep{dubois2024length,jiang2024textual,wu2024self,rosset2024direct}. By formulating the preference optimization problem as a two-player game, self-play alignment methods seek to identify a \emph{Nash Equilibrium} (NE) of the game in which utility is determined by a general preference model \citep{azar2024general, munos2023nash,calandriello2024human}. This NE is regarded as the most aligned LLM policy achieved without Bradley-Terry (BT) reward modeling \citep{david1963method}, which has shown under-performance compared to general preference modeling \citep{ye2024online}. 
% independent on the Bradley-Terry (BT) assumption \citep{david1963method}  \looseness=-1 
% An NE LLM policy is regarded as the most aligned policy, as it is expected to generate responses that are, on average, more preferred than those produced by an RLHF solution. This advantage arises because the general preference modeling does not depend on the Bradley-Terry (BT) assumption \citep{david1963method}. 
% Due to such stronger expressiveness of a general preference model, LLM can be aligned with more complicated human preference 

% Recent methods abandon the Bradley-Terry reward model to avoid the . They instead formulate the preference-based RLHF as a two-player game theory problem , where the preference probability is the utility and the optimal solution as Nash Equilibrium policy \citep{azar2024general,munos2023nash,calandriello2024human,wu2024self}. These methods iteratively update the policy directly according to the preference signals, and treat last-iterate policy as the opponent to seek for self-improvement. We refer to these methods as self-play.

\begin{figure}[t!]
    \centering
    \includegraphics[width=\linewidth]{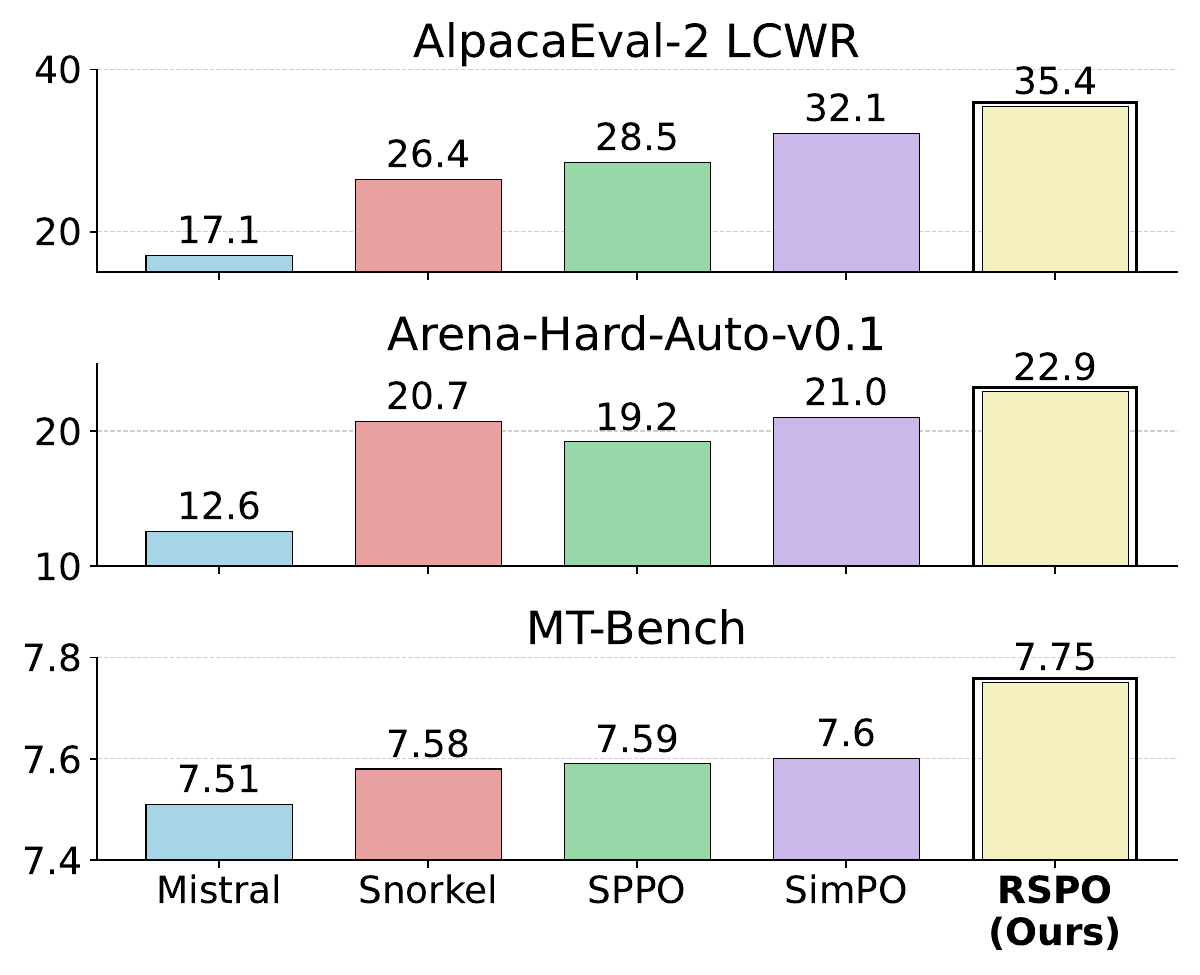}
    \caption{Performance of Mistral-7B-Instruct fine-tuned with methods including iterative DPO (Snorkel) \citep{viethoangtranduong}, SPPO \citep{wu2024self}, SimPO \citep{meng2024simpo}, and our method. For iterative methods, we compare them at iteration 3.}
    \vspace{-1em}
    \label{fig:fig1}
    \vspace{-1em}
\end{figure}

\begin{figure*}[t!]
    \centering    \includegraphics[width=\linewidth]{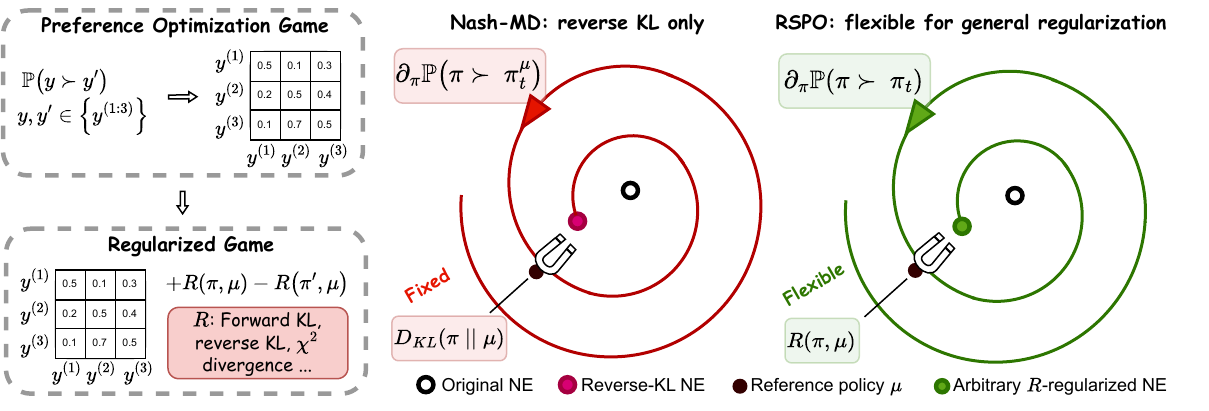}
    \caption{\textbf{RSPO is flexible for general regularization.} The estimation of Nash-MD policy update direction  $\partial_{\pi} \mathbb{P}(\pi\succ  \pi_t^\mu)$ requires samples from geometric mixture policy $\pi_t^{\mu}$. Such update approach is only compatible with reverse KL divergence for regularization.}
    \label{fig:rspo_overview}
    % \vspace{-1em}
\end{figure*}

Despite the significant empirical improvements achieved through self-play, the impact of regularization to the reference policy—commonly used in RLHF to mitigate over-optimization—has received insufficient investigation in self-play alignment. Most existing self-play methods completely lack explicit regularization \citep{wu2024self,rosset2024direct,swamy2024minimaximalist,wang2024magnetic,gao2024rebel}. In practice, unregularized self-play is susceptible to over-optimization, particularly when the preference model is inaccurate or misspecified. Although a few recent self-play approaches like Nash-MD \citep{munos2023nash} incorporate reverse KL divergence as a regularization penalty  \citep{calandriello2024human,wang2024magnetic,zhang2024iterativenashpolicyoptimization}, it remains unclear whether reverse KL is optimal for alignment, and the broader impact of alternative regularization strategies in self-play remains insufficiently explored. Moreover, the extension of current approaches to general forms of regularization is challenging, as their training protocols are intrinsically reliant on the reverse KL divergence for regularization \cite{munos2023nash} (see Figure \ref{fig:rspo_overview}). \looseness=-1

% \textbf{Related Work.} 

In this work, we introduce a novel framework to flexibly incorporate diverse regularization methods into self-play alignment, termed \textbf{Regularized Self-Play Policy Optimization (RSPO)}:
\begin{itemize}[left=0pt]
    \item RSPO offers a simple way to apply general regularization strategies in self-play by \textbf{directly adding} the regularization term to our proposed unified self-play loss function, while maintaining \textbf{last-iterate convergence} to NE of the corresponding regularized preference optimization game. Unlike Nash-MD, which requires a specialized sampling process limited to Reverse-KL regularization (\Cref{fig:rspo_overview}), our method follows the standard sampling procedure in RLHF, making it both simpler and more general. \looseness=-1
     \item RSPO with tuned regularizers demonstrates substantial improvements over the unregularized self-play alignment method (SPPO~\citep{wu2024self}). 
    In particular, it increases the length-controlled win rate on AlpacaEval-2.0 \textbf{from $28.5\%$ to $35.4\%$ with Mistral-7B, from $38.77\%$ to $43.66\%$ with LLaMA-8B, and from $50.54\%$ to $51.83\%$ with Gemma-2B}. 
    RSPO also achieves consistently superior performance on other benchmarks, including Arena-Hard-v0.1, MT-Bench, self-BLEU diversity~\citep{zhu2018texygen}, and ArmoRM across multiple reward dimensions such as instruction following, truthfulness, honesty, and helpfulness. \looseness=-1
    \item Empirical analysis reveals distinct effects of different regularizations. On both Mistral-7B-Instrct and LLaMA-8B-Instruct stronger forward KL regularization reduces the response length, whereas reverse KL regularization significantly improves the raw win rate. Mistral-7B-Instrct with combined forward and reverse KL regularization achieves the most improvement. In addition, RSPO also demonstrate parameter-efficiency when comparing with SPPO trained with stronger preference model, indicating the comprehensive effectiveness of our method in self-play alignment.
\end{itemize}
\vspace{-1em}

% \paragraph{Online Learning and Game Theory.}
% Regularization in online learning has been extensively studied, with most works utilizing regularization to promote sparsity in the solution. However, in alignment and RLHF, the approach to regularization and its effects differ. The Multiplicative Weights Update (MWU) method employs a first-order technique, specifically exponentiated gradient descent \citep{beck2003mirror}. Most studies addressing regularized games \citep{zeng2022regularized,liu2022power} rely on gradient-based approaches, such as Mirror Descent.

\section{Preliminaries}

We denote a prompt as $x$, a response as $y$, and a LLM policy as $\pi(y|x)$, where $\pi(\cdot| x) \in  \Delta_{\mathcal{Y}}$. We denote the set of all prompts as $\mathcal{X}$, and the set of all responses as $\mathcal{Y}=\{y^0,y^1,\cdots \}$. We use $\Delta_{\mathcal{Y}}$ to denote the probability simplex over the responses given a specific prompt. We parametrize the LLM policy $\pi$ as $\pi_\theta$. The reference policy is an LLM denoted as $\mu \in  \Delta^{\mathcal{X}}_{\mathcal{Y}}$. For notational brevity, we remove the dependence of policy $\pi$ and loss functions on the prompt $x$ throughout the paper.

%\subsection{Preference Optimization as Solving Two-Player Constant-Sum Games}
\subsection{Game-Theoretic Preference Optimization}

We study the preference optimization problem in an online setting by formulating it as a two-player max-min game, as studied in previous self-play works \citep{wu2024self}. The players are two LLMs whose strategies are LLM policies, denoted as max-player $\pi$ and min-player $\pi'$. The utility of the max-player is expressed as the preference of itself over the min-player:
\begin{align}
u(\pi; \pi') = \mathbb{P}(\pi \succ \pi') \overset{\text{def}}{=} \mathbb{E}_{y\sim \pi, y' \sim \pi'}[\mathbb{P}(y \succ y')],
\label{eq:utility}
\end{align}
where $u: \Delta^{\mathcal{X}}_{\mathcal{Y}} \times \Delta^{\mathcal{X}}_{\mathcal{Y}} \rightarrow \mathbb{R}$ is \textit{linear} in $\pi$ and $\pi'$; $\mathbb{P}: \mathcal{X} \times \mathcal{Y} \times \mathcal{Y} \rightarrow [0, 1]$ is a general preference model that quantifies the preference of $y$ over $y'$ given a prompt. We extend the notation $\mathbb{P}(y \succ \pi')=\mathbb{E}_{y' \sim \pi'}[\mathbb{P}(y \succ y')]$. The objective is finding a \textit{NE} policy $\pi^*$ of the preference model:
\begin{align}
(\pi^*,\pi^*) = \arg \max_{\pi} \min_{\pi'} \mathbb{P}(\pi \succ \pi').
\label{eq:pm}
\end{align} 
Therefore, an NE strategy $\pi^*$ is an LLM that can generate \textit{the most preferred responses in expectation}, thus achieving human alignment based on the preference model. Most existing self-play alignment methods aim to solve this NE following Algorithm \ref{alg:selfplay} \citep{wu2024self,rosset2024direct,swamy2024minimaximalist,wang2024magnetic}.

% To unify the notations in different methods, we use game-theoretic terminologies.  In preference optimization, utility $U$ can be either unregularized or regularized preference. The objective is to find the Nash Equilibrium policies such that $\max_\pi \min_{\pi'} u(\pi, \pi')$.

% \subsubsection{Multiplicative Weights Update}
\subsection{Preference Optimization via Multiplicative Weights Update}
An effective self-play method to solve the preference optimization game in \Cref{eq:pm} is Self-Play Policy Optimization (SPPO) \citep{wu2024self}. SPPO derives its loss function from the iterative no-regret learning algorithm, Multiplicative Weights Update (MWU) \citep{freund1997decision}. Specifically in a game setting, denote learning rate as $\eta$, and normalization constant $Z(\pi_t)$. In any iteration $t$, the policy update $\forall y \in \mathcal{Y}$ is
\(
\pi_{t+1}(y) = \pi_t(y) \cdot {\exp \big( \eta \mathbb{E}_{y' \sim \pi_t}[u({y}; y')]\big)}/{Z(\pi_t)},
\)
where $u(y;y')$ is the utility function defined in Equation \eqref{eq:utility}, with $y$ treated as a pure strategy.

The practical loss function of SPPO for policy update is then derived according to MWU $\mathcal{L}_{\text{SPPO}}(\theta)=$:
\begin{align}
\mathbb{E}_{y \sim \pi_t} \Big[ \log \frac{\pi_{\theta}(y)}{\pi_t(y)} -  \Big(  \eta \mathbb{P}({y} \succ \pi_t) - \log Z(\pi_t) \Big) \Big]^2.
\label{eq:sppo_mse}
\end{align}
% \vspace{-.5cm}

SPPO converges to the NE of the preference optimization game in \Cref{eq:pm}. However, after running multiple iterations, the deviation of the policy $\pi_\theta$ from $\mu$ can be large. Such deviation is particularly problematic when the preference model is only accurate at evaluating responses sampled from the reference policy \citep{munos2023nash}. Furthermore, in aligning LLMs in practice, the preference model is typically a surrogate $\hat{\mathbb{P}}$, such as PairRM \citep{llm-blender-2023}, which may be misspecified at some out-of-distribution responses and inaccurate due to estimation error or limited model expressiveness (PairRM is only a 0.4B model), causing over-optimization problem. Regularizing the policy optimization to a reference SFT model, which is typically trained on high-quality data \citep{ouyang2022training}, can mitigate the problem. We provide a synthetic example in Appendix \ref{append:reg_game} to demonstrate this problem. \looseness=-1

% since the reference policy $\mu$ in preference optimization is a powerful SFT model trained on high-quality data \citep{rafailov2024direct,llama3modelcard,jiang2023mistral,ouyang2022training}, it can provide a guidance for policy optimization. Although the self-play algorithms designed to find Nash Equilibrium in \Cref{eq:pm} are effective in practical preference optimization, , if only initializing the policy with $\mu$ as in Algorithm \ref{alg:selfplay}. Such deviations are particularly problematic when the preference model is only accurate at evaluating responses sampled from the reference policy \citep{munos2023nash}. Furthermore, in aligning LLMs in practice,  the preference model is typically a surrogate $\hat{\mathbb{P}}$, such as PairRM \citep{llm-blender-2023}, which may be misspecified at some out-of-distribution responses and inaccurate due to estimation error or limited model expressiveness (e.g., PairRM is only a 0.4B model). \looseness=-1

\subsection{Regularized Preference Optimization Game with Reference Policy}
\label{sec:reg_game}

To address the regularization in self-play, we adopt the objective in Nash Learning from Human Feedback \citep{munos2023nash}, and extend the KL divergence regularization to a general regularization function, to penalize the deviation from the reference policy. We define a \textit{convex} regularization function $R: \Delta^{\mathcal{X}}_{\mathcal{Y}} \times \Delta^{\mathcal{X}}_{\mathcal{Y}} \rightarrow (-\infty, \infty)$, where $R(\pi, \mu)$ measures the distance between $\pi$ and the reference model $\mu$, such as KL divergence $D_{\text{KL}}(\pi\|\mu)$. 
% A regularized preference model is $\mathbb{P}(\pi \succ \pi') - \tau R(\pi, \mu) + \tau R(\pi', \mu)$,
Denote regularization temperature as $\tau$, the objective becomes to optimize a \textit{regularized preference model} by solving the NE $(\pi^*,\pi^*)$ of the \textit{regularized} game, where the utility of max player is still $u(\pi;\pi')=\mathbb{P}(\pi \succ \pi')$: \looseness=-1
\begin{align}
\arg \max_{\pi} \min_{\pi'} \mathbb{P}(\pi \succ \pi') - \tau R(\pi, \mu) + \tau R(\pi', \mu).
\label{eq:rpm}
\end{align}
We provide proof of the existence and uniqueness of this NE in Appendix \ref{append:rpm_exists}. A few recent methods leverage Mirror Descent (MD), which is also in a self-play manner, to find a regularized NE in \Cref{eq:rpm} with last-iterate policy \citep{munos2023nash,calandriello2024human,zhang2024iterativenashpolicyoptimization}.
% Mirror Descent (MD) is also in a self-play manner to find a regularized NE in \Cref{eq:rpm}. 

However, these MD-based methods are only compatible with the reverse KL divergence regularizer, and are non-trivial to extend to general divergence. For instance, Nash-MD\footnote{Throughout the paper, regularization specifically refers to the deviation of $\pi$ from $\mu$, rather than from $\pi_t$.} addresses the reverse KL regularization of $\pi$ and $\mu$ requiring responses generated from a geometric mixture policy $\pi_t^\mu(y) \propto 
{\pi_t(y)^{1 - \eta \tau} \mu(y)^{\eta \tau}}$ \citep{munos2023nash}, which is inherently compatible only with reverse KL divergence: \looseness=-1
\begin{align}
\pi_{t+1}=\arg\min_{\pi}  -\eta  \mathbb{E}_{\pi}[\nabla_\pi u(\pi_t; \textcolor{red} {\pi^\mu_t})] + D_{\text{KL}}(\pi \| \pi_t^\mu).
% \pi_{t+1}=\arg\min_{\pi}  -\eta  \langle \pi, \partial_\pi u(\pi_t; {\pi^\mu_t}) \rangle + D_{\text{KL}}(\pi, \pi_t^\mu).
\label{eq:nash-md}
\end{align}
% \todoq{can you define $\pi_t^\mu$}
Therefore, while the LLMs optimized via existing self-play methods exhibit empirical improvement, they all have limited regularization of $\pi$ and $\mu$. The potential benefits of alternative regularization, such as adopting other $f$-divergences than reverse KL, remain unexplored.

\section{Regularized Self-Play Policy Optimization}
\label{sec:sp}
We propose a framework for regularized self-play alignment, namely \textbf{Regularized Self-Play Policy Optimization (RSPO)}. RSPO is simple and flexible for regularization, and provably convergent to Nash Equilibrium. The loss function of RSPO $\mathcal{L}_{\text{RSPO}}$ is defined as the sum of a mean-squared self-play loss and a weighted regularization term:
\begin{align}
\mathcal{L}_{\text{RSPO}}\overset{\text{def}}{=} &\mathbb{E}_{y \sim \pi_t} \Big[ \log \frac{\pi_{\theta}(y)}{\pi_t(y)} - \eta \Big(   G({y}, \pi_t, \mu) - B(\pi_t, \mu) \Big) \Big]^2 \nonumber \\ & \colorbox{cyan!20}{$+ \lambda 
 R(\pi_{\theta}, \mu)$},
\label{eq:RSPO}
\end{align}
where $G({y}, \pi_t, \mu)$, $B(\pi_t, \mu)$, and $R(\pi_{\theta}, \mu)$ are configurable components.
First, $G:\mathcal{Y} \times \Delta^{\mathcal{X}}_{\mathcal{Y}} \times \Delta^{\mathcal{X}}_{\mathcal{Y}} \rightarrow (-\infty, \infty)$ defines the \textit{update direction} of $\pi_{\theta}$, which can be set as the gradient of a utility function to guide the policy update towards increasing the utility. Second, the \textit{baseline} function $B :\Delta^{\mathcal{X}}_{\mathcal{Y}} \times \Delta^{\mathcal{X}}_{\mathcal{Y}} \rightarrow (-\infty, \infty)$ is for variance-reduction of $G$, similar to the baseline in REINFORCE \citep{williams1992simple}. 
Lastly, $R: \Delta^{\mathcal{X}}_{\mathcal{Y}} \times \Delta^{\mathcal{X}}_{\mathcal{Y}} \rightarrow \mathbb{R}$ is the regularization function. The coefficient $\lambda$ is the regularization temperature. 
The first Mean Square Error term in \Cref{eq:RSPO} can be interpreted as a self-play loss of conducting exponentiated gradient descent \citep{beck2003mirror}.
% \todoq{why changing $\tau$ to be $\lambda$ in Eq.(13  ` )?}

% where $\mathcal{L}_{\text{SP}} (\theta; G, B)$ is defined as
% \begin{align}
% . 
% \label{eq:sp}
% \end{align}

RSPO is a modular framework offering a simple way to introduce regularization into self-play alignment with \textit{only an additional term in the loss}. RSPO offers the simplicity and flexibility to incorporate \textit{various} regularization methods into self-play-based preference optimization methods. Additionally, we show in Section \ref{sec:general} that RSPO can generalize existing unregularized self-play methods without external regularization $R$. Thus, regularizing existing methods requires \textit{no change} to their original loss functions or hyperparameters, but simply adding an external plug-and-play regularization to their loss function and tuning the regularization temperature $\lambda$. 

% We provide the theoretical guarantee for RSPO in Section \ref{sec:theory} by showing that RSPO is the RL implementation of a specific type of Mirror Descent. Given a specific regularizer $R$, RSPO has a last-iterate convergence to the NE of the corresponding \textit{regularized} game. We also introduce the implementation details of RSPO specifically for preference optimization (\Cref{eq:rpm}) in Section \ref{sec:implementation}.

In practice, we set baseline function $B=\tfrac{1}{2}$ following Nash-MD and SPPO, and the update direction $G$ to be the gradient of the preference against $\pi_t$, $\forall y \in \mathcal{Y}$:
$G({y}, \pi_t, \mu)=\partial_{\pi(y)} \mathbb{P}(\pi \succ \pi_t) = \mathbb{P}(y \succ \pi_t)$.
We execute Algorithm \ref{alg:selfplay} by applying the following RSPO loss with any regularization $R$ of interests:
\begin{align}
&\mathcal{L}_{\text{RSPO}}\big(\theta; G=\mathbb{P}(y \succ \pi_t), B=\tfrac{1}{2}, R \big) .
\label{eq:RSPO_gmmd}
\end{align}
In theory, $B$ helps minimize the variance of $G$ the most when $B= \mathbb{E}_{y\sim \pi_t}[G(y, \pi_t, \mu)]$. But in preference optimization, due to the typically small minibatch size, the estimation error of the mean of $G$ could be large, leading to additional estimation error of the loss. Thus, we also set the baseline value for variance reduction to be a constant $\tfrac{1}{2}$, the mean value of $G$ when the algorithm converged. For the implementation of the different divergences $R$ for regularization, please refer to Appendix \ref{sec:implement_reg}.

In the following sections, we first illustrate the generalizable formulation of RSPO, so that it can be implemented without modifying the existing self-play component. We then establish theoretical convergence guarantees for RSPO grounded in Mirror Descent theory.

\subsection{Generalizing Existing Self-Play Methods}
\label{sec:general}

% MD-based methods are powerful in online preference optimization, and demonstrate monotone improvement with iterations.

% Both MWU-based and MD-based of preference optimization methods perform well in practice \citep{wu2024self,zhang2024iterativenashpolicyoptimization}. But the way of conducting policy update in these methods differ. In existing works, understanding the connection between both optimization methods has been lacking, which prevent us from understanding  similarities that make these methods successful or further improvement.  \looseness=-1

In this section, we show how RSPO generalize existing self-play methods, which showcase (1) implementing RSPO requires only one additional term to existing self-play loss functions; (2) the limitation of existing regularized methods.
% Detailed derivations of some of the following equivalence between existing methods and $\mathcal{L}_{\text{SP}}$ are provided in Appendix \ref{append:nashmd_proof}.
First, the unregularized self-play method SPPO \citep{wu2024self} has a loss function defined in \Cref{eq:sppo_mse} equivalent to RSPO \textit{without external regularization}:
% , where the update direction is $G(y, \pi_t, \mu)=\mathbb{P}(y \succ \pi_t)$, and the baseline is $B(\pi_t, \mu)=\log Z(\pi_t)$. In the practical version of SPPO, for simplification, the baseline is set to $\tfrac{1}{2}$ rather than estimating $\log Z(\pi_t)$. Formally,
\begin{align}
\mathcal{L}_{\text{SPPO}}(\theta) 
= \mathcal{L}_{\text{RSPO}}\Big(\theta; G=\mathbb{P}(y \succ \pi_t), B=\frac{1}{2},
\colorbox{red!20}{$R=0$} \Big).
\label{eq:sppo_sp}
\end{align}
According to \Cref{eq:RSPO_gmmd} and \Cref{eq:sppo_sp}, $\mathcal{L}_{\text{RSPO}}= \mathcal{L}_{\text{SPPO}} + \lambda R(\pi_{\theta}, \mu)$, i.e. the implementation of RSPO is equivalent to directly add the regularization $R$ to the loss function of SPPO (\Cref{eq:sppo_mse}). This implies that the additional regularization term becomes plug-and-play, requiring minimal changes to existing training pipeline.

% Other unregularized self-play methods following the preference-based MWU can also be generalized by $\mathcal{L}_{\text{SP}}$, and thus can be regularized by simply adding regularization term to the loss functions. Based on the same exponential update rule as in SPPO, SPO \citep{swamy2024minimaximalist} is equivalent to updating policy with the loss in \Cref{eq:sppo_sp}. Magnetic Policy Optimization \citep{wang2024magnetic}, despite incorporating regularization in the policy update, periodically updates $\mu=\pi_t$. Consequently, it inherently follows MWU while incorporating multiple policy updates within each iteration, following \citet{tomar2020mirror}.

In addition, we also show that existing regularized methods can also be generalized by $\mathcal{L}_{\text{RSPO}}$ under on-policy assumption (detailed derivations in Appendix \ref{append:nashmd_proof}): 
% By setting external regularization in RSPO $R(\pi_\theta,\mu)=D_{\text{KL}}(\pi_\theta\|\mu)$: 
% According to the \Cref{eq:nash-md}, Nash-MD is generalized by RSPO with the same external regularization $R=D_{\text{KL}}$:
\begin{align}
&\nabla_\theta \mathcal{L}_{\text{Nash-MD}}(\theta) \nonumber  \\
=& \nabla_\theta \mathcal{L}_{\text{RSPO}}\big(\theta;G=\mathbb{P}(y\succ  \pi_t^\mu), B=\frac{1}{2}, R=D_{\text{KL}}(\pi_\theta\|\mu) \big) \label{eq:rspo_nashmd_gen1} \\
=& \nabla_\theta \mathcal{L}_{\text{RSPO}}\Big(\theta;G=\mathbb{P}_\tau(y\succ  \pi_t^\mu), B=\frac{1}{2},
R=0 \Big) \label{eq:rspo_nashmd_gen2}.
%&\nabla_\theta \mathcal{L}_{\text{MPO}}(\theta) = \nabla_\theta \mathcal{L}_{\text{RSPO}}\big(\theta; G=\mathbb{P}(y\succ  \mu), \nonumber \\
% & \quad \quad \quad \quad  \quad \quad \quad \quad \quad \quad \quad B=\mathbb{P}(\pi_\theta \succ  \mu) \big).
\end{align}
% Here in MPO, advantage function $G-B$ is used for update, where $G$ can be understood as $Q$ function in RL and $B$ is its optimal baseline---value function. 
% \vspace{-.5cm}
% We provide the .  
where $\mathbb{P}_\tau(y\succ  \pi_t^\mu) = \mathbb{P}(y\succ  \pi_t^\mu) - \tau \log \frac{\pi_t(y)}{\mu(y)}$. \Cref{eq:rspo_nashmd_gen1} verifies our summarization shown in Figure \ref{fig:rspo_overview}. Nash-MD \citep[Lemma~2]{munos2023nash} follows a policy update rule specifically designed for reverse KL regularization (\Cref{eq:nash-md}).
% , as other $R$ can not be merged with $D_{\text{KL}}(\pi \| \mu)$ to a regularization w.r.t. geometric mixture $\pi^\mu_t$. % essentially arises from the well-established connection between the policy gradient and the gradient of a quadratic loss function, as explored in previous works \citep{,wu2024self,malkin2022gflownets}. 
In contrast, RSPO can incorporate with general divergence. 

\subsection{Theoretical Guarantees}
\label{sec:theory}

In this section, we examine the theoretical properties of RSPO, with a particular emphasis on its convergence guarantee under the on-policy assumption. We adopt Mirror Descent (MD) as the foundational framework, given its well-established last-iterate convergence to the NE.

We build upon Magnetic Mirror Descent (MMD) \citep{sokota2022unified}, a specialized variant of MD that guarantees convergence to a reverse-KL-regularized NE. To generalize beyond reverse-KL regularization, we introduce Generalized Magnetic Mirror Descent (GMMD), which can accommodate a broader class of regularization techniques. By demonstrating that optimizing the RSPO loss is equivalent to performing reinforcement learning (RL) within the GMMD framework, we establish a formal connection between RSPO and GMMD. This connection ensures the last-iterate convergence of RSPO to the NE of the corresponding \textit{regularized} game.

\textbf{Tabular GMMD.} Denote the utility function of the game as $U$, define $G$ as the element of the vector of partial derivatives of $U$ w.r.t. policy: $G(y;\pi') \overset{\text{def}}{=} \partial_{\pi(y)} U(\pi;\pi')$ for any $y \in \mathcal{Y}$.
% \begin{align}
% \partial_\pi U(\pi;\pi') = (G(y^0;\pi'),\cdots, G(y^{|\mathcal{Y}|};\pi') )^\top \in \mathbb{R}^{|\mathcal{Y}|}
% \end{align}
Then in iteration $t$, GMMD updates the policy as \looseness=-1
\begin{align}
\pi_{t+1} 
= \arg \min_{\pi} -\eta \mathbb{E}_{\pi}[G(y;\pi_t)] + B_{\psi}(\pi; \pi_t) + \tau R(\pi, \mu),
\label{eq:gmmd}
\end{align}
where $\tau$ is a regularization temperature, $R$ is a general regularization function, serving as a ``magnet'' to attract $\pi$ to $\mu$ during policy updating. $B_{\psi}$ is the Bregman Divergence generated by a convex potential function $\psi$ \citep{bregman1967relaxation}.

Notably, the vanilla Magnetic Mirror Descent limits $R$ to be the same regularization method of $\pi$ and $\pi_t$, i.e., $R=B_{\psi}$ \citep[Section~3.2]{sokota2022unified}; whereas in this paper we aim at a general regularizer of $\pi$ and $\mu$, which could be different from $B_{\psi}$, and study the effects of different regularization methods.

% \todoq{in the following proposition, why the strong convexity is with respect to $\psi$?}
\begin{proposition}[\textbf{Last-iterate Convergence}] If $R(\cdot, \mu)$ is $1$-strongly convex relative to $\psi$, $\eta \leq \tau$, and $U$ is linear, then policy updated by GMMD in \Cref{eq:gmmd} has last-iterate convergence to the following regularized NE
\(\max_{\pi} \min_{\pi'} U(\pi; \pi') - \tau R(\pi, \mu) + \tau R(\pi', \mu).\)
\label{prop:RSPO}
% \vspace{-.3cm}
\end{proposition}
Proposition \ref{prop:RSPO} is a direct application of Theorem 3.4 by \citet{sokota2022unified}, which guarantees the last-iterate convergence of GMMD to the NE of a regularized game (Proof in Appendix \ref{append:rspo_proof}).  
% , which can be easily instantiated as the regularized preference optimization game in \Cref{eq:rpm} with $U(\pi; \pi')=\mathbb{P}(\pi \succ \pi')$. 

% \Cref{eq:gmmd} is an update rule that can generalize various game-theoretic methods. MWU can be generalized by $\tau=0$. 

\textbf{Deep RL Implementation of GMMD.} To adapt GMMD to preference optimization problems, RL techniques are commonly employed as practical implementations, as for many MD update \citep{tomar2020mirror,munos2023nash,wang2024magnetic}. 
% Define the loss function of conducting GMMD in preference optimization as
% \begin{align}
% \mathcal{L}_{\text{GMMD}}(\theta) \overset{\text{def}}{=} -\eta \mathbb{E}_{\pi_\theta}\big[G(y;\pi_t)\big] + D_{\text{KL}}(\pi_\theta|| \pi_t)+ \tau R(\pi_\theta, \mu).
% \label{eq:gmmd_loss_def}
% \end{align}
By setting the Bregman divergence $B_\psi$ in \Cref{eq:gmmd} to Reverse KL, as in previous works \citep{munos2023nash,zhang2024iterativenashpolicyoptimization}, we can conduct Policy Gradient (PG) update \citep{sutton1999policy} to approximate GMMD. Under the on-policy assumption \cref{assump:on-policy}, GMMD is equal to $\nabla_\theta\mathcal{L}_{\text{RSPO}}(\theta)$ up to multiplying a constant:\looseness=-1
$\nabla_\theta \mathcal{L}_{\text{GMMD}}(\theta)=$
\begin{align}
% &\nabla_\theta \mathcal{L}_{\text{GMMD}}(\theta)= \tau \nabla_\theta  R(\pi_\theta, \mu) + \nonumber \\
&\mathbb{E}_{y \sim \pi_\theta}\bigg[\nabla_\theta \log \pi_\theta(y)\bigg( -\eta G(y;\pi_t) + \log \frac{\pi_\theta(y)}{ \pi_t(y)} + \eta B \bigg)\bigg] \nonumber \\
& + R(\pi_\theta, \mu) \propto \nabla_\theta\mathcal{L}_{\text{RSPO}}(\theta),
\label{eq:pg}
\end{align}
where $B$ is a baseline function to reduce the variance as in REINFORCE \citep{williams1992simple}. We set $B$ independent to $\theta$ so that adding $B$ does not affect the value of the PG, because $\mathbb{E}_{y \sim \pi_\theta}[\nabla_\theta \log \pi_\theta(y) \cdot \eta B] = \eta B\nabla_\theta \mathbb{E}_{y \sim \pi_\theta}[1]=0$. 

% We further adopt three changes for practical preference optimization. We first leverage a baseline 
% Then we further rewrite this gradient as the derivative of a regularized square loss, \Cref{eq:pg} becomes \looseness=-1
% \begin{align}
% \mathbb{E}_{y \sim \pi_\theta}[\nabla_\theta \log \pi_\theta(y)\big( -\eta G(y;\pi_t) + \log \frac{\pi_\theta(y)}{ \pi_t(y)} \textcolor{red}{+ B} \big)]  \nonumber \\ 
%  + \tau \nabla_\theta  R(\pi_\theta, \mu).
% \label{eq:reinforce}
% \end{align}
% Furthermore, we can rewrite this gradient as the derivative of a regularized square loss:
% \begin{align}
% \frac{1}{2} \mathbb{E}_{y \sim \pi_\theta}[ \nabla_\theta \big(  -\eta G(y;\pi_t) + \log \frac{\pi_\theta(y)}{ \pi_t(y)} {+ B}  \big)^2]    \nonumber \\ 
%  + \tau \nabla_\theta R(\pi_\theta, \mu).
%  \label{eq:square_loss}
% \end{align}

% Due to the equivalence between RSPO and GMMD, we provide the convergence guarantee of RSPO (\Cref{eq:gmmd}), to the Nash equilibrium of the regularized preference optimization game as follows (Proof in Appendix \ref{append:rspo_converge}).
Due to the equivalence between RSPO and GMMD, we analyze the convergence of RSPO (\Cref{eq:gmmd}) to the Nash equilibrium of the regularized preference optimization game as follows (Proof in Appendix \ref{append:rspo_converge}).
\begin{corollary}
Self-play following Algorithm \ref{alg:selfplay} with the RSPO loss function in \Cref{eq:RSPO_gmmd} and regularizer $R$ satisfying the assumptions \ref{assumption:reg} and \ref{assump:on-policy}, has last-iterate convergence to the NE of the regularized preference optimization game, as described in \Cref{eq:rpm}.
\label{coro:rspo_converge}
% \vspace{-.3cm}
\end{corollary}
% Notably, our theoretical guarantee is established under an (approximate) on-policy assumption. In our implementation, RSPO

RSPO guarantees NE convergence while allowing flexible regularization strategies, making it a robust extension of self-play optimization. In summary, the proposed RSPO framework provides a generalized approach that simplifies the incorporation of regularization into existing self-play methods while maintaining theoretical soundness.

\section{Experiments}
\label{sec:experiments}

In this section, we answer the following important questions of regularization in the self-play alignment of Large Language Models (LLMs) by testing on various popular benchmarks:
\begin{itemize}[left=0pt]
\item \textbf{Q1:} Does regularization improve the performance of self-play alignment (Sec. \ref{sec:main_exp})?  
\item \textbf{Q2:} Which regularization method is the most effective in self-play alignment (Sec. \ref{sec:main_reg})? 
\item \textbf{Q3:} What additional advantages can be obtained by regularization in self-play (Sec. \ref{sec:diversity})?
\end{itemize}

\textbf{Experiment Setup.} We evaluate our methods mainly on benchmarks AlpacaEval \citep{dubois2024length}, Arena-Hard \citep{li2024crowdsourced}, and MT-Bench \citep{zheng2023judging}, and test the response generation diversity and quality via self-BLEU \citep{zhu2018texygen} and ArmoRM \citep{wang2024interpretable}, respectively. We follow the experiment setup of SPPO and Snorkel-Mistral-PairRM-DPO
% \footnote{\href{snorkelai/Snorkel-Mistral-PairRM-DPO}{snorkelai/Snorkel-Mistral-PairRM-DPO}} 
(Snorkel) \citep{viethoangtranduong} to examine our regularization methods, where Snorkel is based on iterative DPO and has achieved strong performance on AlpacaEval. We conduct experiments on  base model Mistral-7B-Instruct-v0.2, LLaMA3-8B-Instruct, and Gemma2-2B-IT.
% \footnote{\href{https://huggingface.co/mistralai/Mistral-7B-Instruct-v0.2}{https://huggingface.co/mistralai/Mistral-7B-Instruct-v0.2}} \citep{jiang2023mistral}. 
Since iterative self-play methods require no response data for training, we only use the \textit{prompts of the Ultrafeedback dataset} \citep{cui2023ultrafeedback}, whose size is $\sim 60$K. Following SPPO and Snorkel, we split the prompts into three subsets and use only one subset per iteration to prevent over-fitting. To understand the later-iterate performance of self-play, in section \ref{sec:main_exp}, we also train on the single fold of the prompts iteratively. We use a 0.4B response-pair-wise \textit{preference model} PairRM \citep{llm-blender-2023}, evaluated as comparable to $10\times$ larger reward/preference models \citep{cui2023ultrafeedback}.\looseness=-1

\textbf{Implementations and Baselines.}  The implementation of self-play methods follows Algorithm \ref{alg:selfplay}. In each iteration, given response-pair-wise preference from PairRM and $K=5$ number of response samples from the current policy, we estimate the policies' preference $\mathbb{P}(\pi \succ \pi_t)$ and regularization via Monte-Carlo estimation to compute the loss function. We replicate SPPO with the default hyper-parameters and extend it to $9$ iterations. We implement RSPO as described in \Cref{coro:rspo_converge}. Following SPPO, we collect preference data with the current policy and perform a single optimizer update on the resulting batch in each iteration, making the update effectively on-policy. The implementation of regularizations in RSPO is demonstrated in Appendix \ref{sec:implement_reg} using the $K$ samples. We report some baseline results from the previous papers, including SPPO, Snorkel (Mistral-PairRM-DPO) \citep{viethoangtranduong}, Mistral-7B (Instruct-v0.2) \citep{jiang2023mistral}, iterative DPO by \citet{wu2024self}, and SimPO \cite{meng2024simpo}. Since the SPPO paper only provides results across $3$ iterations \citep{wu2024self}, we replicate SPPO as an important baseline to study the performance across more than $3$ iterations. 
% In the implementation of RSPO, we approximately an on-policy method since we leverage the generated rollouts for only single gradient step update.

\begin{figure*}[t!] 

    \centering 
    \includegraphics[width=.49\linewidth]{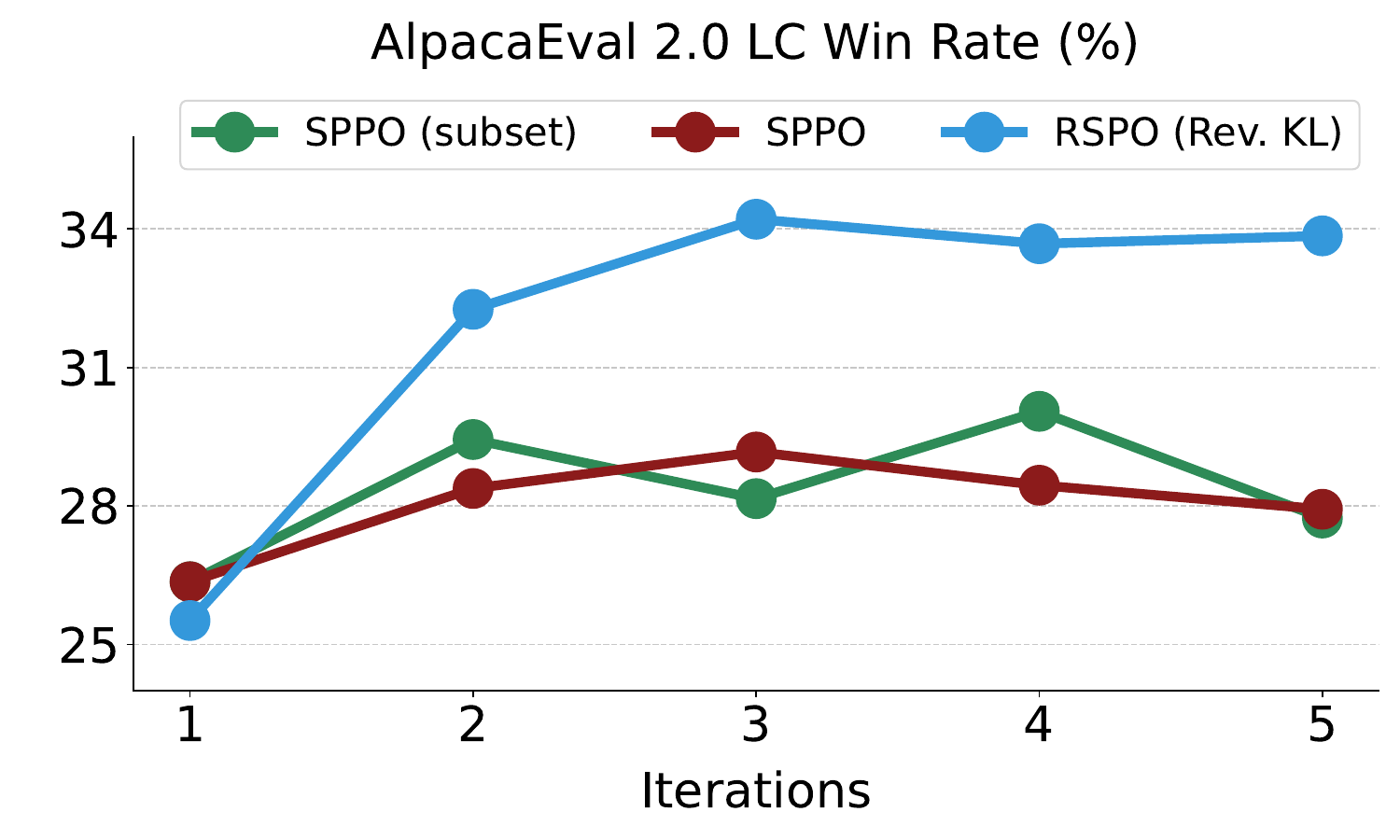}
    \includegraphics[width=.49\linewidth]{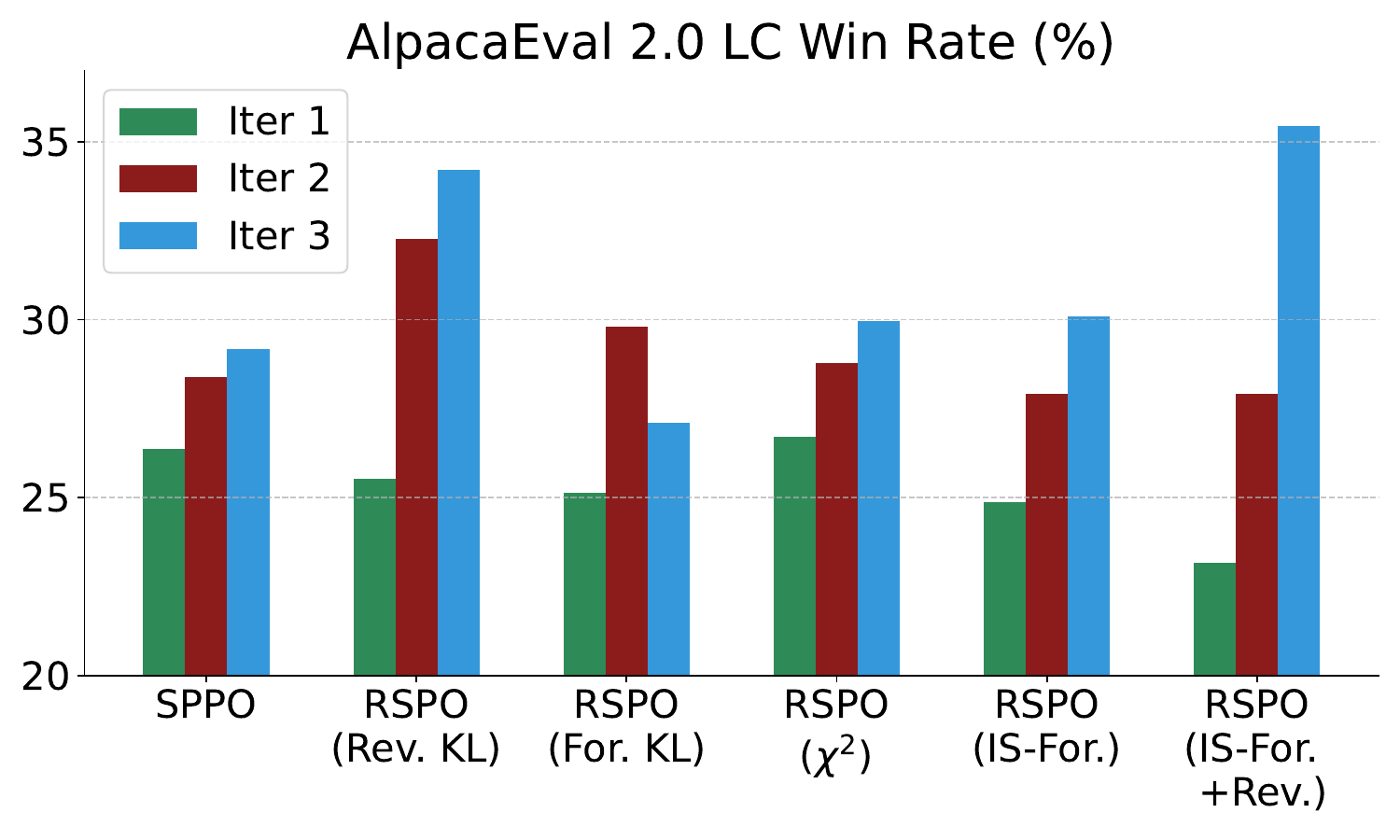}
    
    \caption{\textbf{Left:}  Length-controlled win rate (LCWR) across iterations of: SPPO, SPPO trained on a subset of the data: SPPO (subset), and reverse-KL regularized RSPO: RSPO (Rev. KL). SPPO starts to degrade from iteration 3 due to overoptimization, RSPO with reverse KL regularization mitigates it. \textbf{Right:} SPPO compared to RSPO with various regularization methods. IS-For.+Rev. regularization in RSPO perform the best. The results are obtained by training on base model Mistral-7B-Instruct.}
    \label{fig:sppo_overopt}
    \vspace{-.3cm}
\end{figure*}

\subsection{Effectiveness of Regularization}
\label{sec:main_exp}

In this section, we assess the effectiveness of regularization by comparing the performance of unregularized and regularized self-play methods. We first examine the over-optimization issue inherent in practical self-play alignment by extending the execution of SPPO to Iteration 5. As depicted in Figure \ref{fig:sppo_overopt} (left), a decline in performance appears during the later iterations of SPPO. We hypothesize that this behavior arises from the practical challenges associated with over-optimization.

We then present comprehensive results across three widely used benchmarks (\Cref{fig:fig1}). RSPO with forward and reverse KL regularization outperforms other regularization, showing consistent superior win rates compared to the unregularized baseline (SPPO)\footnote{We report our replicated testing of the published SPPO Iter3 model (\hyperlink{https://huggingface.co/UCLA-AGI/Mistral7B-PairRM-SPPO-Iter3}{link}) on Arena-Hard. Thus, it might be different from the reported results in \citep{wu2024self}.}, iterative DPO in iteration 3, and the strong offline method SimPO across all benchmarks, with a clear performance margin. These results underscore the importance of incorporating regularization into self-play alignment. We hypothesize that the effectiveness of regularization arises from the continued utility of the reference policy during optimization, which provides stable guidance and helps mitigate inaccuracies or misspecifications in the general preference model (PairRM).\looseness=-1

\begin{table}
  \centering
  \caption{\textbf{AlpacaEval LCWR of iterative methods.} RSPO with IS-For.+Rev. regularization shows fast improvement over iterations.}
  \resizebox{\linewidth}{!}{
    \begin{tabular}{cccc} 
    \toprule
    \multirow{2}{*}{Model} & \multicolumn{3}{c}{AlpacaEval 2.0}          \\
                           & LC Win Rate    & Win Rate       & Avg. Len  \\ 
    \midrule
    Mistral-7B             & 17.11          & 14.72          & 1676      \\
    Snorkel                & 26.39          & 30.22          & 2736      \\ 
    SimPO                &  32.1           & 34.8          & 2193      \\ 
    \midrule
    DPO Iter1              & 23.81          & 20.44          & 1723      \\
    DPO Iter2              & 24.23          & 24.46          & 2028      \\
    DPO Iter3              & 22.30          & 23.39          & 2189      \\ 
    % \midrule
    % IPO Iter1              & 23.78          & 20.77          & 1693      \\
    % IPO Iter2              & 21.08          & 23.38          & 2660      \\
    % IPO Iter3              & 20.06          & 22.47          & 2760      \\ 
    \midrule
    SPPO Iter1             & 24.79          & 23.51          & 1855      \\
    SPPO Iter2             & 26.89          & 27.62          & 2019      \\
    SPPO Iter3             & 28.53          & 31.02          & 2163      \\ 
    \rowcolor[gray]{.9} SPPO$^{(3)}\leq$ 9  & 29.17 &	29.75 & 2051 \\
    \midrule \midrule
    \rowcolor[gray]{.9} RSPO Iter1        & 23.16          & 21.06          & 1763      \\
    \rowcolor[gray]{.9} RSPO Iter2        & 27.91          & 27.38          & 1992      \\
    \rowcolor[gray]{.9} RSPO Iter3        & \textbf{35.44} & \textbf{38.31} & 2286      \\
    \bottomrule
    \end{tabular}}
    \label{table:sppo_overopt}
    % \vspace{-1em}
\end{table}

In Table \ref{table:sppo_overopt}, we further contrast the performance dynamics across iterations of SPPO, other iterative methods, and RSPO (For.+Rev.), regularized by the linear combination of Forward KL and Reverse KL divergence with temperatures of $0.1$ and $0.5$, respectively. The comparative results reveal that regularization significantly improve iterative alignment methods.
% regularization enhances the SPPO win rate from $31.02\%$ to $38.31\%$, and the LC win rate increases from $28.53\%$ to $35.44\%$ in iteration $3$. Notably, in the first iteration, reg. SPPO exhibits a slightly lower LC win rate, potentially attributable to the influence of strong regularization. However, subsequent iterations show a marked improvement, with the LC win rate of reg. SPPO increases by up to $7.53\%$ within a single iteration. In summary, Table \ref{table:sppo_overopt} underscores the effectiveness of regularization in self-play optimization.
To rule out the possibility of insufficient iterations affecting performance, we report the best result among nine iterations of our replicated SPPO in Table \ref{table:sppo_overopt}, denoted as "SPPO$^{(3)}\leq9$", where $(3)$ represents that the strongest model is SPPO-Iter3. SPPO$^{(3)}\leq9$ consistently underperforms the RSPO. This observation emphasizes that extended training under the unregularized framework fails to match the performance gains achieved through regularization, thereby affirming again the critical role of regularization, and the policy update guidance provided by reference $\mu$ in self-play methodologies.

We conduct a comprehensive evaluation of regularization effectiveness across multiple base models and examine scalability to larger preference models, as illustrated in Figure \ref{fig:plot_comparison}. Following the methodology established with Mistral-7B, we identify optimal regularization and hyperparameter for additional base architectures. Our experimental results demonstrate that RSPO consistently outperforms SPPO across all evaluated base models, establishing the robustness of our approach.
We also evaluate performance of RSPO trained on a stronger preference model (PM) called GPM-8B \citep{zhang2024beyond}. While the performance improvement with the larger preference model is modest ($0.44$), RSPO demonstrates a significant advantage in parameter efficiency. Notably, RSPO achieves comparable performance to SPPO while utilizing substantially fewer parameters. Specifically, RSPO-LLaMA-8B paired with pairRM-0.5B requires only 8.5B total parameters, achieving performance similar to SPPO-LLaMA-GPM, which utilizes 16B total parameters. This represents a 47\% reduction in model size while maintaining competitive performance, highlighting the parameter efficiency.

% \begin{figure}
%     \centering
%     \caption{\textbf{Comparisons} of using different base models.}
%     \label{fig:across_bases}
% \end{figure}

% \begin{figure}[t!]
%     \centering
%     % \includegraphics[width=.33\linewidth]{figures/RSPO_all.pdf}
%     % \includegraphics[width=.48\linewidth]{figures/RSPO_all_wr.pdf}
%     \includegraphics[width=.5\linewidth]{figures/RSPO_all.pdf}
%     % \includegraphics[width=.48\linewidth]{figures/RSPO_all_length.pdf}
%     \caption{LC win rate of SPPO and RSPO with different regularization methods. From left to right regularization methods: Reverse KL ($\lambda=0.5$), Forward KL ($\lambda=1.0$), Chi-Squared ($\lambda=0.1$), Importance-Sampling Forward KL ($\lambda=0.1$), Forward and Reverse KL linear combination ($\lambda_1=0.1$, $\lambda_2=0.5$).}
%     \vspace{-.3cm}
%     \label{fig:RSPO_all_reg}
%     % \vspace{-2cm}
% \end{figure}

\subsection{Impact of Different Regularizations}
\label{sec:main_reg}

We then study the effect of applying different regularization $R$ in RSPO trained on Mistral-7B. To obtain a well-regularized self-play, the tuning of regularization temperature $\lambda$ is necessary. An ablation study of the regularization temperature of different methods is shown in Figure \ref{fig:rspo_length_para}. According to the figure, the response length increases along with the temperature in reverse KL divergence and Chi-square divergence regularized RSPO. However, both the length and win rate are decreased with stronger regularization via Forward KL divergence, implemented using importance sampling. We attribute the decreasing win rate to the violation of relative convexity assumption (\ref{assumption:reg}), and the length reduction to the intrinsic mass-averaging property of forward-KL divergence used for regularization.

In particular, the raw win rate analysis highlights reverse KL divergence as a crucial factor in enhancing self-play performance. We attribute the observed effect to the inherent mode-seeking behavior of reverse KL divergence. For Mistral-7B, given that forward KL divergence tends to reduce response length while reverse KL divergence yields significant improvements, we adopt a linear combination of both. This approach is designed to balance their complementary effects, ultimately optimizing for a higher LCWR (RSPO (IS-For. + Rev.) in Figure \ref{fig:sppo_overopt} (Right). The hyperparameters are provided in \Cref{tab:divergences}. As for fine-tuning other base models, the results also show similar pattern: forward KL reduce length (Table \ref{tab:LLaMA_for}) and reverse KL improve performance (\Cref{tab:LLaMA_all}). However, for LLaMA-8B, since the length drop is not significant, applying only reverse-KL regularization is still more effective. \looseness=-1

% \vspace{-1em}

% In Figure \ref{fig:sppo_overopt} (Right), we show the results of win rate and LCWR in AlpacaEval 2.0 of different regularizations (). Only vanilla Forward KL decreases the win rate of SPPO. The regularizations that consists of Reverse KL including RSPO (Rev. KL) and RSPO (For.+Rev.) have shown significant improvement in win rates compared to vanilla SPPO. In particular, the results of RSPO (For.+Rev.) demonstrate the largest improvement between iterations, achieving the best LCWR. \looseness=-1

% \begin{figure}[t!]
%     \centering   
%     \begin{minipage}{.49\linewidth}
%     \vspace{-.4cm}
%     \includegraphics[width=\linewidth]{figures/plot_comparison.pdf}
%     \end{minipage}
%     \begin{minipage}{.49\linewidth}
%     \includegraphics[width=\linewidth]{figures/RSPO_para_length.pdf} \includegraphics[width=\linewidth]{figures/RSPO_para_wr.pdf}
%     \end{minipage}
%     \caption{\textbf{Impact of different regularizations}: also an ablation study on regularization temperature $\lambda$ of RSPO conducted on base model Mistral-7B. We evaluate how the average response length and raw WR are affected by the regularization temperature. Higher temperature of forward KL leads to shorter response length.}
%     \label{fig:rspo_length_para}
%     \vspace{-.2cm}
% \end{figure}

\begin{figure}
    \centering
        \includegraphics[width=\linewidth]{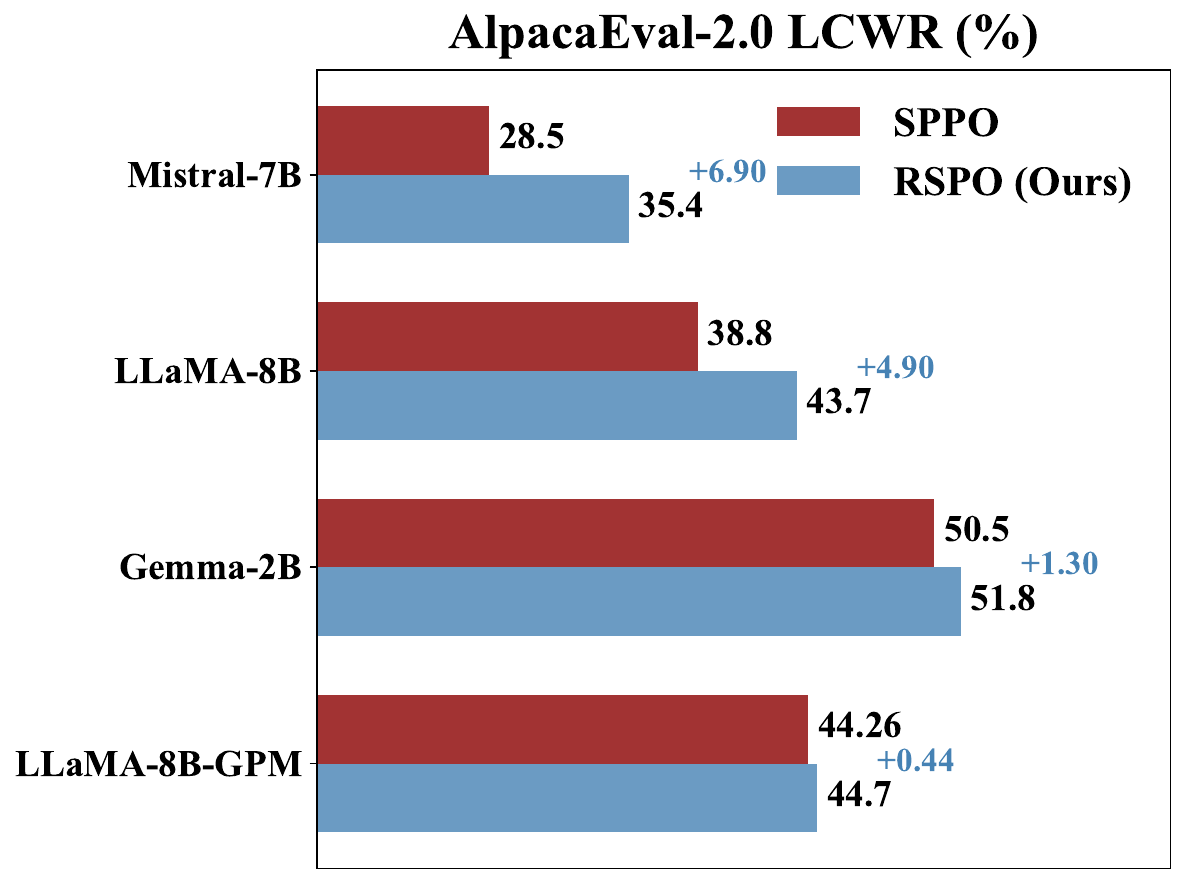}
        \caption{ \textbf{Performance on different base models and stronger general preference model.} The best regularization for Mistral-7B is IS-For.+Rev. and Rev. KL for the rest. LLaMA-8B-GPM represents conducting training on base model LLaMA-8B and preference model GPM-8B \citep{zhang2024beyond}.}
        \label{fig:plot_comparison}
        \vspace{-1em}
\end{figure}

\begin{figure}[t!]
    \centering
    \includegraphics[width=\linewidth]{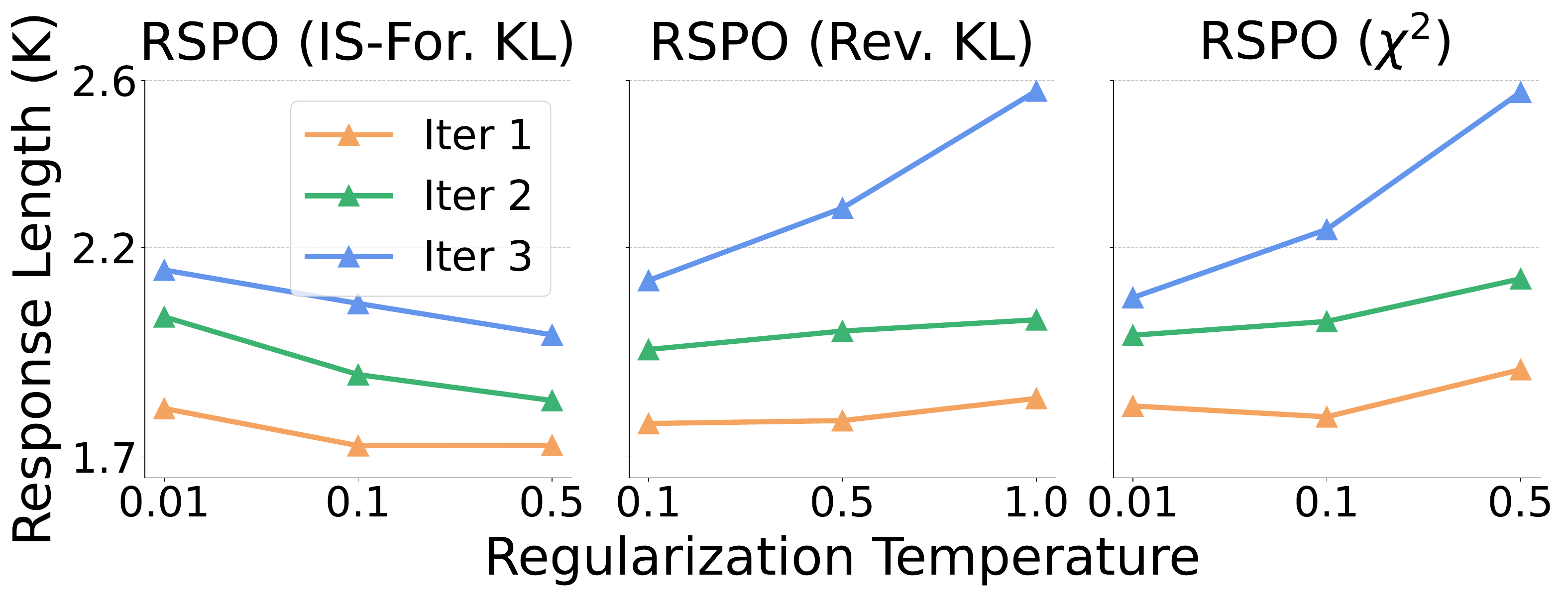}
    \includegraphics[width=\linewidth]{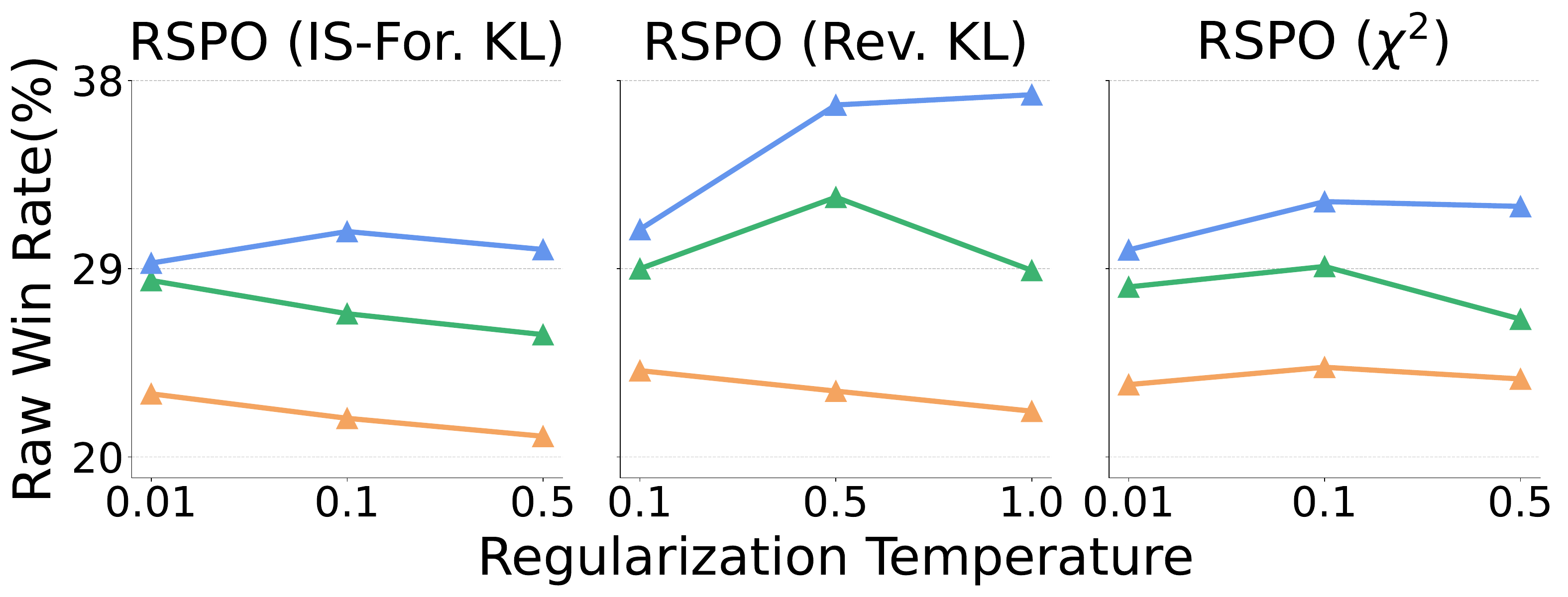}
    \caption{\textbf{Impact of different regularizations on Mistral-7B}: also an ablation study on regularization temperature $\lambda$ of RSPO conducted on Mistral-7B. We evaluate how the average response length and raw WR are affected by the stronger regularization. Higher temperature of forward KL leads to shorter response.}
    \label{fig:rspo_length_para}
    % \vspace{-1.4em}
\end{figure}

% We study the effect of applying different regularization $R$ in RSPO. In Figure \ref{fig:sppo_overopt} (Right), we show the results of win rate and average response length on AlpacaEval 2.0. Among different regularizations, only Forward KL decreases the win rate of SPPO. The regularizations that consist of Reverse KL including RSPO (Rev. KL) and RSPO (For.+Rev.) have shown significant improvement in win rates compared to vanilla SPPO. In particular, the results of RSPO (For.+Rev.) demonstrate the largest improvement between iterations. 

% As for average length, in general, the models with higher win rate have larger response length except RSPO (For. KL) which increases the average length but decreases the win rate.

% We have also done ablation study on regularization temperature.

\subsection{Response Diversity and Other Aspects}
\label{sec:diversity}

We demonstrate additional advantages introduced by regularization on Mistral-7B (see also Table \ref{table:armorm}). Here we demonstrate the diversity of the response cause by regularization.
We first provide a motivating example with synthetic data in Appendix \ref{append:2d_diversity}, which shows that the unregularized self-play may converge to a collapsed response when multiple equally good responses exist. On the contrary, RSPO with maximum entropy regularization has multi-modal distribution for generation. For LLMs, we investigate the diversity of generations by estimating the variability of the responses. We use the Self-BLEU \citep{zhu2018texygen} score, where a lower score implies higher response diversity. We take the first $200$ tokens of each of the $16$ generated responses using the prompts of AlpacaEval.  \looseness=-1

\begin{table}
\centering
% \vspace{-1em}
\captionof{table}{\textbf{Response diversity} of SPPO and RSPO evaluated with Self-BLEU score. Regularization except forward KL improves the diversity.}
\begin{tabular}{cccc} 
\toprule
\multirow{2}{*}{Regularization}                                                   & \multirow{2}{*}{Iteration} & \multicolumn{2}{c}{AlpacaEval 2.0 Dataset}        \\
&                            & LCWR $\uparrow$ & Self-BLEU $\downarrow$  \\ 
\midrule
\multirow{3}{*}{$\times$}                                                               & 1                          & 24.79           & 0.751                   \\
& 2                          & 26.89           & 0.754                   \\
& 3                          & 28.53           & 0.758                   \\ 
\midrule \midrule
\multirow{3}{*}{\begin{tabular}[c]{@{}c@{}}IS-Forward KL\\+ Reverse KL\end{tabular}} & 1                          & 23.16           & 0.747                   \\
& 2                          & 27.91           & 0.743                   \\
& 3                          & \cellcolor{gray!30} \textbf{35.44}  & \cellcolor{gray!30} 0.714                   \\ 
\midrule
\multirow{3}{*}{Reverse KL~}                                                      & 1                          & 25.52           & 0.747                   \\
& 2                          & 32.26           & 0.730                   \\
& 3                          & \cellcolor{gray!30} 34.21           & \cellcolor{gray!30} \textbf{0.691}          \\ 
\midrule
% \multirow{3}{*}{IS-Forward KL}                                                       & 1                          & 24.88           & 0.756                   \\
% & 2                          & 27.9            & 0.759                   \\
% & 3                          & 30.09           & 0.760                   \\
% \midrule
\multirow{3}{*}{$\chi^2$}                                                       & 1                          & 26.7           & 0.745                   \\
& 2                          & 28.78            & 0.740                   \\
& 3                          & 29.97           & 0.739                   \\
\bottomrule
\end{tabular}
% \vspace{.3cm}
\label{table:iter_exp}
% \vspace{-2em}
\end{table}

The trend of Self-BLEU scores presented in \Cref{table:iter_exp} (Right) show that applying RSPO with Reverse KL increases response diversity the most, as well as the LCWRs of AlpacaEval 2.0. Although reverse KL regularization is typically associated with reduced diversity, it can, counterintuitively, enhance diversity when the high-probability region of the reference policy $\mu$ contains multiple modes—a scenario commonly arising when $\mu$ is pretrained on a diverse dataset. In such cases, the sampling-regularized optimization process with reverse KL can also induce additional modes in the learned policy distribution, thereby promoting greater diversity in responses. In contrast, IS-Forward KL yields slightly lower diversity, as its importance sampling–based implementation necessitates hard clipping for numerical stability. Compared to reverse KL, the $\chi^2$ divergence functions as a stronger regularizer \citep{huang2024correcting}, promoting diversity, albeit at a slower rate.

% \textcolor{blue}{Xiaohang: Not sure about this claim: We attribute this phenomenon to the fact that Reverse KL contains negative entropy term of $\pi_\theta$, so that minimising the loss leads to increasing the entropy of the policy.}  
% The application of IS-Forward KL results in slightly less generation diversity since the importance-sampling-based implementation requires hard clipping for numerical stability. Compared to reverse KL divergence, $\chi^2$ divergence is treated as a strong regularizer \citep{huang2024correcting}, and thus improves diversity but less quickly. 
% These results highlight that applying regularization in self-play methods can improve test performance and the diversity of the generations simultaneously.

% Finally, we assess additional aspects of response quality on the Ultrafeedback validation set using ArmoRM \cite{wang2024interpretable} (Table \ref{table:armorm}). Reverse KL regularization improves both truthfulness and helpfulness. Notably, while forward and reverse KL regularization individually tend to diminish instruction-following performance, their combination yields improvements across nearly all evaluated aspects, attaining the highest overall score—with particularly strong gains in Instruction Following and Truthfulness.

% \vspace{-.5em}
\section{Related Work}
\textbf{Offline RLHF with general divergence for regularization.} The use of general divergence-based regularization has been explored in the context of offline alignment. $f$-DPO \citep{wang2023beyond} extends Direct Preference Optimization \citep{rafailov2024direct} from reverse KL regularization to a broader class of $f$-divergences, but primarily demonstrates benefits in generation diversity. The specific effects of individual divergences—and their performance on widely-used benchmarks such as AlpacaEval—remain unexamined. $\chi$PO \citep{huang2024correcting} emphasizes the theoretical importance of $\chi^2$ divergence for uncertainty quantification. However, the role of regularization in online iterative preference optimization, particularly its empirical impact on standard benchmarks, has yet to be studied. \looseness=-1

\textbf{Self-Play .} We emphasize the distinction between our self-play approach and \textit{contrastive} self-play methods including Direct Nash Optimization (DNO) \citep{rosset2024direct} and Iterative Nash Policy Optimization (INPO) \citep{zhang2024iterativenashpolicyoptimization}. These methods conduct policy optimization with a loss objective necessary but not sufficient for Mirror Descent (MD) update \citep{beck2003mirror}. This objective is constructed via winner-loser response comparisons similar to Direct Preference Optimization (DPO) and Identity Preference Optimization (IPO) \citep{azar2024general}. Optimizing such contrastive loss can lead to only an increase in the relative likelihood gap without necessarily enhancing the absolute probability of the preferred response \citep{pal2024smaug}. In contrast, our method estimates the payoff and directly approximates the MD update by converting it to an equivalent reinforcement learning problem, thereby circumventing the limitations of contrastive approaches. \looseness=-1

\section{Conclusion}
In this paper, we study the regularization in self-play by proposing a framework, namely Regularized Self-Play Policy Optimization (RSPO). Based on RSPO, we can apply different regularization functions for policy updates by adding the regularization term to the loss functions of existing self-play alignment methods, which we prove is still guaranteed to converge to the NE of the regularized preference optimization game. In the experiments, RSPO with tuned regularizers achieve significant improvement over the base model and unregularized self-play method, SPPO. RSPO significantly improves the performance across various base models, and shows additional parameter-efficiency. We also empirically show that regularization promotes various response quality including diversity. These findings underscore the critical role of regularization in self-play alignment.

\section*{Impact Statement}
We develop an algorithmic framework for stable self-play optimization for large language model alignment. We do not anticipate significant new ethical or societal risks beyond those generally associated with improving optimization and capabilities of LLMs.

\section*{Acknowledgements}
We thank the anonymous reviewers and area chair for their helpful comments. Ilija Bogunovic was supported by the EPSRC New Investigator Award EP/X03917X/1; the Engineering
and Physical Sciences Research Council EP/S021566/1. Sangwoong Yoon was supported by Institute of Information \& Communications Technology Planning \& Evaluation (IITP) grant funded by the Korea government (MSIT) (No. RS-2020-II201336, Artificial Intelligence Graduate School Program (UNIST); No. RS-2025-25442824, AI Star Fellowship Program (UNIST)), the Center for Advanced Computation at Korea Institute for Advanced Study, and the InnoCORE program of the Ministry of Science and ICT (1.260017.01). Xiaohang Tang was supported by the UKRI Engineering and Physical Sciences Research Council [grant number EP/T517793/1, EP/W524335/1]. HY and QG are supported in part by the National Science Foundation under grants DMS-2323113, IIS-2403400 and DMS-2502536, as well as the Sloan Research Fellowship. The views and conclusions contained in this paper are those of the authors and should not be interpreted as representing any funding agencies.

% As for reproducibility of our results, we provide details of implementations in \Cref{sec:experiments} Implementations and Baselines, and \Cref{sec:implement_reg}. We have provided Hyperparameters of each regularization methods in \Cref{tab:divergences}. All the theoretical results are proved in \Cref{append:proof}. 

% \section*{Impact Statement}
% This paper presents a novel framework for fine-tuning Large Language Models via self-play, with regularization to a Supervised Fine-Tuning (SFT) reference model. Ethical challenges may arise if the SFT reference model tends to generate harmful content or if the preference model for policy updates assigns higher ratings to harmful texts. However, based on previous research, we believe our work does not have any direct negative societal consequences.

\bibliography{example_paper}

\begin{thebibliography}{70}
\providecommand{\natexlab}[1]{#1}
\providecommand{\url}[1]{\texttt{#1}}
\expandafter\ifx\csname urlstyle\endcsname\relax
  \providecommand{\doi}[1]{doi: #1}\else
  \providecommand{\doi}{doi: \begingroup \urlstyle{rm}\Url}\fi

\bibitem[Abe et~al.(2022)Abe, Ariu, Sakamoto, Toyoshima, and Iwasaki]{abe2022last}
Abe, K., Ariu, K., Sakamoto, M., Toyoshima, K., and Iwasaki, A.
\newblock Last-iterate convergence with full and noisy feedback in two-player zero-sum games.
\newblock \emph{arXiv preprint arXiv:2208.09855}, 2022.

\bibitem[Abe et~al.(2023)Abe, Ariu, Sakamoto, and Iwasaki]{abe2023adaptively}
Abe, K., Ariu, K., Sakamoto, M., and Iwasaki, A.
\newblock Adaptively perturbed mirror descent for learning in games.
\newblock \emph{arXiv preprint arXiv:2305.16610}, 2023.

\bibitem[Abe et~al.(2024)Abe, Sakamoto, Ariu, and Iwasaki]{abe2024boosting}
Abe, K., Sakamoto, M., Ariu, K., and Iwasaki, A.
\newblock Boosting perturbed gradient ascent for last-iterate convergence in games.
\newblock \emph{arXiv preprint arXiv:2410.02388}, 2024.

\bibitem[Azar et~al.(2024)Azar, Guo, Piot, Munos, Rowland, Valko, and Calandriello]{azar2024general}
Azar, M.~G., Guo, Z.~D., Piot, B., Munos, R., Rowland, M., Valko, M., and Calandriello, D.
\newblock A general theoretical paradigm to understand learning from human preferences.
\newblock In \emph{International Conference on Artificial Intelligence and Statistics}, pp.\  4447--4455. PMLR, 2024.

\bibitem[Bai et~al.(2022)Bai, Jones, Ndousse, Askell, Chen, DasSarma, Drain, Fort, Ganguli, Henighan, et~al.]{bai2022training}
Bai, Y., Jones, A., Ndousse, K., Askell, A., Chen, A., DasSarma, N., Drain, D., Fort, S., Ganguli, D., Henighan, T., et~al.
\newblock Training a helpful and harmless assistant with reinforcement learning from human feedback.
\newblock \emph{arXiv preprint arXiv:2204.05862}, 2022.

\bibitem[Beck \& Teboulle(2003)Beck and Teboulle]{beck2003mirror}
Beck, A. and Teboulle, M.
\newblock Mirror descent and nonlinear projected subgradient methods for convex optimization.
\newblock \emph{Operations Research Letters}, 31\penalty0 (3):\penalty0 167--175, 2003.

\bibitem[Bregman(1967)]{bregman1967relaxation}
Bregman, L.~M.
\newblock The relaxation method of finding the common point of convex sets and its application to the solution of problems in convex programming.
\newblock \emph{USSR computational mathematics and mathematical physics}, 7\penalty0 (3):\penalty0 200--217, 1967.

\bibitem[Brown \& Sandholm(2018)Brown and Sandholm]{brown2018superhuman}
Brown, N. and Sandholm, T.
\newblock Superhuman ai for heads-up no-limit poker: Libratus beats top professionals.
\newblock \emph{Science}, 359\penalty0 (6374):\penalty0 418--424, 2018.

\bibitem[Calandriello et~al.(2024)Calandriello, Guo, Munos, Rowland, Tang, Pires, Richemond, Lan, Valko, Liu, et~al.]{calandriello2024human}
Calandriello, D., Guo, D., Munos, R., Rowland, M., Tang, Y., Pires, B.~A., Richemond, P.~H., Lan, C.~L., Valko, M., Liu, T., et~al.
\newblock Human alignment of large language models through online preference optimisation.
\newblock \emph{arXiv preprint arXiv:2403.08635}, 2024.

\bibitem[Cen et~al.(2021)Cen, Wei, and Chi]{cen2021fast}
Cen, S., Wei, Y., and Chi, Y.
\newblock Fast policy extragradient methods for competitive games with entropy regularization.
\newblock \emph{Advances in Neural Information Processing Systems}, 34:\penalty0 27952--27964, 2021.

\bibitem[Cen et~al.(2024)Cen, Mei, Goshvadi, Dai, Yang, Yang, Schuurmans, Chi, and Dai]{cen2024value}
Cen, S., Mei, J., Goshvadi, K., Dai, H., Yang, T., Yang, S., Schuurmans, D., Chi, Y., and Dai, B.
\newblock Value-incentivized preference optimization: A unified approach to online and offline rlhf.
\newblock \emph{arXiv preprint arXiv:2405.19320}, 2024.

\bibitem[Christiano et~al.(2017)Christiano, Leike, Brown, Martic, Legg, and Amodei]{christiano2017deep}
Christiano, P.~F., Leike, J., Brown, T., Martic, M., Legg, S., and Amodei, D.
\newblock Deep reinforcement learning from human preferences.
\newblock \emph{Advances in neural information processing systems}, 30, 2017.

\bibitem[Cui et~al.(2023)Cui, Yuan, Ding, Yao, Zhu, Ni, Xie, Liu, and Sun]{cui2023ultrafeedback}
Cui, G., Yuan, L., Ding, N., Yao, G., Zhu, W., Ni, Y., Xie, G., Liu, Z., and Sun, M.
\newblock Ultrafeedback: Boosting language models with high-quality feedback.
\newblock 2023.

\bibitem[David(1963)]{david1963method}
David, H.~A.
\newblock \emph{The method of paired comparisons}, volume~12.
\newblock London, 1963.

\bibitem[DeepSeek-AI et~al.(2025)DeepSeek-AI, Guo, Yang, Zhang, Song, Zhang, Xu, Zhu, Ma, Wang, Bi, Zhang, Yu, Wu, Wu, Gou, Shao, Li, Gao, Liu, Xue, Wang, Wu, Feng, Lu, Zhao, Deng, Zhang, Ruan, Dai, Chen, Ji, Li, Lin, Dai, Luo, Hao, Chen, Li, Zhang, Bao, Xu, Wang, Ding, Xin, Gao, Qu, Li, Guo, Li, Wang, Chen, Yuan, Qiu, Li, Cai, Ni, Liang, Chen, Dong, Hu, Gao, Guan, Huang, Yu, Wang, Zhang, Zhao, Wang, Zhang, Xu, Xia, Zhang, Zhang, Tang, Li, Wang, Li, Tian, Huang, Zhang, Wang, Chen, Du, Ge, Zhang, Pan, Wang, Chen, Jin, Chen, Lu, Zhou, Chen, Ye, Wang, Yu, Zhou, Pan, Li, Zhou, Wu, Ye, Yun, Pei, Sun, Wang, Zeng, Zhao, Liu, Liang, Gao, Yu, Zhang, Xiao, An, Liu, Wang, Chen, Nie, Cheng, Liu, Xie, Liu, Yang, Li, Su, Lin, Li, Jin, Shen, Chen, Sun, Wang, Song, Zhou, Wang, Shan, Li, Wang, Wei, Zhang, Xu, Li, Zhao, Sun, Wang, Yu, Zhang, Shi, Xiong, He, Piao, Wang, Tan, Ma, Liu, Guo, Ou, Wang, Gong, Zou, He, Xiong, Luo, You, Liu, Zhou, Zhu, Xu, Huang, Li, Zheng, Zhu, Ma, Tang, Zha, Yan, Ren, Ren, Sha, Fu, Xu, Xie, Zhang,
  Hao, Ma, Yan, Wu, Gu, Zhu, Liu, Li, Xie, Song, Pan, Huang, Xu, Zhang, and Zhang]{deepseekai2025deepseekr1incentivizingreasoningcapability}
DeepSeek-AI, Guo, D., Yang, D., Zhang, H., Song, J., Zhang, R., Xu, R., Zhu, Q., Ma, S., Wang, P., Bi, X., Zhang, X., Yu, X., Wu, Y., Wu, Z.~F., Gou, Z., Shao, Z., Li, Z., Gao, Z., Liu, A., Xue, B., Wang, B., Wu, B., Feng, B., Lu, C., Zhao, C., Deng, C., Zhang, C., Ruan, C., Dai, D., Chen, D., Ji, D., Li, E., Lin, F., Dai, F., Luo, F., Hao, G., Chen, G., Li, G., Zhang, H., Bao, H., Xu, H., Wang, H., Ding, H., Xin, H., Gao, H., Qu, H., Li, H., Guo, J., Li, J., Wang, J., Chen, J., Yuan, J., Qiu, J., Li, J., Cai, J.~L., Ni, J., Liang, J., Chen, J., Dong, K., Hu, K., Gao, K., Guan, K., Huang, K., Yu, K., Wang, L., Zhang, L., Zhao, L., Wang, L., Zhang, L., Xu, L., Xia, L., Zhang, M., Zhang, M., Tang, M., Li, M., Wang, M., Li, M., Tian, N., Huang, P., Zhang, P., Wang, Q., Chen, Q., Du, Q., Ge, R., Zhang, R., Pan, R., Wang, R., Chen, R.~J., Jin, R.~L., Chen, R., Lu, S., Zhou, S., Chen, S., Ye, S., Wang, S., Yu, S., Zhou, S., Pan, S., Li, S.~S., Zhou, S., Wu, S., Ye, S., Yun, T., Pei, T., Sun, T., Wang, T., Zeng, W.,
  Zhao, W., Liu, W., Liang, W., Gao, W., Yu, W., Zhang, W., Xiao, W.~L., An, W., Liu, X., Wang, X., Chen, X., Nie, X., Cheng, X., Liu, X., Xie, X., Liu, X., Yang, X., Li, X., Su, X., Lin, X., Li, X.~Q., Jin, X., Shen, X., Chen, X., Sun, X., Wang, X., Song, X., Zhou, X., Wang, X., Shan, X., Li, Y.~K., Wang, Y.~Q., Wei, Y.~X., Zhang, Y., Xu, Y., Li, Y., Zhao, Y., Sun, Y., Wang, Y., Yu, Y., Zhang, Y., Shi, Y., Xiong, Y., He, Y., Piao, Y., Wang, Y., Tan, Y., Ma, Y., Liu, Y., Guo, Y., Ou, Y., Wang, Y., Gong, Y., Zou, Y., He, Y., Xiong, Y., Luo, Y., You, Y., Liu, Y., Zhou, Y., Zhu, Y.~X., Xu, Y., Huang, Y., Li, Y., Zheng, Y., Zhu, Y., Ma, Y., Tang, Y., Zha, Y., Yan, Y., Ren, Z.~Z., Ren, Z., Sha, Z., Fu, Z., Xu, Z., Xie, Z., Zhang, Z., Hao, Z., Ma, Z., Yan, Z., Wu, Z., Gu, Z., Zhu, Z., Liu, Z., Li, Z., Xie, Z., Song, Z., Pan, Z., Huang, Z., Xu, Z., Zhang, Z., and Zhang, Z.
\newblock Deepseek-r1: Incentivizing reasoning capability in llms via reinforcement learning, 2025.
\newblock URL \url{https://arxiv.org/abs/2501.12948}.

\bibitem[Dong et~al.(2024)Dong, Xiong, Pang, Wang, Zhao, Zhou, Jiang, Sahoo, Xiong, and Zhang]{dong2024rlhf}
Dong, H., Xiong, W., Pang, B., Wang, H., Zhao, H., Zhou, Y., Jiang, N., Sahoo, D., Xiong, C., and Zhang, T.
\newblock Rlhf workflow: From reward modeling to online rlhf.
\newblock \emph{arXiv preprint arXiv:2405.07863}, 2024.

\bibitem[Dubey et~al.(2024)Dubey, Jauhri, Pandey, Kadian, Al-Dahle, Letman, Mathur, Schelten, Yang, Fan, et~al.]{dubey2024LLaMA}
Dubey, A., Jauhri, A., Pandey, A., Kadian, A., Al-Dahle, A., Letman, A., Mathur, A., Schelten, A., Yang, A., Fan, A., et~al.
\newblock The llama 3 herd of models.
\newblock \emph{arXiv preprint arXiv:2407.21783}, 2024.

\bibitem[Dubois et~al.(2024)Dubois, Galambosi, Liang, and Hashimoto]{dubois2024length}
Dubois, Y., Galambosi, B., Liang, P., and Hashimoto, T.~B.
\newblock Length-controlled alpacaeval: A simple way to debias automatic evaluators.
\newblock \emph{arXiv preprint arXiv:2404.04475}, 2024.

\bibitem[Facchinei \& Pang(2003)Facchinei and Pang]{facchinei2003finite}
Facchinei, F. and Pang, J.-S.
\newblock \emph{Finite-dimensional variational inequalities and complementarity problems}.
\newblock Springer, 2003.

\bibitem[Freund \& Schapire(1997)Freund and Schapire]{freund1997decision}
Freund, Y. and Schapire, R.~E.
\newblock A decision-theoretic generalization of on-line learning and an application to boosting.
\newblock \emph{Journal of computer and system sciences}, 55\penalty0 (1):\penalty0 119--139, 1997.

\bibitem[Gao et~al.(2024)Gao, Chang, Zhan, Oertell, Swamy, Brantley, Joachims, Bagnell, Lee, and Sun]{gao2024rebel}
Gao, Z., Chang, J.~D., Zhan, W., Oertell, O., Swamy, G., Brantley, K., Joachims, T., Bagnell, J.~A., Lee, J.~D., and Sun, W.
\newblock Rebel: Reinforcement learning via regressing relative rewards.
\newblock \emph{arXiv preprint arXiv:2404.16767}, 2024.

\bibitem[Go et~al.(2023)Go, Korbak, Kruszewski, Rozen, Ryu, and Dymetman]{go2023aligning}
Go, D., Korbak, T., Kruszewski, G., Rozen, J., Ryu, N., and Dymetman, M.
\newblock Aligning language models with preferences through f-divergence minimization.
\newblock \emph{arXiv preprint arXiv:2302.08215}, 2023.

\bibitem[Goodfellow et~al.(2020)Goodfellow, Pouget-Abadie, Mirza, Xu, Warde-Farley, Ozair, Courville, and Bengio]{goodfellow2020generative}
Goodfellow, I., Pouget-Abadie, J., Mirza, M., Xu, B., Warde-Farley, D., Ozair, S., Courville, A., and Bengio, Y.
\newblock Generative adversarial networks.
\newblock \emph{Communications of the ACM}, 63\penalty0 (11):\penalty0 139--144, 2020.

\bibitem[Guo et~al.(2024)Guo, Zhang, Liu, Liu, Khalman, Llinares, Rame, Mesnard, Zhao, Piot, et~al.]{guo2024direct}
Guo, S., Zhang, B., Liu, T., Liu, T., Khalman, M., Llinares, F., Rame, A., Mesnard, T., Zhao, Y., Piot, B., et~al.
\newblock Direct language model alignment from online ai feedback.
\newblock \emph{arXiv preprint arXiv:2402.04792}, 2024.

\bibitem[Heinrich \& Silver(2016)Heinrich and Silver]{heinrich2016deep}
Heinrich, J. and Silver, D.
\newblock Deep reinforcement learning from self-play in imperfect-information games.
\newblock \emph{arXiv preprint arXiv:1603.01121}, 2016.

\bibitem[Hennes et~al.(2019)Hennes, Morrill, Omidshafiei, Munos, Perolat, Lanctot, Gruslys, Lespiau, Parmas, Duenez-Guzman, et~al.]{hennes2019neural}
Hennes, D., Morrill, D., Omidshafiei, S., Munos, R., Perolat, J., Lanctot, M., Gruslys, A., Lespiau, J.-B., Parmas, P., Duenez-Guzman, E., et~al.
\newblock Neural replicator dynamics.
\newblock \emph{arXiv preprint arXiv:1906.00190}, 2019.

\bibitem[Huang et~al.(2024)Huang, Zhan, Xie, Lee, Sun, Krishnamurthy, and Foster]{huang2024correcting}
Huang, A., Zhan, W., Xie, T., Lee, J.~D., Sun, W., Krishnamurthy, A., and Foster, D.~J.
\newblock Correcting the mythos of kl-regularization: Direct alignment without overoptimization via chi-squared preference optimization.
\newblock \emph{arXiv preprint arXiv:2407.13399}, 2024.

\bibitem[Jacob et~al.(2022)Jacob, Wu, Farina, Lerer, Hu, Bakhtin, Andreas, and Brown]{jacob2022modeling}
Jacob, A.~P., Wu, D.~J., Farina, G., Lerer, A., Hu, H., Bakhtin, A., Andreas, J., and Brown, N.
\newblock Modeling strong and human-like gameplay with kl-regularized search.
\newblock In \emph{International Conference on Machine Learning}, pp.\  9695--9728. PMLR, 2022.

\bibitem[Jiang et~al.(2023{\natexlab{a}})Jiang, Sablayrolles, Mensch, Bamford, Chaplot, Casas, Bressand, Lengyel, Lample, Saulnier, et~al.]{jiang2023mistral}
Jiang, A.~Q., Sablayrolles, A., Mensch, A., Bamford, C., Chaplot, D.~S., Casas, D. d.~l., Bressand, F., Lengyel, G., Lample, G., Saulnier, L., et~al.
\newblock Mistral 7b.
\newblock \emph{arXiv preprint arXiv:2310.06825}, 2023{\natexlab{a}}.

\bibitem[Jiang et~al.(2023{\natexlab{b}})Jiang, Ren, and Lin]{llm-blender-2023}
Jiang, D., Ren, X., and Lin, B.~Y.
\newblock Llm-blender: Ensembling large language models with pairwise comparison and generative fusion.
\newblock In \emph{Proceedings of the 61th Annual Meeting of the Association for Computational Linguistics (ACL 2023)}, 2023{\natexlab{b}}.

\bibitem[Jiang et~al.(2024)Jiang, Huang, Wu, and Wei]{jiang2024textual}
Jiang, L., Huang, S., Wu, X., and Wei, F.
\newblock Textual aesthetics in large language models.
\newblock \emph{arXiv preprint arXiv:2411.02930}, 2024.

\bibitem[Li et~al.(2024)Li, Chiang, Frick, Dunlap, Wu, Zhu, Gonzalez, and Stoica]{li2024crowdsourced}
Li, T., Chiang, W.-L., Frick, E., Dunlap, L., Wu, T., Zhu, B., Gonzalez, J.~E., and Stoica, I.
\newblock From crowdsourced data to high-quality benchmarks: Arena-hard and benchbuilder pipeline.
\newblock \emph{arXiv preprint arXiv:2406.11939}, 2024.

\bibitem[Liu et~al.(2022)Liu, Ozdaglar, Yu, and Zhang]{liu2022power}
Liu, M., Ozdaglar, A., Yu, T., and Zhang, K.
\newblock The power of regularization in solving extensive-form games.
\newblock \emph{arXiv preprint arXiv:2206.09495}, 2022.

\bibitem[Liu et~al.(2023)Liu, Zhao, Joshi, Khalman, Saleh, Liu, and Liu]{liu2023statistical}
Liu, T., Zhao, Y., Joshi, R., Khalman, M., Saleh, M., Liu, P.~J., and Liu, J.
\newblock Statistical rejection sampling improves preference optimization.
\newblock \emph{arXiv preprint arXiv:2309.06657}, 2023.

\bibitem[Meng et~al.(2024)Meng, Xia, and Chen]{meng2024simpo}
Meng, Y., Xia, M., and Chen, D.
\newblock Simpo: Simple preference optimization with a reference-free reward.
\newblock \emph{arXiv preprint arXiv:2405.14734}, 2024.

\bibitem[Munos et~al.(2023)Munos, Valko, Calandriello, Azar, Rowland, Guo, Tang, Geist, Mesnard, Michi, et~al.]{munos2023nash}
Munos, R., Valko, M., Calandriello, D., Azar, M.~G., Rowland, M., Guo, Z.~D., Tang, Y., Geist, M., Mesnard, T., Michi, A., et~al.
\newblock Nash learning from human feedback.
\newblock \emph{arXiv preprint arXiv:2312.00886}, 2023.

\bibitem[Ouyang et~al.(2022)Ouyang, Wu, Jiang, Almeida, Wainwright, Mishkin, Zhang, Agarwal, Slama, Ray, et~al.]{ouyang2022training}
Ouyang, L., Wu, J., Jiang, X., Almeida, D., Wainwright, C., Mishkin, P., Zhang, C., Agarwal, S., Slama, K., Ray, A., et~al.
\newblock Training language models to follow instructions with human feedback.
\newblock \emph{Advances in neural information processing systems}, 35:\penalty0 27730--27744, 2022.

\bibitem[Pal et~al.(2024)Pal, Karkhanis, Dooley, Roberts, Naidu, and White]{pal2024smaug}
Pal, A., Karkhanis, D., Dooley, S., Roberts, M., Naidu, S., and White, C.
\newblock Smaug: Fixing failure modes of preference optimisation with dpo-positive.
\newblock \emph{arXiv preprint arXiv:2402.13228}, 2024.

\bibitem[Pang et~al.(2024)Pang, Yuan, Cho, He, Sukhbaatar, and Weston]{pang2024iterative}
Pang, R.~Y., Yuan, W., Cho, K., He, H., Sukhbaatar, S., and Weston, J.
\newblock Iterative reasoning preference optimization.
\newblock \emph{arXiv preprint arXiv:2404.19733}, 2024.

\bibitem[Pattathil et~al.(2023)Pattathil, Zhang, and Ozdaglar]{pattathil2023symmetric}
Pattathil, S., Zhang, K., and Ozdaglar, A.
\newblock Symmetric (optimistic) natural policy gradient for multi-agent learning with parameter convergence.
\newblock In \emph{International Conference on Artificial Intelligence and Statistics}, pp.\  5641--5685. PMLR, 2023.

\bibitem[Perolat et~al.(2021)Perolat, Munos, Lespiau, Omidshafiei, Rowland, Ortega, Burch, Anthony, Balduzzi, De~Vylder, et~al.]{perolat2021poincare}
Perolat, J., Munos, R., Lespiau, J.-B., Omidshafiei, S., Rowland, M., Ortega, P., Burch, N., Anthony, T., Balduzzi, D., De~Vylder, B., et~al.
\newblock From poincar{\'e} recurrence to convergence in imperfect information games: Finding equilibrium via regularization.
\newblock In \emph{International Conference on Machine Learning}, pp.\  8525--8535. PMLR, 2021.

\bibitem[Perolat et~al.(2022)Perolat, De~Vylder, Hennes, Tarassov, Strub, de~Boer, Muller, Connor, Burch, Anthony, et~al.]{perolat2022mastering}
Perolat, J., De~Vylder, B., Hennes, D., Tarassov, E., Strub, F., de~Boer, V., Muller, P., Connor, J.~T., Burch, N., Anthony, T., et~al.
\newblock Mastering the game of stratego with model-free multiagent reinforcement learning.
\newblock \emph{Science}, 378\penalty0 (6623):\penalty0 990--996, 2022.

\bibitem[Pinto et~al.(2017)Pinto, Davidson, Sukthankar, and Gupta]{pinto2017robust}
Pinto, L., Davidson, J., Sukthankar, R., and Gupta, A.
\newblock Robust adversarial reinforcement learning.
\newblock In \emph{International conference on machine learning}, pp.\  2817--2826. PMLR, 2017.

\bibitem[Rafailov et~al.(2024)Rafailov, Sharma, Mitchell, Manning, Ermon, and Finn]{rafailov2024direct}
Rafailov, R., Sharma, A., Mitchell, E., Manning, C.~D., Ermon, S., and Finn, C.
\newblock Direct preference optimization: Your language model is secretly a reward model.
\newblock \emph{Advances in Neural Information Processing Systems}, 36, 2024.

\bibitem[Rosset et~al.(2024)Rosset, Cheng, Mitra, Santacroce, Awadallah, and Xie]{rosset2024direct}
Rosset, C., Cheng, C.-A., Mitra, A., Santacroce, M., Awadallah, A., and Xie, T.
\newblock Direct nash optimization: Teaching language models to self-improve with general preferences.
\newblock \emph{arXiv preprint arXiv:2404.03715}, 2024.

\bibitem[Rubenstein et~al.(2019)Rubenstein, Bousquet, Djolonga, Riquelme, and Tolstikhin]{rubenstein2019practical}
Rubenstein, P., Bousquet, O., Djolonga, J., Riquelme, C., and Tolstikhin, I.~O.
\newblock Practical and consistent estimation of f-divergences.
\newblock \emph{Advances in Neural Information Processing Systems}, 32, 2019.

\bibitem[Schulman et~al.(2017)Schulman, Wolski, Dhariwal, Radford, and Klimov]{schulman2017proximal}
Schulman, J., Wolski, F., Dhariwal, P., Radford, A., and Klimov, O.
\newblock Proximal policy optimization algorithms.
\newblock \emph{arXiv preprint arXiv:1707.06347}, 2017.

\bibitem[Silver et~al.(2016)Silver, Huang, Maddison, Guez, Sifre, Van Den~Driessche, Schrittwieser, Antonoglou, Panneershelvam, Lanctot, et~al.]{silver2016mastering}
Silver, D., Huang, A., Maddison, C.~J., Guez, A., Sifre, L., Van Den~Driessche, G., Schrittwieser, J., Antonoglou, I., Panneershelvam, V., Lanctot, M., et~al.
\newblock Mastering the game of go with deep neural networks and tree search.
\newblock \emph{nature}, 529\penalty0 (7587):\penalty0 484--489, 2016.

\bibitem[Sion(1958)]{sion1958general}
Sion, M.
\newblock On general minimax theorems.
\newblock 1958.

\bibitem[Sokota et~al.(2022)Sokota, D'Orazio, Kolter, Loizou, Lanctot, Mitliagkas, Brown, and Kroer]{sokota2022unified}
Sokota, S., D'Orazio, R., Kolter, J.~Z., Loizou, N., Lanctot, M., Mitliagkas, I., Brown, N., and Kroer, C.
\newblock A unified approach to reinforcement learning, quantal response equilibria, and two-player zero-sum games.
\newblock \emph{arXiv preprint arXiv:2206.05825}, 2022.

\bibitem[Sutton et~al.(1999)Sutton, McAllester, Singh, and Mansour]{sutton1999policy}
Sutton, R.~S., McAllester, D., Singh, S., and Mansour, Y.
\newblock Policy gradient methods for reinforcement learning with function approximation.
\newblock \emph{Advances in neural information processing systems}, 12, 1999.

\bibitem[Swamy et~al.(2024)Swamy, Dann, Kidambi, Wu, and Agarwal]{swamy2024minimaximalist}
Swamy, G., Dann, C., Kidambi, R., Wu, Z.~S., and Agarwal, A.
\newblock A minimaximalist approach to reinforcement learning from human feedback.
\newblock \emph{arXiv preprint arXiv:2401.04056}, 2024.

\bibitem[Tajwar et~al.(2024)Tajwar, Singh, Sharma, Rafailov, Schneider, Xie, Ermon, Finn, and Kumar]{tajwar2024preference}
Tajwar, F., Singh, A., Sharma, A., Rafailov, R., Schneider, J., Xie, T., Ermon, S., Finn, C., and Kumar, A.
\newblock Preference fine-tuning of llms should leverage suboptimal, on-policy data.
\newblock \emph{arXiv preprint arXiv:2404.14367}, 2024.

\bibitem[Tomar et~al.(2020)Tomar, Shani, Efroni, and Ghavamzadeh]{tomar2020mirror}
Tomar, M., Shani, L., Efroni, Y., and Ghavamzadeh, M.
\newblock Mirror descent policy optimization.
\newblock \emph{arXiv preprint arXiv:2005.09814}, 2020.

\bibitem[Touvron et~al.(2023)Touvron, Martin, Stone, Albert, Almahairi, Babaei, Bashlykov, Batra, Bhargava, Bhosale, et~al.]{touvron2023LLaMA}
Touvron, H., Martin, L., Stone, K., Albert, P., Almahairi, A., Babaei, Y., Bashlykov, N., Batra, S., Bhargava, P., Bhosale, S., et~al.
\newblock Llama 2: Open foundation and fine-tuned chat models.
\newblock \emph{arXiv preprint arXiv:2307.09288}, 2023.

\bibitem[Tran et~al.(2023)Tran, Glaze, and Hancock]{viethoangtranduong}
Tran, H., Glaze, C., and Hancock, B.
\newblock Iterative dpo alignment.
\newblock Technical report, Snorkel AI, 2023.

\bibitem[Wang et~al.(2023)Wang, Jiang, Yang, Liu, and Chen]{wang2023beyond}
Wang, C., Jiang, Y., Yang, C., Liu, H., and Chen, Y.
\newblock Beyond reverse kl: Generalizing direct preference optimization with diverse divergence constraints.
\newblock \emph{arXiv preprint arXiv:2309.16240}, 2023.

\bibitem[Wang et~al.(2024{\natexlab{a}})Wang, Xiong, Xie, Zhao, and Zhang]{wang2024interpretable}
Wang, H., Xiong, W., Xie, T., Zhao, H., and Zhang, T.
\newblock Interpretable preferences via multi-objective reward modeling and mixture-of-experts.
\newblock \emph{arXiv preprint arXiv:2406.12845}, 2024{\natexlab{a}}.

\bibitem[Wang et~al.(2024{\natexlab{b}})Wang, Ma, Chen, Meng, Han, Xiao, Zhang, Huo, Su, and Yang]{wang2024magnetic}
Wang, M., Ma, C., Chen, Q., Meng, L., Han, Y., Xiao, J., Zhang, Z., Huo, J., Su, W.~J., and Yang, Y.
\newblock Magnetic preference optimization: Achieving last-iterate convergence for language models alignment.
\newblock \emph{arXiv preprint arXiv:2410.16714}, 2024{\natexlab{b}}.

\bibitem[Wang et~al.(2022)Wang, Zheng, He, Chen, and Zhou]{wang2022diffusion}
Wang, Z., Zheng, H., He, P., Chen, W., and Zhou, M.
\newblock Diffusion-gan: Training gans with diffusion.
\newblock \emph{arXiv preprint arXiv:2206.02262}, 2022.

\bibitem[Williams(1992)]{williams1992simple}
Williams, R.~J.
\newblock Simple statistical gradient-following algorithms for connectionist reinforcement learning.
\newblock \emph{Machine learning}, 8:\penalty0 229--256, 1992.

\bibitem[Wu et~al.(2024)Wu, Sun, Yuan, Ji, Yang, and Gu]{wu2024self}
Wu, Y., Sun, Z., Yuan, H., Ji, K., Yang, Y., and Gu, Q.
\newblock Self-play preference optimization for language model alignment.
\newblock \emph{arXiv preprint arXiv:2405.00675}, 2024.

\bibitem[Xie et~al.(2024)Xie, Foster, Krishnamurthy, Rosset, Awadallah, and Rakhlin]{xie2024exploratory}
Xie, T., Foster, D.~J., Krishnamurthy, A., Rosset, C., Awadallah, A., and Rakhlin, A.
\newblock Exploratory preference optimization: Harnessing implicit q*-approximation for sample-efficient rlhf.
\newblock \emph{arXiv preprint arXiv:2405.21046}, 2024.

\bibitem[Xiong et~al.(2024)Xiong, Dong, Ye, Wang, Zhong, Ji, Jiang, and Zhang]{xiong2024iterative}
Xiong, W., Dong, H., Ye, C., Wang, Z., Zhong, H., Ji, H., Jiang, N., and Zhang, T.
\newblock Iterative preference learning from human feedback: Bridging theory and practice for rlhf under kl-constraint.
\newblock In \emph{Forty-first International Conference on Machine Learning}, 2024.

\bibitem[Xu et~al.(2023)Xu, Lee, Sukhbaatar, and Weston]{xu2023some}
Xu, J., Lee, A., Sukhbaatar, S., and Weston, J.
\newblock Some things are more cringe than others: Preference optimization with the pairwise cringe loss.
\newblock \emph{arXiv preprint arXiv:2312.16682}, 2023.

\bibitem[Ye et~al.(2024)Ye, Xiong, Zhang, Dong, Jiang, and Zhang]{ye2024online}
Ye, C., Xiong, W., Zhang, Y., Dong, H., Jiang, N., and Zhang, T.
\newblock Online iterative reinforcement learning from human feedback with general preference model.
\newblock \emph{Advances in Neural Information Processing Systems}, 37:\penalty0 81773--81807, 2024.

\bibitem[Zhang et~al.(2024{\natexlab{a}})Zhang, Yu, Peng, Song, Tian, Huo, Jiang, Mi, and Yu]{zhang2024iterativenashpolicyoptimization}
Zhang, Y., Yu, D., Peng, B., Song, L., Tian, Y., Huo, M., Jiang, N., Mi, H., and Yu, D.
\newblock Iterative nash policy optimization: Aligning llms with general preferences via no-regret learning, 2024{\natexlab{a}}.
\newblock URL \url{https://arxiv.org/abs/2407.00617}.

\bibitem[Zhang et~al.(2024{\natexlab{b}})Zhang, Zhang, Wu, Xu, and Gu]{zhang2024beyond}
Zhang, Y., Zhang, G., Wu, Y., Xu, K., and Gu, Q.
\newblock Beyond bradley-terry models: A general preference model for language model alignment.
\newblock \emph{arXiv preprint arXiv:2410.02197}, 2024{\natexlab{b}}.

\bibitem[Zheng et~al.(2023)Zheng, Chiang, Sheng, Zhuang, Wu, Zhuang, Lin, Li, Li, Xing, et~al.]{zheng2023judging}
Zheng, L., Chiang, W.-L., Sheng, Y., Zhuang, S., Wu, Z., Zhuang, Y., Lin, Z., Li, Z., Li, D., Xing, E., et~al.
\newblock Judging llm-as-a-judge with mt-bench and chatbot arena.
\newblock \emph{Advances in Neural Information Processing Systems}, 36:\penalty0 46595--46623, 2023.

\bibitem[Zhu et~al.(2018)Zhu, Lu, Zheng, Guo, Zhang, Wang, and Yu]{zhu2018texygen}
Zhu, Y., Lu, S., Zheng, L., Guo, J., Zhang, W., Wang, J., and Yu, Y.
\newblock Texygen: A benchmarking platform for text generation models.
\newblock In \emph{The 41st international ACM SIGIR conference on research \& development in information retrieval}, pp.\  1097--1100, 2018.

\end{thebibliography}
\bibliographystyle{icml2026}

%%%%%%%%%%%%%%%%%%%%%%%%%%%%%%%%%%%%%%%%%%%%%%%%%%%%%%%%%%%%%%%%%%%%%%%%%%%%%%%
%%%%%%%%%%%%%%%%%%%%%%%%%%%%%%%%%%%%%%%%%%%%%%%%%%%%%%%%%%%%%%%%%%%%%%%%%%%%%%%
% APPENDIX
%%%%%%%%%%%%%%%%%%%%%%%%%%%%%%%%%%%%%%%%%%%%%%%%%%%%%%%%%%%%%%%%%%%%%%%%%%%%%%%
%%%%%%%%%%%%%%%%%%%%%%%%%%%%%%%%%%%%%%%%%%%%%%%%%%%%%%%%%%%%%%%%%%%%%%%%%%%%%%%
\newpage
\appendix

\clearpage
\appendix
\onecolumn

\section{Proofs}
\label{append:proof}
In this section, we provide detailed assumptions, derivations and proofs of propositions.

\begin{assumption}[\textbf{Relative Convexity of $R$ w.r.t. entropy function}]
We assume the regularization function $R$ of policy $\pi$ is a $1$-strongly convex relative to entropy function. In other words, $\forall \pi, \pi' \in \Delta^{\mathcal{X}}_{\mathcal{Y}}$, and $\psi(\pi)=\langle \pi, \log \pi\rangle $, we have
\begin{align}
\langle \partial_\pi R(\pi)- \partial_\pi R(\pi') , \pi - \pi' \rangle \geq \langle {\partial_\pi \psi(\pi) - \partial_\pi \psi(\pi')}, \pi - \pi' \rangle.
\end{align}
\label{assumption:reg}
\vspace{-2em}
\end{assumption}
% The Assumption \ref{assumption:reg} restricts the scope of regularization $R$ that can have convergence guarantees, though there is still a broad class of $R$ that can provide RSPO with convergence guarantees. As for the divergence in our experiments, including the linear combinations of divergences, only the Forward KL divergence violates the assumption. While interestingly, Forward KL divergence leads to response length reduction. Thus we propose to linearly combine it with reverse KL divergence for objective length-controlled performance.
Assumption \ref{assumption:reg} constrains the class of regularization terms $R$ under which theoretical convergence guarantees can be established. Nonetheless, a broad family of divergences still satisfies this assumption, allowing RSPO to retain convergence properties in a wide range of settings. Among the divergences used in our experiments—including linear combinations—only the forward KL divergence violates this assumption. Interestingly, however, forward KL regularization is empirically observed to reduce response length. To leverage this desirable property while preserving theoretical validity, we propose a linear combination of forward and reverse KL divergences, enabling effective length-controlled generation without sacrificing convergence guarantees, and obtains the best generation quality empirically.

\begin{assumption}[\textbf{On-Policy Assumption}]
\label{assump:on-policy}
We assume that the proposed Regularized Self-Play Policy Optimization (RSPO) method operates in an on-policy manner at each iteration $t$. Formally, we require that sampling $y \sim \pi_\theta$ is well-approximated by sampling $y \sim \pi_t$.
\end{assumption}

\subsection{Proof of the Existence of regularized Nash Equilibrium}
\label{append:rpm_exists}
\begin{proposition}
Nash Equilibrium in the regularized game in \Cref{eq:rpm} exists, and it is unique.
\end{proposition}

\begin{proof}We prove the existence of in this section, largely following the idea of proving the existence of KL regularized Nash Equilibrium by \citet{munos2023nash}.

Since the utility $u(\pi, \pi')$ is linear in $\pi$ and $\pi'$, and the regularization function is assumed to be convex (Assumption \ref{assumption:reg}), the regularized preference is concave in $\pi$ and convex in $\pi'$. Therefore, the existence and the uniqueness of a regularized Nash Equilibrium in \Cref{eq:rpm} can be directly derived from the minimax theorem \citep{sion1958general}.
\end{proof}

\subsection{Proof of Relative Convexity of Combined Regularizers}
\begin{lemma}
The linear combination of Forward and Reverse KL divergence regularization \( R = aKL(\pi\|\mu) + bKL(\mu\|\pi) \), for $a,b>0$ satisfies that:
\begin{align}
\langle \partial_\pi R(\pi) - \partial_\pi R(\pi'),\, \pi - \pi' \rangle
\ge
a\,\langle \partial_\pi \log \pi - \partial_{\pi'} \log \pi',\, \pi - \pi' \rangle 
\end{align}

\end{lemma}
\begin{proof}
From the definition \( R = aKL(\pi\|\mu) + bKL(\mu\|\pi) \), we decompose the left–hand side:
\begin{align*}
&\langle \partial_\pi R(\pi) - \partial_\pi R(\pi'),\, \pi - \pi' \rangle \\
&= a \langle \partial_\pi KL(\pi\|\mu) - \partial_\pi KL(\pi'\|\mu),\, \pi - \pi' \rangle + b \langle \partial_\pi KL(\mu\|\pi) - \partial_\pi KL(\mu\|\pi'),\, \pi - \pi' \rangle \\
&= a \langle \partial_\pi \log \pi - \partial_{\pi'} \log \pi',\, \pi - \pi' \rangle  + b \langle \partial_\pi KL(\mu\|\pi) - \partial_\pi KL(\mu\|\pi'),\, \pi - \pi' \rangle \\
&= a \langle \partial_\pi \psi(\pi) - \partial_\pi \psi(\pi'),\, \pi - \pi' \rangle + b \langle \partial_\pi KL(\mu\|\pi) - \partial_\pi KL(\mu\|\pi'),\, \pi - \pi' \rangle .
\end{align*}

Besides,
\begin{align*}
\langle \partial_\pi KL(\mu\|\pi) - \partial_\pi KL(\mu\|\pi'),\, \pi - \pi' \rangle
&= \sum_y (\pi(y) - \pi'(y))\!\left(-\frac{\mu(y)}{\pi(y)} + \frac{\mu(y)}{\pi'(y)}\right) \\
&= \sum_y \frac{\mu(y)\,(\pi'(y) - \pi(y))^2}{\pi(y)\pi'(y)} \ge 0.
\end{align*}

Therefore,
\[
\langle \partial_\pi R(\pi) - \partial_\pi R(\pi'),\, \pi - \pi' \rangle
\ge
a\,\langle \partial_\pi \log \pi - \partial_{\pi'} \log \pi',\, \pi - \pi' \rangle .
\]
\end{proof}

\subsection{Proof of Equivalence between MD and RSPO}
\label{append:nashmd_proof}

\begin{proposition}
Nash-MD and Online Mirror Descent \citep[Section~6]{munos2023nash} can be seen as instances of Regularized Self-Play Policy Optimization (RSPO) (\Cref{eq:RSPO}).
\end{proposition}

\begin{proof}
In this section, we first provide derivations of how Nash-MD is equivalent to RSPO: 
\begin{align}
\nabla_\theta \mathcal{L}_{\text{Nash-MD}} = \nabla_\theta \mathcal{L}_{\text{RSPO}}\big(\theta;G=\mathbb{P}(y\succ  \pi_t^\mu), B=\frac{1}{2}, R=D_{\text{KL}}(\pi_\theta\|\mu) \big)
\end{align}
On one hand, Nash-MD practical loss \citep[Section~7]{munos2023nash} is defined as
\begin{align}
&\nabla_{\theta} \mathcal{L}_{\text{Nash-MD}}(\theta) \\ 
=&\mathbb{E}_{\substack{y\sim \pi_{\theta},  \\ 
y'\sim \pi^{\mu}_t}}\Big[ \nabla_\theta \log \pi_\theta(y) \Big( \mathbb{P}(y \succ \pi_t^\mu) - \frac{1}{2} - \tau \log \frac{\pi_\theta(y)}{\mu(y)} \Big) \Big] \label{eq:nash_origin_0} \\ 
=&\mathbb{E}_{\substack{y\sim \pi_{\theta},  \\ 
y'\sim \pi^{\mu}_t}}\Big[ \nabla_\theta \log \pi_\theta(y) \Big( \mathbb{P}(y \succ \pi_t^\mu) - \frac{1}{2} - \tau \log \frac{\pi_\theta(y)}{\mu(y)} + 2\tau \log \frac{\pi_t(y)}{\mu(y)} \Big) \Big] \label{eq:nash_origin_1} \\
=&\mathbb{E}_{\substack{y\sim \pi_{\theta},  \\ 
y'\sim \pi^{\mu}_t}}\Big[ \nabla_\theta \log \pi_\theta(y) \Big( \mathbb{P}(y \succ \pi_t^\mu) - \frac{1}{2} - 2\tau \log \frac{\pi_\theta(y)}{\pi_t(y)} + \tau \log \frac{\pi_\theta(y)}{\mu(y)} \Big) \Big] \\
=&\mathbb{E}_{\substack{y\sim \pi_{\theta},  \\ 
y'\sim \pi^{\mu}_t}}\Big[ \nabla_\theta \log \pi_\theta(y) \Big( \mathbb{P}(y \succ \pi_t^\mu) - \frac{1}{2} - 2\tau \log \frac{\pi_\theta(y)}{\pi_t(y)} \Big) \Big] + \tau \nabla_\theta \mathbb{E}_{\substack{y\sim \pi_{\theta},  \\ 
y'\sim \pi^{\mu}_t}}\Big[ \log \frac{\pi_\theta(y)}{\mu(y)} \Big]  \\
=&\mathbb{E}_{\substack{y\sim \pi_{\theta},  \\ 
y'\sim \pi^{\mu}_t}}\Big[ \nabla_\theta \log \pi_\theta(y) \Big( \mathbb{P}(y \succ \pi_t^\mu) - \frac{1}{2} - 2\tau \log \frac{\pi_\theta(y)}{\pi_t(y)} \Big) \Big] + \tau \nabla_\theta D_{\text{KL}}(\pi_\theta \| \mu) \label{eq:nash_origin_2} \\
=& 2\tau^2  \mathbb{E}_{\substack{y\sim \pi_{\theta},  \\ 
y'\sim \pi^{\mu}_t}}\Big[ \nabla_\theta \Big( \log \frac{\pi_\theta(y)}{\pi_t(y)} - \frac{1}{2\tau} \Big( \mathbb{P}(y \succ \pi_t^\mu) - \frac{1}{2} \Big) \Big)^2 \Big] + \tau \nabla_\theta D_{\text{KL}}(\pi_\theta \| \mu) \label{eq:nash_origin_3} \\
=& 2\tau^2 \nabla_\theta \mathbb{E}_{\substack{y\sim \pi_{t},  \\ 
y'\sim \pi^{\mu}_t}}\Big[ \log \frac{\pi_\theta(y)}{\pi_t(y)} - \frac{1}{2\tau} \Big( \mathbb{P}(y \succ \pi_t^\mu) - \frac{1}{2} \Big) \Big]^2 + \tau \nabla_\theta D_{\text{KL}}(\pi_\theta \| \mu) \label{eq:nash_origin_4}.
\end{align}

\Cref{eq:nash_origin_0} is the definition of Nash-MD policy gradient. \Cref{eq:nash_origin_1} holds because the additional term satisfies that $\mathbb{E}_{\substack{y\sim \pi_{\theta},  \\ 
y'\sim \pi^{\mu}_t}}\big[ \nabla_\theta \log \pi_\theta(y) \big( \log \frac{\pi_t(y)}{\mu(y)} \big) \big]=\nabla_\theta \mathbb{E}_{\substack{y\sim \pi_{\theta},  \\ 
y'\sim \pi^{\mu}_t}}\big[  \log \frac{\pi_t(y)}{\mu(y)} \big] = \nabla_\theta  \log \frac{\pi_t(y)}{\mu(y)} =0$. \Cref{eq:nash_origin_2} holds due to the definition of reverse KL divergence. \Cref{eq:nash_origin_3} is derived by computing the integral of $\log \pi_\theta(y) \big( \mathbb{P}(y \succ \pi_t^\mu) - \frac{1}{2} - 2\tau \log \frac{\pi_\theta(y)}{\pi_t(y)} \big)$. 
% And \Cref{eq:nash_origin_4} holds due to the assumption of $y \sim \pi_\theta \Leftrightarrow y \sim \pi_t$ (Assumption \ref{assumption:on-policy}). 

On the other hand, we show that Nash-MD and OMD can also be generalized by RSPO \textit{without} external regularization, such that we can add additional regularization to existing regularized self-play methods. Nash-MD practical loss \citep[Section~7]{munos2023nash} is defined as
% \begin{align}
% \nabla_{\theta} \mathcal{L}_{\text{Nash-MD}}(\theta)
% &=\mathbb{E}_{\substack{y\sim \pi_{\theta}, \nonumber \\ 
% y'\sim \pi^{\mu}_t}}\Big[ \nabla_\theta \log \pi_\theta(y) \Big( \mathbb{P}(y \succ y') - \frac{1}{2} - \tau \log \frac{\pi_\theta(y)}{\mu(y)} \Big) \Big] \nonumber \\
% &=\mathbb{E}_{\substack{y\sim \pi_{\theta}, \nonumber \\ y'\sim \pi^{\mu}_t}}\Big[ \nabla_\theta \log \pi_\theta(y) \Big( \mathbb{P}(y \succ y') - \frac{1}{2} - \tau \log \frac{\pi_\theta(y)}{\pi_t(y)} - \tau \log \frac{\pi_t(y)}{\mu(y)} \Big) \Big]  \nonumber \\
% &=\mathbb{E}_{y\sim \pi_{\theta}}\Big[ \nabla_\theta \log \pi_\theta(y) \Big( \mathbb{P}(y \succ \pi^{\mu}_t)  - \frac{1}{2} - \tau \log \frac{\pi_\theta(y)}{\pi_t(y)} - \tau \log \frac{\pi_t(y)}{\mu(y)} \Big) \Big] \nonumber \\
% &=\mathbb{E}_{y\sim \pi_{t}}\Big[ \nabla_\theta \log \pi_\theta(y) \Big( \mathbb{P}(y \succ \pi^{\mu}_t)  - \frac{1}{2} - \tau \log \frac{\pi_\theta(y)}{\pi_t(y)} - \tau \log \frac{\pi_t(y)}{\mu(y)} \Big) \Big] \nonumber \\
% & = \nabla_\theta \mathbb{E}_{y\sim \pi_{t}}\Big[ \tau \log \frac{\pi_\theta(y)}{\pi_t(y)} - \Big( \mathbb{P}(y \succ \pi^{\mu}_t) - \tau \log \frac{\pi_t(y)}{\mu(y)} - \frac{1}{2} \Big) \Big]^2 / 2\nonumber \\
% & = \tau^2 \nabla_\theta \mathbb{E}_{y\sim \pi_{t}}\Big[ \log \frac{\pi_\theta(y)}{\pi_t(y)} - \frac{1}{\tau} \Big( \mathbb{P}(y \succ \pi^{\mu}_t) - \tau \log \frac{\pi_t(y)}{\mu(y)} - \frac{1}{2} \Big) \Big]^2 / 2.
% \end{align}
\begin{align}
\nabla_{\theta} \mathcal{L}_{\text{Nash-MD}}(\theta) 
&=\mathbb{E}_{\substack{y\sim \pi_{\theta},  \\ 
y'\sim \pi^{\mu}_t}}\Big[ \nabla_\theta \log \pi_\theta(y) \Big( \mathbb{P}(y \succ y') - \frac{1}{2} - \tau \log \frac{\pi_\theta(y)}{\mu(y)} \Big) \Big] \label{eq:nash1} \\
&=\mathbb{E}_{\substack{y\sim \pi_{\theta},  \\ y'\sim \pi^{\mu}_t}}\Big[ \nabla_\theta \log \pi_\theta(y) \Big( \mathbb{P}(y \succ y') - \frac{1}{2} - \tau \log \frac{\pi_\theta(y)}{\pi_t(y)} - \tau \log \frac{\pi_t(y)}{\mu(y)} \Big) \Big]  \label{eq:nash2} \\
&=\mathbb{E}_{y\sim \pi_{\theta}}\Big[ \nabla_\theta \log \pi_\theta(y) \Big( \mathbb{P}(y \succ \pi^{\mu}_t)  - \frac{1}{2} - \tau \log \frac{\pi_\theta(y)}{\pi_t(y)} - \tau \log \frac{\pi_t(y)}{\mu(y)} \Big) \Big] \label{eq:nash3} \\
&=\mathbb{E}_{y\sim \pi_{t}}\Big[ \nabla_\theta \log \pi_\theta(y) \Big( \mathbb{P}(y \succ \pi^{\mu}_t)  - \frac{1}{2} - \tau \log \frac{\pi_\theta(y)}{\pi_t(y)} - \tau \log \frac{\pi_t(y)}{\mu(y)} \Big) \Big] \label{eq:nash4} \\
& = \nabla_\theta \mathbb{E}_{y\sim \pi_{t}}\Big[ \tau \log \frac{\pi_\theta(y)}{\pi_t(y)} - \Big( \mathbb{P}(y \succ \pi^{\mu}_t) - \tau \log \frac{\pi_t(y)}{\mu(y)} - \frac{1}{2} \Big) \Big]^2 / 2 \label{eq:nash5} \\
& = \tau^2 \nabla_\theta \mathbb{E}_{y\sim \pi_{t}}\Big[ \log \frac{\pi_\theta(y)}{\pi_t(y)} - \frac{1}{\tau} \Big( \mathbb{P}(y \succ \pi^{\mu}_t) - \tau \log \frac{\pi_t(y)}{\mu(y)} - \frac{1}{2} \Big) \Big]^2 / 2. \label{eq:nash6}
\end{align}
\Cref{eq:nash1} is the definition of practical Nash-MD loss \citep[Section~7]{munos2023nash}. \Cref{eq:nash2} holds by adding an subtracting the same element $\log \pi_t(y)$. \Cref{eq:nash3} holds due to $\mathbb{E}_{y'\sim \pi_t^\mu}[\mathbb{P}(y \succ y')]= \mathbb{P}(y \succ \pi_t^\mu)$. The learning rate $\eta$ is originally omitted in the paper \citep{munos2023nash}. Here Nash-MD is generalized by $\mathcal{L}_{\text{RSPO}}$ with $\eta=\tfrac{1}{\tau}$ and $R=0$. 

OMD is to execute
$\arg\max_{\pi}  \eta \mathbb{E}_{y \sim \pi} \left[ \mathbb{P}(y \succ \pi_t) - \tau \log \frac{\pi_t(y)}{\mu(y)} \right] - \text{KL}(\pi, \pi_t)$.
Therefore, the loss function of the OMD update satisfies
% \begin{align}
% \nabla_{\theta} \mathcal{L}_{\text{OMD}}(\theta) &= -\nabla_\theta  \eta \mathbb{E}_{y \sim \pi_\theta} \left[ \mathbb{P}(y \succ \pi_t) - \tau \log \frac{\pi_t(y)}{\mu(y)} \right] + D_{\text{KL}}(\pi_\theta, \pi_t) \nonumber \\
% &= -\nabla_\theta  \eta \mathbb{E}_{y \sim \pi_\theta} \left[ \mathbb{P}(y \succ \pi_t) - \tau \log \frac{\pi_t(y)}{\mu(y)} - \log \frac{\pi_\theta}{\pi_t} \right] \nonumber \\
% &=  \eta \mathbb{E}_{y \sim \pi_\theta} \left[ -\nabla_\theta \log \pi_\theta \Big( \mathbb{P}(y \succ \pi_t) - \tau \log \frac{\pi_t(y)}{\mu(y)} - \log \frac{\pi_\theta}{\pi_t} \Big) \right] \nonumber \\
% &=  \frac{\eta}{2} \cdot \mathbb{E}_{y \sim \pi_\theta} \left[ \nabla_\theta \Big( \mathbb{P}(y \succ \pi_t) - \tau \log \frac{\pi_t(y)}{\mu(y)} - \log \frac{\pi_\theta(y)}{\pi_t(y)} \Big)^2 \right] \nonumber \\
% &=  \frac{\eta}{2} \cdot \mathbb{E}_{y \sim \pi_t} \left[ \nabla_\theta  \log \frac{\pi_\theta(y)}{\pi_t(y)} - \Big( \mathbb{P}(y \succ \pi_t) - \tau \log \frac{\pi_t(y)}{\mu(y)} \Big) \right]^2.
% \end{align}
\begin{align}
\nabla_{\theta} \mathcal{L}_{\text{OMD}}(\theta) &= -\nabla_\theta  \eta \mathbb{E}_{y \sim \pi_\theta} \left[ \mathbb{P}(y \succ \pi_t) - \tau \log \frac{\pi_t(y)}{\mu(y)} \right] + D_{\text{KL}}(\pi_\theta, \pi_t) \label{eq:omd_gradient_1} \\
&= -\nabla_\theta  \eta \mathbb{E}_{y \sim \pi_\theta} \left[ \mathbb{P}(y \succ \pi_t) - \tau \log \frac{\pi_t(y)}{\mu(y)} - \log \frac{\pi_\theta}{\pi_t} \right] \label{eq:omd_gradient_2} \\
&=  \eta \mathbb{E}_{y \sim \pi_\theta} \left[ -\nabla_\theta \log \pi_\theta \Big( \mathbb{P}(y \succ \pi_t) - \tau \log \frac{\pi_t(y)}{\mu(y)} - \log \frac{\pi_\theta}{\pi_t} \Big) \right] \label{eq:omd_gradient_3} \\
&=  \frac{\eta}{2} \cdot \mathbb{E}_{y \sim \pi_\theta} \left[ \nabla_\theta \Big( \mathbb{P}(y \succ \pi_t) - \tau \log \frac{\pi_t(y)}{\mu(y)} - \log \frac{\pi_\theta(y)}{\pi_t(y)} \Big)^2 \right] \label{eq:omd_gradient_4} \\
&=  \frac{\eta}{2} \cdot \mathbb{E}_{y \sim \pi_t} \left[ \nabla_\theta  \log \frac{\pi_\theta(y)}{\pi_t(y)} - \Big( \mathbb{P}(y \succ \pi_t) - \tau \log \frac{\pi_t(y)}{\mu(y)} \Big) \right]^2. \label{eq:omd_gradient_5}
\end{align}
\Cref{eq:omd_gradient_1} holds because the OMD update is equivalent to descending negative gradient of the feedback $\eta \mathbb{E}_{y \sim \pi} \left[ \mathbb{P}(y \succ \pi_t) - \tau \log \frac{\pi_t(y)}{\mu(y)} \right] - \text{KL}(\pi, \pi_t)$. \Cref{eq:omd_gradient_2} holds due to the definition of $D_{\text{KL}}$. \Cref{eq:omd_gradient_3} holds by conducting differentiation on multiplication. The remaining equations hold due to simple algebra. Therefore, OMD can also be generalized by RSPO with $G=\mathbb{P}(y \succ \pi_t) - \tau \log \frac{\pi_t(y)}{\mu(y)}$ and without external regularization.
\end{proof}

\subsection{Proof of Proposition \ref{prop:RSPO}}
\label{append:rspo_proof}

    % ]]]]]]gin{align}
% &\min_{\pi_\theta} \mathbb{E}_{y \sim \pi_t} \Big[ \log \frac{\pi_{\theta}(y)}{\pi_t(y)} - \eta \Big(  \mathbb{P}({y} \succ \pi_t) + \tau g(y; \pi_\theta, \mu) - \log Z(\pi_t) \Big) \Big]^2 \label{eq:qudratic_to_minp} \\
% \Leftrightarrow  &\min_{\pi_\theta} -\eta \mathbb{P}(\pi_\theta\succ \pi_{t}) + \tau' D_{f}(\mu || \pi_\theta) + D_{\text{KL}}(\pi_\theta|| \pi_t)
% \end{align}
% \end{lemma}

% \begin{proof}

% \end{proof}

% The policy updated with external regularization (Equation \ref{eq:rsp}) converges to the same solution as the IR-SP.

%\todoq{this assumption is very strange, you're comparing two vectors? because the denominator is a vector?}

\begin{customprop}{\ref{prop:RSPO}}
If $R(\cdot, \mu)$ is $1$-strongly convex relative to $\psi$ (Assumption \ref{assumption:reg}), policy updated by GMMD in \Cref{eq:gmmd} has last-iterate convergence to the following Nash Equilibrium of a regularized game:
\begin{align}
\max_{\pi} \min_{\pi'} U(\pi; \pi') - \tau R(\pi, \mu) + \tau R(\pi', \mu).
\end{align}
\end{customprop}

\begin{proof}

According to \Cref{eq:gmmd}, GMMD is equivalent to the Algorithm 3.1 in \citet{sokota2022unified}:
\begin{align}
z_{t+1} = \arg\min_{z \in \mathcal{Z}} \eta \left( \langle F(z_t), z \rangle + \alpha g(z) \right) + B_{\psi}(z; z_t),
\label{eq:mmd}
\end{align}
where in our setting, $z=\pi$ is the LLM policy, $F(z_t)=-\partial_\pi U(\pi; \pi_t)$ is the vector of negative partial derivatives of preference w.r.t. each component of $\pi$, $\alpha =\tau $, $g(z)$ is the regularizer $R(\pi)$, and we set $\psi(z) = z\log z$ to convert the Bregman divergence $B_{\psi}$ to KL divergence. Here $U(\pi; \pi_t)$ is treated as a function of vector form of $\pi$, i.e., $[\pi^0\ \pi^1\ \cdots\ \pi^{|\mathcal{Y}|}]$, thus the gradient is a vector gradient where $\partial_\pi U(\pi; \pi_t) = [\partial U /  \partial \pi^0\ \ \partial U /  \partial \pi^1 \ \ \cdots \ \ \  \partial U / \partial \pi^{|\mathcal{Y}|}]$. 

We then show that in our setting the following assumptions are satisfied. $F$ satisfies that for any $z, z'$, $\langle F(z) - F(z'), z- z' \rangle =0$ since $U$ is linear in $\pi$, and $F(z) - F(z') = -\partial_\pi U(\pi; \pi_t) + \partial_\pi U(\pi';\pi_t)=0$. Therefore, $F$ is Monotone and $L$-smooth. According to Assumption \ref{assumption:reg}, $g$ is $1$-strongly convex relative to $\psi$, i.e., $g(z) \geq g(z') + \frac{g'(z)}{\psi'(z)}(\psi(z) - \psi(z'))$.

Given the assumptions above, according to the Theorem 3.4. in \citet{sokota2022unified}, the update rule defined in \Cref{eq:mmd} has a last-iterate convergence guarantee to a policy $\pi^*$, which is the solution to the variational inequality problem $\text{VI}(\Delta^{\mathcal{X}}_{\mathcal{Y}}, F+ \alpha \partial g)$, i.e., $\pi^*$ satisfies
\begin{align}
% \langle -\mathbb{P}(y \succ \pi^*) + \tau \nabla R(\pi^*), \pi - \pi^* \rangle & \geq 0 \quad \forall \pi \in \Delta^{\mathcal{X}}_{\mathcal{Y}}. \\
\langle \partial_\pi \Big( -U(\pi; \pi^*) + \tau R(\pi, \mu) \Big) \mid_{\pi=\pi^*}, \pi - \pi^* \rangle &\geq 0, \quad \forall \pi \in \Delta^{\mathcal{X}}_{\mathcal{Y}} \nonumber \\
\Leftrightarrow  \langle \partial_\pi \Big( -U(\pi; \pi^*) + \tau R(\pi, \mu) - \tau R(\pi^*, \mu) \Big) \mid_{\pi=\pi^*}, \pi - \pi^* \rangle &\geq 0, \quad \forall \pi \in \Delta^{\mathcal{X}}_{\mathcal{Y}}.
\label{eq:vi_final}
\end{align}
\Cref{eq:vi_final} indicates that moving from $\pi^*$ towards any direction $\pi - \pi^*$ can not increase the value of the objective preference model $U(\pi; \pi^*) - \tau R(\pi, \mu) + \tau R(\pi^*, \mu)$ at the point of $\pi=\pi^*$, given the opponent is $\pi^*$. Therefore, by symmetry, $\pi^*$ is the Nash Equilibrium of the regularized preference model:
$
\max_\pi \min_{\pi'} U(\pi;\pi') - \tau R(\pi, \mu) + \tau R(\pi', \mu).$ The specific guarantee according to \citet{sokota2022unified} is that, for Nash Equilibrium policy $\pi^*$:
\begin{align}
\mathrm{KL}(\pi^* \,\|\, \pi_{t+1})
\le
\left(\frac{1}{1+\eta \tau}\right)^t
\mathrm{KL}(\pi^* \,\|\, \mu).
\end{align}

% According to Proposition 2.4. in \citep{sokota2022unified}, only when $R = B_{\psi}= D_{\text{KL}}$, solving $\text{VI}(\Delta^{\mathcal{X}}_{\mathcal{Y}}, F+ \alpha \nabla g)$ is equivalent to solving the two-player general-sum game as follows via Magnetic Mirror Descent:
% \begin{align}
% \max_{\pi} \min_{\pi'} \eta \mathbb{P}(\pi \succ \pi' ) - \tau R(\pi) + R(\pi').
% \end{align}

% When the payoff or utility function (loss when opponent is fixed) is $U(\pi, \pi_t) = u(\pi; \pi_t) - \tau  g(\pi, \mu) + \tau  g(\pi_t, \mu)$, Online Mirror Descent is:
% \begin{align}
% &\max_{\pi}  \mathbb{E}_{\pi}[ \eta \partial_\pi U(\pi, \pi_t)\mid_{\pi=\pi_t}] - D_{\text{KL}}(\pi || \pi_t) \\
% \Longleftrightarrow
% &\max_{\pi}  \mathbb{E}_{\pi}[ \eta \partial_\pi \big( u(\pi; \pi_t) - \tau  g(\pi, \mu) + \tau  g(\pi_t, \mu) \big) \mid_{\pi=\pi_t}] - D_{\text{KL}}(\pi || \pi_t) \\
% \Longleftrightarrow & \max_{\pi} \mathbb{E}_{\pi}[ \eta \mathbb{P}(y \succ \pi_t) - \tau \partial_\pi \big( g(\pi, \mu) \big) \mid_{\pi=\pi_t} ] - D_{\text{KL}}(\pi || \pi_t) \\
% \Longleftrightarrow & \max_{\pi}  \eta u(\pi; \pi_t) - \tau \mathbb{E}_{\pi}[ \partial_\pi g(\pi_t, \mu)]   - D_{\text{KL}}(\pi || \pi_t).
% \end{align}

% Thus RSPO is equivalent to OMD if $f$ satisfies that
% \begin{align}
% R(\pi) = \mathbb{E}_{\pi}[ \partial_\pi g(\pi_t, \mu)] + C.
% \end{align}
\end{proof}

\subsection{Proof of Corollary \ref{coro:rspo_converge}}
\label{append:rspo_converge}
\begin{proof}
We prove that RSPO in \Cref{eq:RSPO_gmmd} is equivalent to GMMD up to multiplying a constant to the gradient, leading to a regularized Nash Equilibrium.
We follow SPPO to replace the samples $y \sim \pi_\theta$ with $y \sim \pi_t$ directly since they are equivalent while computing the loss before updating, and rewrite the loss equivalent to GMMD: 
% The expectation in \Cref{eq:pg} requires samples from $\pi_\theta$. Since sampling responses from LLMs is more costly than standard RL, using samples from $\pi_\theta$ only once per iteration may lead to insufficient policy optimization, necessitating additional update iterations.
% $$ becomes:
\begin{align}
&\nabla_\theta\mathcal{L}_{\text{GMMD}}(\theta) \overset{\text{def}}{=} \mathbb{E}_{y \sim \pi_\theta}\bigg[\nabla_\theta \log \pi_\theta(y)\bigg( -\eta G(y;\pi_t) + \log \frac{\pi_\theta(y)}{ \pi_t(y)} + B \bigg)\bigg]  + \tau \nabla_\theta  R(\pi_\theta, \mu) \\
&= \nabla_\theta \Bigg(\frac{1}{2}   {\mathbb{E}_{{y \sim 
\pi_\theta}}\Big[  -\eta G(y;\pi_t) + \log \frac{\pi_\theta(y)}{ \pi_t(y)} {+ \eta B}\Big]^2}  + \tau R(\pi_\theta, \mu) \Bigg)\\ 
&= \nabla_\theta \Bigg(\frac{1}{2}   {\mathbb{E}_{{y \sim 
\pi_t}}\Big[  -\eta G(y;\pi_t) + \log \frac{\pi_\theta(y)}{ \pi_t(y)} {+ \eta B}\Big]^2}  + \tau R(\pi_\theta, \mu) \Bigg) = \frac{1}{2}\nabla_\theta \mathcal{L}_{\text{RSPO}}(\theta) 
 \label{eq:final_square_loss}
\end{align}
The first equation is the definition of GMMD in \Cref{eq:pg}. The second equation holds due to simple calculus. \Cref{eq:final_square_loss} holds due to the \Cref{assump:on-policy}.

Therefore, according to \Cref{eq:final_square_loss}, RSPO is the RL implementation of GMMD, since gradients of losses are equivalent up to multiplying a constant. Then we can derive the convergence guarantee of RSPO by instantiating $G=\mathbb{P}(y \succ \pi_t), B=\frac{1}{2}$ in GMMD
\begin{align}
&\nabla_\theta\mathcal{L}_{\text{RSPO}}(\theta; G=\mathbb{P}(y \succ \pi_t), B=\frac{1}{2}) \nonumber \\
&=  \nabla_\theta \Big( \mathbb{E}_{y \sim \pi_t} \Big[ \log \frac{\pi_{\theta}(y)}{\pi_t(y)} -  \eta \Big(   \mathbb{P}(y \succ \pi_t) - \frac{1}{2} \Big) \Big]^2 + \lambda 
 R(\pi_{\theta}, \mu) \Big) \label{eq:rspo_proof_2} \\
&=  \nabla_\theta \Big(  \langle \pi_t, , \big(  -\eta \partial_\pi \mathbb{P}(\pi \succ \pi_t) + \log \frac{\pi_\theta}{ \pi_t} {+ B}  \big)^2 \rangle + \lambda 
 R(\pi_{\theta}, \mu) \Big). \label{eq:rspo_proof_3}
% &=  2 \nabla_\theta  \langle \pi_t, , \big(  -\eta \nabla_\pi u(\pi_t; \pi_t) + \log \frac{\pi_\theta}{ \pi_t} {+ B}  \big)^2 \rangle \cdot \frac{1}{2}   + \tau \nabla_\theta R(\pi_\theta, \mu) \label{eq:rspo_proof_4} \\
% &=  2 \Big( \nabla_\theta  \mathbb{E}_{{y \sim 
% \pi_t}}[\big(  -\eta G(y, \pi_t) + \log \frac{\pi_\theta(y)}{ \pi_t(y)} {+ B}  \big)^2] \cdot \frac{1}{2}   
%  + \tau \nabla_\theta R(\pi_\theta, \mu) \label{eq:rspo_proof_5}\Big).
\end{align}
\Cref{eq:rspo_proof_2} holds due to definition. \Cref{eq:rspo_proof_3} holds by treating policy as a vector and rewrite the expectation in vector product form, and $\nabla_\pi \mathbb{P}(\pi \succ \pi_t) \mid_{\pi=\pi_t} = [\mathbb{P}(y^0 \succ \pi_t)\quad \mathbb{P}(y^1 \succ \pi_t) \quad \cdots \quad \mathbb{P}(y^{|\mathcal{Y}|} \succ \pi_t)]^T$, where $y^0, y^1, \cdots, y^{\mathcal{Y}}$ represent all possible values of $y$.
    % \item \Cref{eq:rspo_proof_4} holds due to that when $u=\mathbb{P}$, $\nabla_\pi u(\pi_t; \pi_t) = \nabla_\pi \mathbb{E}_{y \sim \pi} [\mathbb{P}(y \succ \pi_t)]\mid_{\pi=\pi_t}$, and $\tau=\frac{\lambda}{2}$.
% \Cref{eq:rspo_proof_5} holds by rewriting the form of dot product as expectation.

Since in GMMD $G\overset{\text{def}}{=} \partial_{\pi(y)} U(\pi;\pi')$ is set to $\mathbb{P}(y \succ \pi_t)$, $U=\mathbb{P}(\pi \succ \pi_t)$ is linear in $\pi$. Thus, according to Proposition \ref{prop:RSPO}, when the regularizer $R$ satisfies the relative convexity, updating policy following Algorithm \ref{alg:selfplay} with the loss function $\mathcal{L}_{\text{RSPO}}(\theta; G=\mathbb{P}(y \succ \pi_t), B=\frac{1}{2})$ has last-iterate convergence to the Nash Equilibrium of the regularized preference optimization game in \Cref{eq:rpm} by setting $u(\pi;\pi')=\mathbb{P}(\pi \succ \pi')$.
\end{proof}

\subsection{Proof of Proposition \ref{theo:reverse_kl}}
\label{sec:reverse_kl_square}

% Since $\nabla_{\theta} D_{\text{KL}}(\pi_{\theta}||\mu)=  \mathbb{E}_{\pi_{\theta}}[ \nabla_{\theta}\log {\pi_{\theta}(y)}/{\mu(y)}]^2/2$, when $R=D_{\text{KL}}(\pi_{\theta}||\mu)$, RSPO converges to the Nash Equilibrium of preference model:
% \begin{align}
% \mathbb{P}(\pi \succ \pi') -\tau D_{\text{KL}}(\pi||\mu) +\tau D_{\text{KL}}(\pi'||\mu).
% \end{align}
\begin{proof}
$\pi$ is parametrized by $\theta$, $\nabla_{\theta} D_{\text{KL}}(\pi||\mu)=  \mathbb{E}_{\pi_{\theta}}[ \nabla_{\theta}\log \pi_{\theta}(y) - \log \mu(y)]^2/ 2$. This is because
\begin{align}
\nabla_{\theta} D_{\text{KL}}(\pi||\mu)
&= \nabla_{\theta} \sum_y \pi_{\theta}(y) \cdot (\log \pi_{\theta}(y) - \log \mu(y)) \\ &=\sum_y \nabla_{\theta} \pi_{\theta}(y) \cdot (\log \pi_{{\theta}}(y) - \log \mu(y)) + \sum_y\nabla_{\theta} \pi_{\theta}(y) \nonumber \\
&=\sum_y  \pi_{{\theta}}(y) \frac{\nabla_{\theta} \pi_{\theta}(y)}{ \pi_{{\theta}}(y)} \cdot (\log \pi_{{\theta}}(y) - \log \mu(y)) + \nabla_{\theta} \sum_y\pi_{\theta}(y) \nonumber \\
&=\mathbb{E}_{\pi_{{\theta}}} [(\log \pi_{{\theta}}(y) - \log \mu(y)) \cdot \nabla_{\theta} (\log \pi_{\theta}(y) - \log \mu(y))]  \nonumber \\
&= \mathbb{E}_{\pi_{{\theta}}} [\nabla_{\theta}(\log \pi_{\theta}(y) - \log \mu(y))^2] / 2. \label{eq: reverse_kl_result}
\end{align}
The first equation holds because of the definition of KL divergence. The second equation holds due to applying the product rule of differentiation. The third equation holds due to simple algebra, and the second term will then vanish because of the sum of the probabilities. The fourth equation holds because of simple algebra.
% Therefore
% According to Proposition \ref{theorem:rsp}, when $R$ is the backward KL divergence, RSPO converges to the same regularized Nash Equilibrium as IR-SP:
% \begin{align}
% \mathbb{P}(\pi \succ \pi') - \mathbb{E}_{y \sim \pi}[g(y; \pi, \mu)] + \mathbb{E}_{y' \sim \pi'}[g(y'; \pi', \mu)],
% \end{align}
% where regularization function $g$ is $\log \pi(y) - \log \mu(y)$, and thus $\mathbb{E}_{y \sim \pi}[g(y; \pi, \mu)]$ is backward KL divergence as well.
\end{proof}

% \textcolor{red}{
% Seongho: It seems Eq. 64 is using 
% }

% \begin{align}
%     \mathbb{E}_{\pi_t}[ \nabla_{\theta}\log \pi_{\theta}(y) - \log \mu(y)]^2/ 2 = \nabla_{\theta}\mathbb{E}_{\pi_t}[ \log \pi_{\theta}(y) - \log \mu(y)]^2/ 2.
% \end{align}

% \textcolor{red}{However,}
% \begin{align}
%     \nabla_{\theta} D_{\text{KL}}(\pi||\mu)&=  \mathbb{E}_{\pi_{\theta}}[ \nabla_{\theta}\log \pi_{\theta}(y) - \log \mu(y)]^2/ 2 &\neq \nabla_{\theta}\mathbb{E}_{\pi_{\theta}}[ \log \pi_{\theta}(y) - \log \mu(y)]^2/ 2, \\
%     D_{\text{KL}}(\pi||\mu) &= \int \tfrac{1}{2}\mathbb{E}_{\pi_\theta}[ \nabla_{\theta}\log \pi_{\theta}(y) - \log \mu(y)]^2d\theta &\neq \int \tfrac{1}{2}\mathbb{E}_{\pi_t}[ \nabla_{\theta}\log \pi_{\theta}(y) - \log \mu(y)]^2d\theta.
% \end{align}
% \textcolor{red}{I think the derivation from Eq. 91 to Eq. 95 is correct. However, I am not sure if Eq. 63-66 is compatible with Eq. 95.}

\subsection{Proof of Proposition \ref{theo:forward_kl}}
% Since $\nabla_{\theta} D_{\text{KL}}(\mu||\pi)=  \mathbb{E}_{\pi_{\theta}}[\nabla_{\theta} {\mu(y)}/{\pi_{\theta}(y)}]$, RSPO with $R=D_{\text{KL}}(\mu||\pi)$ converges to the Nash Equilibrium of regularized preference as
% \begin{align}
% \mathbb{P}(\pi \succ \pi') -\tau D_{\text{B}}(\pi||\mu) +\tau D_{\text{B}}(\pi'||\mu),
% \label{eq:forward_kl_rsp}
% \end{align}
% where $D_{\text{B}}(\pi_{\theta}||\mu)=\sum_{y \in \mathcal{Y}} \pi_{\theta}(y) \sqrt{\frac{\mu(u)}{\pi_{\theta}(y)}}$ is the Bhattacharyya Distance \citep{bhattacharyya1946measure}.
\begin{proof}
$\pi$ is parametrized by $\theta$, then $\nabla_{\theta} D_{\text{KL}}(\mu||\pi)=  \mathbb{E}_{\mu}[\nabla_{\theta} \frac{\mu(y)}{\pi_{\theta}(y)}]$ because
\begin{align}
\nabla_{\theta} D_{\text{KL}}(\mu||\pi)
&= \nabla_{\theta} \sum_y \mu(y) \cdot (\log \mu(y) - \log \pi_{\theta}(y)) \\&= -\sum_y \mu(y) \nabla_{\theta} \log \pi_{\theta}(y) = -\sum_y \pi_{\theta}(y) \frac{\mu(y)}{\pi_{\theta}(y)} \nabla_{\theta} \log \pi_{\theta}(y) \nonumber \\
&= -\mathbb{E}_{\pi_{\theta}}\bigg[\frac{\mu(y) \nabla_{\theta} \log \pi_{\theta}(y)}{\pi_{\theta}(y)}\bigg] =  -\mathbb{E}_{\pi_{\theta}}\bigg[\frac{\mu(y) \nabla_{\theta} \pi_{\theta}(y)}{\pi_{\theta}(y)^2}\bigg] =   \mathbb{E}_{\pi_{\theta}}\bigg[\nabla_{\theta} \frac{\mu(y)}{\pi_{\theta}(y)}\bigg].
\end{align}
The first three equations hold due to the definition of forward KL divergence and simple algebra. The fourth equation comes from rewriting the forward KL following the first three equations. The fifth equation holds by taking the derivative of $\log \pi_\theta$. The sixth equation holds since $\frac{ \nabla_{\theta} \pi_{\theta}(y)}{\pi_{\theta}(y)^2}= \nabla_{\theta} \frac{-1}{\pi_{\theta}(y)}$.
% According to Proposition \ref{theorem:rsp}, when $R$ is the forward KL divergence, RSPO converges to the same regularized Nash Equilibrium as IR-SP:
% \begin{align}
% \mathbb{P}(\pi \succ \pi') - \mathbb{E}_{y \sim \pi}[g(y; \pi, \mu)] + \mathbb{E}_{y' \sim \pi'}[g(y'; \pi', \mu)],
% \end{align}
% where $g = \sqrt{\frac{\mu(u)}{\pi(y)}}$, and $\mathbb{E}_{y \sim \pi}[g(y; \pi, \mu)] = \sum_{y \in \mathcal{Y}} \pi(y) \sqrt{\frac{\mu(u)}{\pi(y)}} =  D_{\text{B}}(\pi||\mu)$ is the Bhattacharyya Distance.
\end{proof}

\subsection{Proof of Proposition \ref{theo:chisquare}}
% Since $\nabla_\theta D_{\chi^2}(\pi_\theta||\mu)=\mathbb{E}_{\pi_\theta}\left[{\nabla_\theta \pi_\theta(y)}/{\mu(y)}\right]$,
% so RSPO with $R=D_{\chi^2}(\mu||\pi)$ converges to the Nash Equilibrium of regularized preference as
% \begin{align}
% \mathbb{P}(\pi \succ \pi') -\tau D_{f}(\pi||\mu) +\tau D_{f}(\pi'||\mu),
% \end{align}
% where $D_{f}$ is $f$-divergence with $f(r)=r^{3/2}$.
\begin{proof}
$\pi$ is parametrized by $\theta$, $\nabla_\theta D_{\chi^2}(\pi_\theta(y)||\mu(y))=\mathbb{E}_{\pi_\theta}\left[\frac{\nabla_\theta \pi_\theta(y)}{\mu(y)}\right]$ since
\begin{align}
D_{\chi^2}(\pi_\theta(y)||\mu(y))&=\frac{1}{2}\sum_y \left(\frac{\pi_\theta(y)}{\mu(y)}-1\right)^2 \mu(y) =\frac{1}{2}\sum_y \frac{\pi_\theta(y)^2 - 2\pi_\theta(y)\mu(y) + \mu(y)^2}{\mu(y)} \nonumber \\
&=\frac{1}{2}\sum_y \frac{\pi_\theta(y)^2}{\mu(y)}+ C(\mu) =\frac{1}{2}\mathbb{E}_{\pi_\theta(y)}\left[\frac{\pi_\theta(y)}{\mu(y)}\right]+C,
\end{align}
where $C(\mu)$ is independent to $\theta$. The first two equations hold according to the definition of Chi-squared divergence. The third equation holds by separating the terms only related to $\mu$ and the term related to $\pi_\theta$. The fourth equation holds by rewriting the summation as the expectation.
% Thus, according to Proposition \ref{theorem:rsp}, when $R$ is the Chi-Square divergence, RSPO converges to the same regularized Nash Equilibrium as IR-SP:
% \begin{align}
% \mathbb{P}(\pi \succ \pi') - \mathbb{E}_{y \sim \pi}[g(y; \pi, \mu)] + \mathbb{E}_{y' \sim \pi'}[g(y'; \pi', \mu)],
% \end{align}
% where $g = \sqrt{\frac{\pi(y)}{\mu(u)}}$, and $\mathbb{E}_{y \sim \pi}[g(y; \pi, \mu)] = \sum_{y \in \mathcal{Y}} \pi(y) \sqrt{\frac{\pi(y)}{\mu(u)}} =  D_{f}(\pi||\mu)$ is the $f$-Divergence with $f(r) = r^{3/2}$.
\end{proof}

\section{Additional Related Work}
\paragraph{Preference Optimization.} Large Language Models (LLMs) recently have obtained remarkable capabilities to accomplish a range of tasks \citep{jiang2023mistral,dubey2024LLaMA,deepseekai2025deepseekr1incentivizingreasoningcapability}, generating more desirable and helpful content following the user’s intention. One of the most important methods to align LLMs with human intentions is Reinforcement Learning from Human Feedback (RLHF), maximizing a preference-based reward penalized by a reverse KL regularization term of the LLM policy and a reference model \citep{christiano2017deep,ouyang2022training,rafailov2024direct,azar2024general,xiong2024iterative}. Since the reference model usually provides safer guidance for policy optimization \citep{munos2023nash}, this regularization is crucial in RLHF to prevent over-optimization, which has been extensively studied and extended beyond KL divergence  \citep{wang2023beyond,go2023aligning,huang2024correcting}. In this work, we instead study the regularization problems in self-play alignment.

% \textbf{Offline RLHF with general divergence for regularization.} The use of general divergence-based regularization has been explored in the context of offline alignment. $f$-DPO \citep{wang2023beyond} extends Direct Preference Optimization \citep{rafailov2024direct} from reverse KL regularization to a broader class of $f$-divergences, but primarily demonstrates benefits in generation diversity. The specific effects of individual divergences—and their performance on widely-used benchmarks such as AlpacaEval—remain unexamined. $\chi$PO \citep{huang2024correcting} emphasizes the theoretical importance of $\chi^2$ divergence for uncertainty quantification. However, the role of regularization in online iterative preference optimization, particularly its empirical impact on standard benchmarks, has yet to be studied. \looseness=-1

\textbf{RLHF with General Preference Optimization (Self-Play Alignment).} \citet{azar2024general} introduced the first approach for optimizing LLM policy via general preference models. Nash-MD \citep{munos2023nash} pioneered the application of self-play to general preference optimization by framing it as a two-player game. Subsequent methods have either focused on learning the NE of the original unregularized game (e.g. \citep{swamy2024minimaximalist, wu2024self, rosset2024direct, wang2024magnetic}) or the NE of a reverse-KL-regularized preference optimization game (e.g. \citep{munos2023nash, calandriello2024human, zhang2024iterativenashpolicyoptimization}). In contrast, our work explores a broader class of divergence-based regularization techniques for self-play alignment.

Notably, our RSPO can generalize existing self-play methods. Unregularized self-play methods following the preference-based MWU can all be generalized by $\mathcal{L}_{\text{RSPO}}$ without external regularization, and thus can be regularized by simply adding regularization term to the loss functions. Based on the same exponential update rule as in SPPO, SPO \citep{swamy2024minimaximalist} is equivalent to updating policy with the loss in \Cref{eq:sppo_sp}. Magnetic Policy Optimization \citep{wang2024magnetic}, despite incorporating regularization in the policy update, periodically updates $\mu=\pi_t$. Consequently, it inherently follows MWU while incorporating multiple policy updates within each iteration, following 

% Notably, self-play alignment methods requiring a pre-trained or external preference model for policy optimization is called (preference) model-based self-play \citep{munos2023nash}. Model-free self-play on the contrary directly optimize the policy based on the preference data. INPO and DNO are model-free methods \citep{zhang2024iterativenashpolicyoptimization,rosset2024direct}, where only bandit preference is provided in the dataset. To avoid estimating a preference model, these methods leverage similar idea of Direct Preference Optimization, subtracting the policy logits of preferred response by that of dis-preferred response to approximately conduct Mirror Descent or MWU. However, in this work, we investigate a general self-play policy optimization, similar to the original RLHF \citep{ouyang2022training} where a reward model is pre-trained for RL, we also assume to have a pre-trained or external preference model. And we empirically show that, a small external preference model is sufficient for improving the base model significantly with self-play.

\textbf{Online iterative RLHF.} Iterative alignment method incorporates a reliable reward or preference model—including self-play—functions as a self-improving framework by iteratively generating new data using models and optimizing policies based on this data \citep{schulman2017proximal, ouyang2022training, bai2022training, touvron2023LLaMA, dong2024rlhf}. Moreover, extending powerful offline methods such as DPO to iterative frameworks has led to significant performance gains \citep{xu2023some, liu2023statistical, viethoangtranduong, dong2024rlhf, calandriello2024human, pang2024iterative, xiong2024iterative, guo2024direct, tajwar2024preference, cen2024value, xie2024exploratory}. In contrast, our work investigates general preference optimization through self-play from a game-theoretic perspective, shifting the objective from conventional RL optimization to the computation of NE.\looseness=-1

\section{Additional Details}
In this section, we provide additional details of this paper, including the algorithm descriptions of self-play alignment methods, a summarizing table for generalizing existing methods, and our implementation of regularizations.
\subsection{Self-Play Alignment Algorithm} 
\Cref{alg:selfplay} shows the overall self-play alignment process. Note that we are sampling $K$ responses per each prompt and obtain pair-wise preferences amongst them for training.
\begin{algorithm}[H]
\caption{Self-Play Alignment}
\label{alg:selfplay}
\begin{algorithmic}
  \STATE \textbf{Input:} LLM $\pi_\theta$, preference model ${\mathbb{P}}$, number of iterations $T$, reference policy $\mu$, loss function for policy update conditioned on utility function $U$: $\mathcal{L}(\theta; U)$, sample size $K$.
  \STATE \textbf{Initialize:}  $\pi_0 =\mu$.   
  \FOR{$t \in [T]$}
    \STATE Sample prompts and responses: $x \sim \mathcal{X}$, $y_{1:K} \sim \pi_t$
    \STATE Get pair-wise preferences $u_{ij}={\mathbb{P}}(y_i \succ y_j),\ \forall i,j \in [K]$
    \STATE Update policy parameters $\theta= \arg \min_\theta \mathcal{L}(\theta; U)$, $U=[u_{ij}] \in \mathbb{R}^{K\times K}$
    \STATE $\pi_{t+1} = \pi_\theta$
  \ENDFOR
  \STATE \textbf{Output:} Last-iterate policy $\pi_T$.
\end{algorithmic}
\end{algorithm}
Specifically, the policy is first initialized as $\pi_0=\mu$. Then in each iteration $t$, the opponent is set to be the last-iterate policy $\pi_t$ (the reason why it is called self-play), and the responses are sampled from $\pi_t$ (Line 4). The pairwise preferences of the sampled responses are collected using the preference model $\mathbb{P}$ (Line 5). The policy parameters are updated by minimizing a specified loss function $\mathcal{L}(\theta; \mathbb{P})$ based on preferences over responses (Line 6). The loss function $\mathcal{L}(\theta; \mathbb{P})$ is dependent on the inherent online learning method. The main difference between these methods is the choice of loss function $\mathcal{L}(\theta; \mathbb{P})$ applied to the policy update. 

\subsection{Generalizing Existing Methods}
\label{section: generalising existing methods}
\Cref{table:main} shows how the existing methods of self-play alignment can be generalized without external regularization. The algorithms introduced below share the same loss structure as in \Cref{eq:RSPO}, while their differences present in the update direction $G$, baseline $B$ and the preference model.
\begin{table*}[h]
    
    \centering
    \begin{tabular}{c|c|c|c} 
    
\toprule
Loss & Update Direction ($G$) & Baseline ($B$) & Preference Model  \\ 
\midrule
% $\mathcal{L}_\text{SPO}$ \citep{swamy2024minimaximalist} & $\mathbb{P}(y \succ \pi_t)$ &  Est. &  $\mathbb{P}(y \succ y')$ \\
$\mathcal{L}_\text{SPPO}$ \citep{wu2024self} & $\mathbb{P}(y \succ \pi_t)$ &  $0.5$ &  $\mathbb{P}(y \succ y')$ \\
$\mathcal{L}_\text{OMD}$ \citep{munos2023nash} & $\mathbb{P}(y \succ \pi_t) - \tau \log \frac{\pi_t(y)}{\mu(y)}$ & Est. & $\mathbb{P}_{\tau}(y \succ y')$  \\
$\mathcal{L}_\text{Nash-MD}$ \citep{munos2023nash} &$\mathbb{P}^{\mu}(y \succ \pi_t) - \tau \log \frac{\pi_t(y)}{\mu(y)}$ & $0.5$ & $\mathbb{P}_{\tau}(y \succ y')$  \\
 % $\mathcal{L}_\text{RPM}$ & $\mathbb{P}_{\tau}(y \succ \pi_t)$ &  Est.  & $\mathbb{P}_{\tau}(y \succ y')$ \\
% \midrule
% $\mathcal{L}_\text{MaxEnt} := \mathcal{L}_\text{MSE} + \tau' \mathcal{H}(\pi_{\theta}) $ & $\mathbb{P}$ &  $0.5$ &  Entropy  \\
% $\mathcal{L}_\text{reverse} := \mathcal{L}_\text{MSE} + \tau' D_{\text{KL}}(\pi_{\theta}||\mu) $  & $\mathbb{P}$ &  $0.5$ & Reverse KL  \\
% $\mathcal{L}_\text{forward} := \mathcal{L}_\text{MSE} + \tau' D_{\text{KL}}(\mu||\pi_{\theta})$  & $\mathbb{P}$ &  $0.5$ & Forward KL  \\
\bottomrule
\end{tabular}
\caption{Self-play losses $\mathcal{L}_{\text{RSPO}}$ generalizes different self-play policy optimization methods. $\mathbb{P}^{\mu}(y \succ \pi_t) = \mathbb{P}(y \succ \pi^{\mu}_t)$, $\pi^{\mu}_t$ is the geometric mixture of $\pi_t$ and $\mu$. We abbreviate the estimated baseline that reduce the variance of $G$ the most as est.. $\mathbb{P}_{\tau}(y \succ y') = \mathbb{P}(y \succ y') -\tau \log \frac{\pi_{\theta}(y)}{\mu(y)} + \tau \log \frac{\pi'(y')}{\mu(y')}$ is the regularized preference model.}
\label{table:main}
\end{table*}

\subsection{Implementation of Regularization}
\label{sec:implement_reg}
In practice, accurately estimating the gradient of the regularizer is essential, as many commonly used divergence measures are defined as expectations over $\pi_\theta$. The estimation of divergences has been extensively studied and widely applied in various domains \citep{rubenstein2019practical}. For completeness, in this section, we introduce the regularization methods investigated in this study, including Reverse KL, Forward KL, and Chi-Square Divergence.

We begin by deriving the estimation of the Reverse KL divergence based on the following proposition.
\begin{proposition}
Reverse KL divergence satisfies:
\begin{align}
\nabla_\theta D_{\textit{KL}}(\pi_\theta||\mu)=\mathbb{E}_{y \sim \pi_\theta}[{\nabla_\theta (\log \pi_\theta(y) - \log \mu(y))^2}].
\end{align}
\label{theo:reverse_kl}
\end{proposition}
\vspace{-.5cm}
According to Proposition \ref{theo:reverse_kl}, we can estimate the divergence with $\mathbb{E}_{y \sim \pi_\theta}[{(\log \pi_\theta(y) - \log \mu(y))^2}]$.

We employ two distinct approaches to estimate the forward KL divergence. The first method utilizes importance sampling, referred to as IS-For. KL, and is derived based on the following proposition.
\begin{proposition} The gradient of forward KL divergence satisfies that 
\begin{align}
\nabla_{\theta} D_{\text{KL}}(\mu||\pi_\theta)=  \mathbb{E}_{y \sim \pi_{\theta}}[\nabla_{\theta} {\mu(y)}/{\pi_{\theta}(y)}].
\end{align}
\label{theo:forward_kl}
\vspace{-.5em}
\end{proposition}
Therefore, we can estimate the forward KL divergence by leveraging the expectation $\mathbb{E}_{y \sim \pi_{\theta}}[{\mu(y)}/{\pi_{\theta}(y)}]$ to estimate the forward KL. Notably, to mitigate the risk of gradient explosion, we apply gradient clipping with a maximum value of $10$.

The second method for forward KL is a direct estimation of $D_{\text{KL}}(\mu||\pi_\theta)$. To achieve this, we resample responses from the reference policy $\mu$ using the same prompts from the training dataset, constructing a reference dataset. The KL divergence is then estimated directly based on its definition by uniformly drawing samples from this reference dataset. A key advantage of this approach is that it eliminates the need for importance sampling, as each policy update iteration only requires samples from $\pi_t$.
% \vspace{-.1cm}

Similarly, we estimate the Chi-Square divergence using $\mathbb{E}_{y \sim \pi_\theta}\left[{\pi_\theta(y)}/{\mu(y)}\right]$, based on the following proposition. Due to the presence of the ratio term, Chi-Square divergence estimation also necessitates gradient clipping to prevent instability, for which we set a clip value of $10$.
\begin{proposition} Chi-Square divergence has gradient \begin{align}
\nabla_\theta D_{\chi^2}(\pi_\theta||\mu)=\mathbb{E}_{y \sim \pi_\theta}\left[{\nabla_\theta \pi_\theta(y)}/{\mu(y)}\right].
\end{align}
\label{theo:chisquare}
\end{proposition}
\vspace{-.5cm}
We also explore the linear combination of different regularization functions to leverage their complementary effects, as in offline RLHF \citep{huang2024correcting}. The previously established propositions for estimating divergences can still be used in the combined regularization method.
% As demonstrated in our experiments (Section \ref{sec:main_reg}), the combination of reverse and forward KL divergences yields the best performance.

Apart from the flexibility and simplicity of applying different regularization methods, RSPO can generalize existing self-play methods, including the unregularized ones, which enables regularizing off-the-shelf self-play methods in practice with \textit{no change} on their original loss functions or hyperparameters, directly adding an external regularization term to their loss functions.

We then provide the hyperparameters of regularization temperature for each regularizer in our experiments:
\begin{table}[h]
\centering
\begin{tabular}{|l|c|}
\hline
\textbf{Divergence} & \textbf{Parameter(s)} \\
\hline
Reverse KL (Rev. KL) & $\lambda = 0.5$ \\
Forward K (For. KL) & $\lambda = 1.0$ \\
Chi-Squared ($\chi^2$) & $\lambda = 0.1$ \\
Importance-Sampling Forward KL (IS-For.) & $\lambda = 0.1$ \\
Forward and Reverse KL (IS-For.+Rev. KL) & $\lambda_1 = 0.1$, $\lambda_2 = 0.5$ \\
\hline
\end{tabular}
\vspace{1em}
\caption{Divergences and their corresponding $\lambda$ parameters.}
\label{tab:divergences}
\end{table}

\section{Additional Experiments}
In this section, we provide additional experiments, including two synthetic motivating examples and additional results on language tasks.
\subsection{Regularization in Game Solving}
\label{append:reg_game}

\begin{figure}
\vspace{-3em}
  \centering    \includegraphics[width=.5\linewidth]{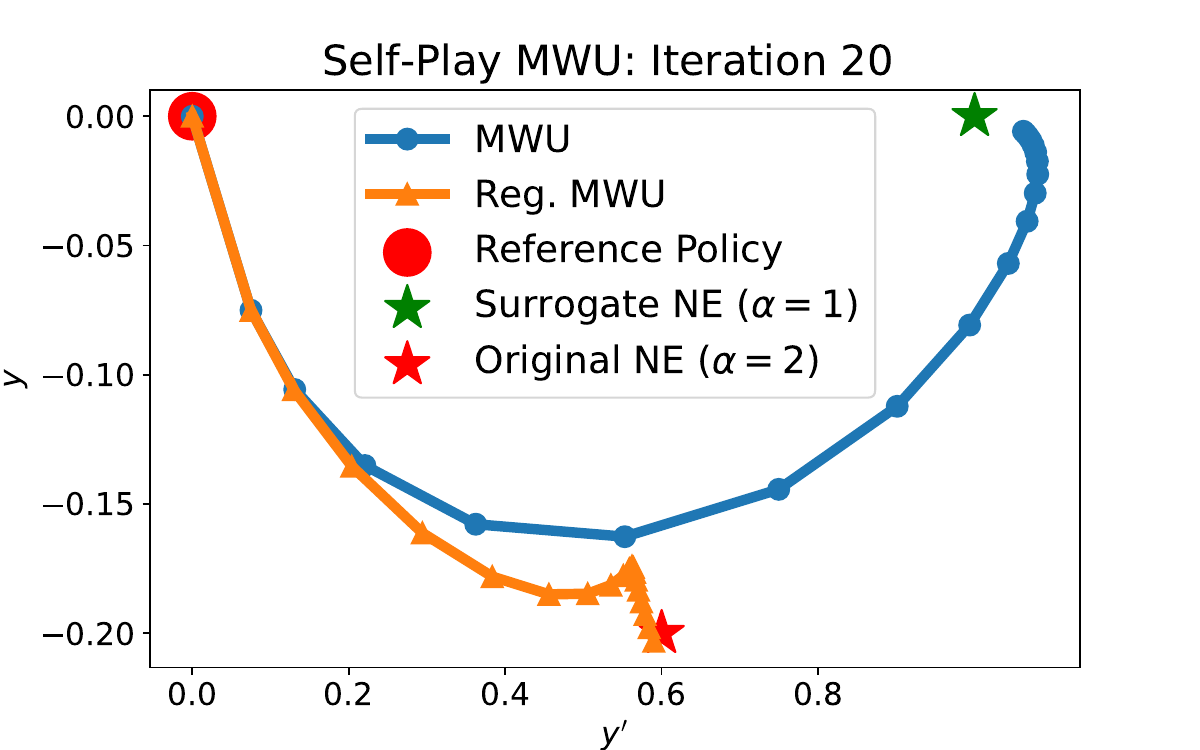}
    \caption{Motivating Example: 20 iterations of MWU and regularized MWU with the same learning rate to solve saddle point problem $ \max_y \min_{y'} f(y, y', \alpha)$, where $f(y, y'; \alpha)=\frac{\alpha}{2}{y'}^2 + (y'-1)(y-1) - \frac{\alpha}{2}y^2$, first introduced in \citep{sokota2022unified}. We assume that we only have access to a misspecified (surrogate) preference $f(y, y'; \alpha=1)$, while the ground truth human preference is $f(y, y'; \alpha=2)$. }
    \label{fig:simpletoy}
    \vspace{-.3cm}
\end{figure}

The regularization in the preference model is not used in all game-theoretic self-play methods. Here we investigate the necessity of regularization and offer a motivating example in Figure \ref{fig:simpletoy}, a saddle point solving problem $\min_x \max_y \frac{\alpha}{2} x^2 + (x-1)(y-1) - \frac{\alpha}{2}  y^2$. There exists a reference point as the initial values of $x$ and $y$. We assume that both reference point and the Nash Equilibrium (NE) of the surrogate preference model (Surrogate NE) are close to the original NE but on different sides of the original NE.

Typically, the surrogate preference/reward models are not positively related to the reference policy. Thus, it is a reasonable abstracted example of NLHF by treating the reference point as reference policy and surrogate NE as the optimal policy obtained by optimizing the surrogate preference/reward. The results of the $20$ iterations self-play MWU with an early stopping show that regularization can be used to prevent reward over-optimization (reaching surrogate NE). A well-tuned regularization leads to faster convergence to the unknown original NE. Thus, regularization can be effective in preventing over-optimization in self-play.

\subsection{Diversity on 2D Example}
\label{append:2d_diversity}
We offer an analysis of our method compared to unregularized self-play (SPPO) on a 2D example in Figure \ref{fig:RSPO_toy}. The area with a darker color is assigned a higher reward value. We use the preference defined by the $L^2$ norm between two actions. We also set the reference policy to be uniform. According to the figure, the unregularized method tends to converge to a single point on the manifold of the large reward. While regularized methods have diverse sampled actions.    
\begin{figure*}[h]
    \centering
    \includegraphics[width=\textwidth]{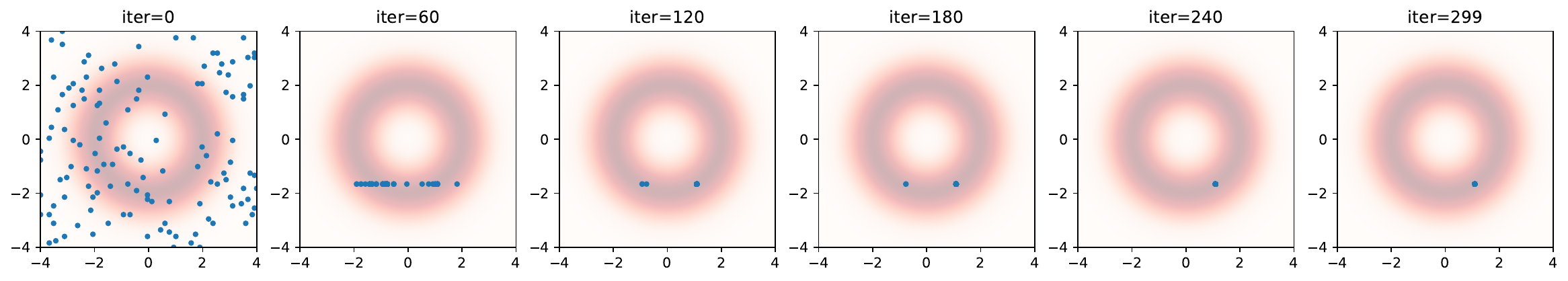}
    \includegraphics[width=\linewidth]{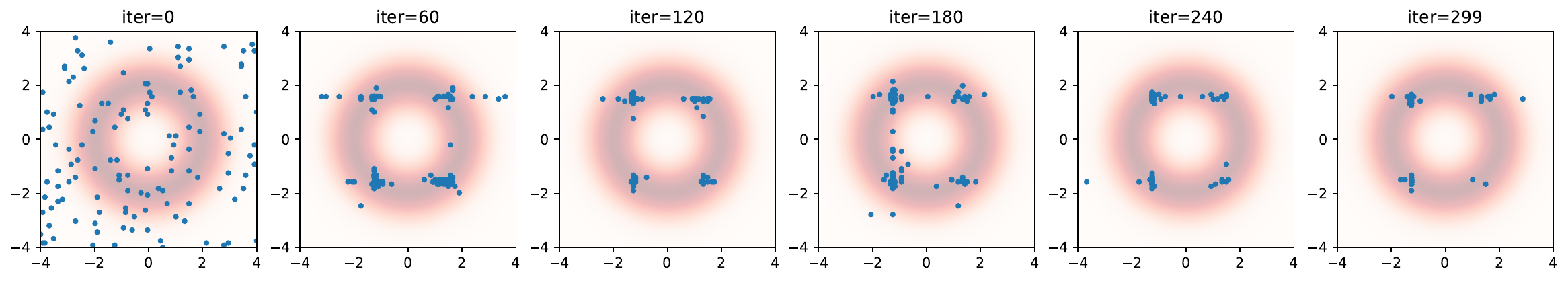}
    \caption{Samples in a 2D example of different iterations of SPPO (top) and RSPO (bottom) with external forward KL regularization to a uniform random reference policy. SPPO added simple external regularization that can generate multi-modal policies.}
    \label{fig:RSPO_toy}
    \vspace{-.6cm}
\end{figure*} 

\subsection{More Results on AlpacaEval-2.0 and PairRM}
\label{section: more results on alpacaeval-2.0}
In \Cref{fig:RSPO_all_reg_append} and \Cref{tab:alpacaeval}, we present further results of RSPO evaluated using AlpacaEval. As presented in \Cref{fig:RSPO_all_reg_append}, mixed regularization of the forward and reverse KL resulted in the best performance, while its average response length did not exceed that of reverse KL-only regularization. When compared to various other well-known baselines including GPT-4 and Claude, RSPO-trained model initialized from Mistral-7B shows notable performance, outperforming GPT-4 0314 and LLaMA 3 70B Instruct in LCWR. When response lengths are ignored, our RSPO-trained 7B model even outperforms Claude 3 Opus.
\begin{figure}[H]
% \vspace{-1cm}
    \centering
    \includegraphics[width=.48\linewidth]{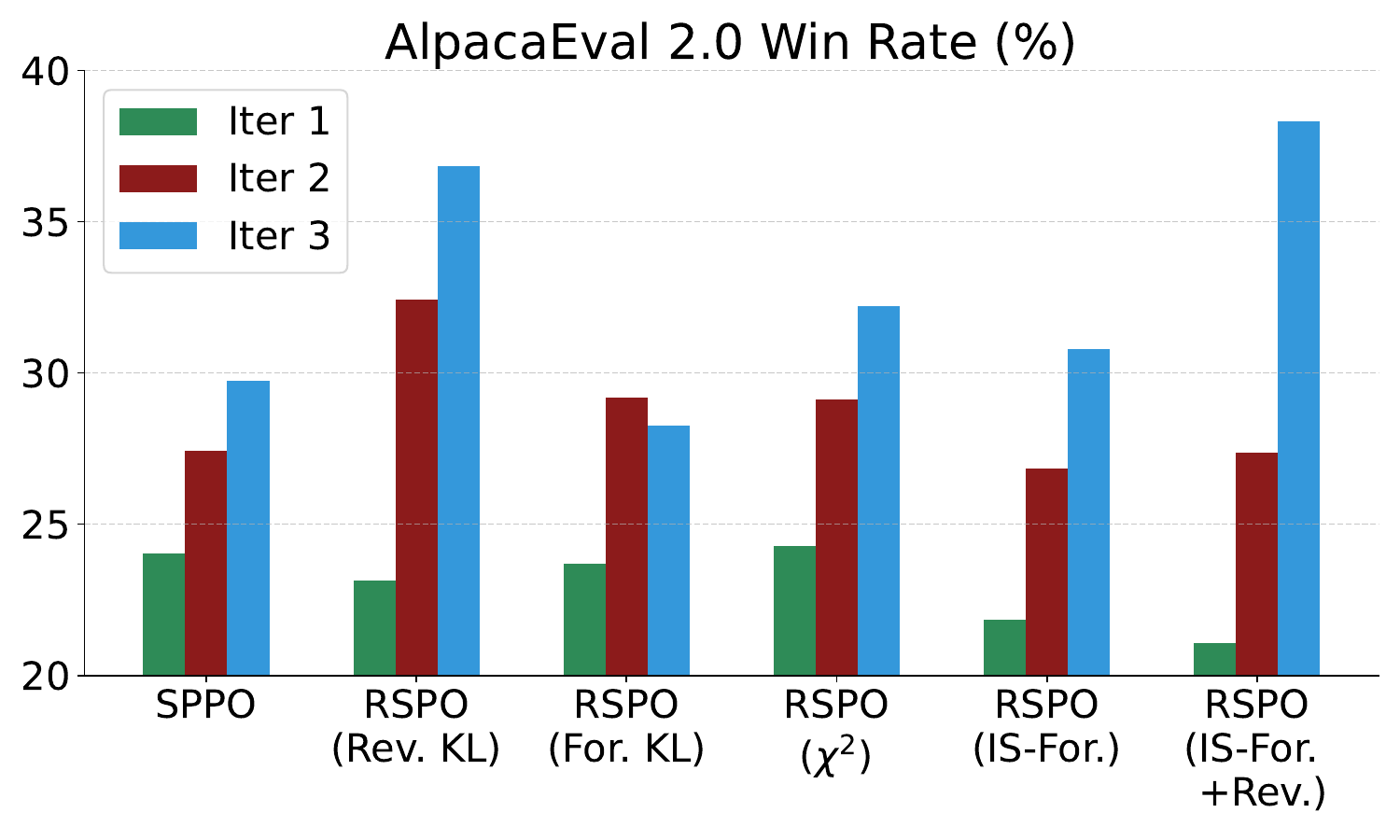}
    \includegraphics[width=.48\linewidth]{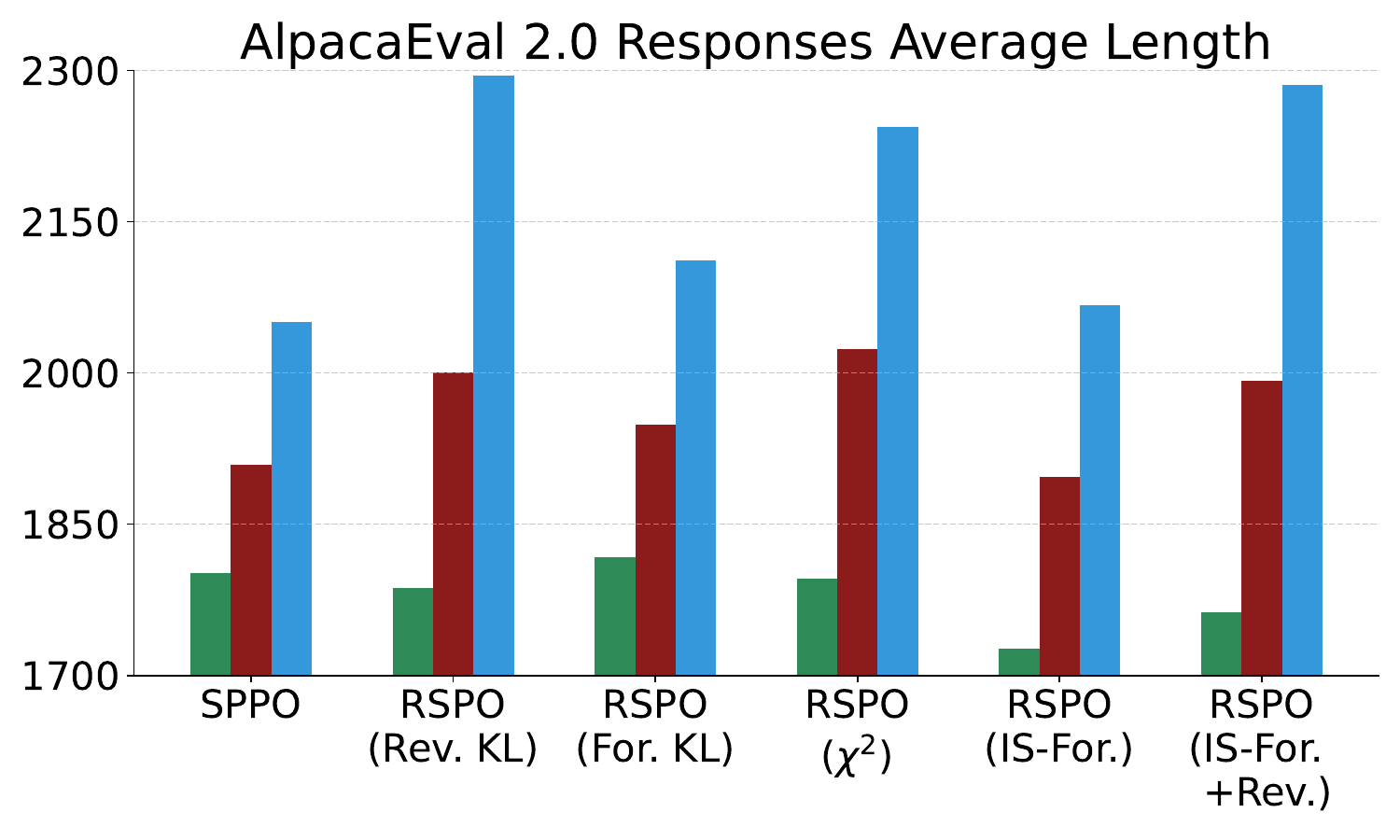}
    \caption{Win rates and the average length of SPPO and RSPO with different regularization methods. From left to right, regularization methods: Reverse KL, Forward KL, Chi-Squared, Importance-Sampling Forward KL, Importance-Sampling Forward, and Reverse KL linear combination.}
    \label{fig:RSPO_all_reg_append}
    \vspace{-.6cm}
\end{figure}

\begin{table}[h]
    \centering
    \caption{\textbf{Left:} AlpacaEval-2.0 performance of \textbf{RSPO} with different regularization temperatures. \textbf{Right:}  AlpacaEval-2.0 performance comparison with popular models. Our model, Mistral-7B-RSPO Iter3, outperforms GPT-4 0314 and LLaMA 3 70B Instruct in LCWR. When only win rate is considered, our model even outperforms Claude 3 Opus. }

    \resizebox{.55\textwidth}{!}{
    \begin{tabular}{l|c|c|c|c}
    \toprule
    \textbf{Regularization Temperature} & \textbf{Iter} & \textbf{LCWR (\%)} & \textbf{WR (\%)} & \textbf{Length} \\
    \midrule
    forward: 0.1 reverse: 0.5 & 1 & 23.16 & 21.06 & 1763 \\
    forward: 0.1 reverse: 0.5 & 2 & 27.91 & 27.38 & 1992 \\
    forward: 0.1 reverse: 0.5 & 3 & \textbf{35.44} & 38.31 & 2286 \\
    forward: 0.01 reverse: 0.5 & 1 & 24.63 & 22.57 & 1793 \\
    forward: 0.01 reverse: 0.5 & 2 & 28.21 & 28.56 & 2006 \\
    forward: 0.01 reverse: 0.5 & 3 & 32.24 & 36.77 & 2411 \\
    \bottomrule
    \end{tabular}}
    \resizebox{.42\textwidth}{!}{
    \begin{tabular}{lcc}
        \toprule
        \begin{tabular}{c}
        \multirow{2}{*}{Model}
        \end{tabular} & \multicolumn{2}{c}{AlpacaEval 2.0} \\
        & LC. Win Rate & Win Rate \\
        \midrule
        GPT-4 Turbo & 50.0 & 50.0 \\
        Claude 3 Opus & 40.5 & 29.1 \\
        \rowcolor{blue!15} Mistral-7B-RSPO Iter3 & 35.44 & 38.31 \\
        GPT-4 0314 & 35.3 & 22.1 \\
        LLaMA 3 70B Instruct & 34.4 & 33.2 \\
        GPT-4 0613 & 30.2 & 15.8 \\
        Mistral Medium & 28.6 & 21.9 \\ \rowcolor{gray!15} Mistral-7B-SPPO Iter3 & 28.5 & 31.0 \\
        \bottomrule
    \end{tabular}}
    \vspace{.5em}
    \label{tab:alpacaeval}
    \vspace{-1em}
\end{table}

\begin{table}[h!]
    \centering
    \caption{\textbf{Left:} \textbf{SPPO replication} Iteration-wise LCWR, WR, and Length results. Overoptimization exists according to the results. \textbf{Right:} Pairwise win rate of RSPO on Ultrafeedback validation set rated by pairRM. RSPO has higher win rates against all the baselines.}

    \resizebox{.38\textwidth}{!}{
    \begin{tabular}{cccc}
    \toprule
    \textbf{Iter} & \textbf{LCWR (\%)} & \textbf{WR (\%)} & \textbf{Length} \\
    \midrule
    1 & 26.36 & 24.04 & 1802 \\
    2 & 28.38 & 27.43 & 1909 \\
    3 & 29.17 & 29.75 & 2051 \\
    4 & 28.45 & 30.20 & 2257 \\
    5 & 27.93 & 30.11 & 2301 \\
    6 & 28.03 & 30.99 & 2435 \\
    7 & 25.46 & 28.25 & 2471 \\
    8 & 22.94 & 28.26 & 2691 \\
    9 & 24.47 & 28.57 & 3402 \\
    \bottomrule
    \end{tabular}}
    \resizebox{.6\textwidth}{!}{
    \begin{tabular}{lc}
    \toprule
    Methods & \textbf{RSPO (IS-For.+Rev.) Iter3} Win Rate \\
    \midrule
    \textbf{Snorkel (Iterative-DPO)} & 0.55 \\ \textbf{SPPO Iter3}  & 0.57 \\ 
    \textbf{SimPO} & 0.50 \\
    \bottomrule
    \end{tabular}}
    \vspace{.5em}
    \label{tab:sppo_replicate}
\end{table}

\begin{table}[h!]
\centering
\caption{\textbf{RSPO on LLaMA3-8B-Instruct:} Comparison of SPPO and RSPO iterations on AlpacaEval LCWR and Avg. Len.}
\begin{tabular}{lccc}
\hline
\textbf{Base: LLaMA3-8B-Instruct} & \textbf{AlpacaEval LCWR} & \textbf{Avg. Len} \\
\hline
SPPO-Iter1 & 31.73 & 1962 \\
SPPO-Iter2 & 35.15 & 2021 \\
SPPO-Iter3 & 38.77 & 2066 \\
\hline
RSPO (Rev. KL, 0.5)-Iter1 & 31.04 & 2100 \\
RSPO (Rev. KL, 0.5)-Iter2 & 35.89 & 2289 \\
\textbf{RSPO (Rev. KL, 0.5)-Iter3} & \textbf{43.66} & \textbf{2504} \\
\hline
\end{tabular}
\end{table}

\begin{table}[h!]
\centering
\caption{\textbf{RSPO on Gemma-2-2B-IT:} Comparison of SPPO and RSPO iterations on AlpacaEval LCWR and Avg. Len for Gemma-2-2B-IT.}
\begin{tabular}{lccc}
\hline
\textbf{Base: Gemma-2-2B-IT} & \textbf{AlpacaEval LCWR} & \textbf{Avg. Len} \\
\hline
SPPO-Iter1 & 43.68 & 2191 \\
SPPO-Iter2 & 49.24 & 2372 \\
SPPO-Iter3 & 50.54 & 2471 \\
\hline
RSPO (Rev. KL, 0.1)-Iter1 & 41.87 & 2213 \\
RSPO (Rev. KL, 0.1)-Iter2 & 48.94 & 2337 \\
RSPO (Rev. KL, 0.1)-Iter3 & 51.83 & 2443 \\
\hline
\end{tabular}
\end{table}

\begin{table}[h!]
\centering
\caption{\textbf{Regularization Effect of RSPO on LLaMA3-8B-Instruct:} Forward KL regularization effect on AlpacaEval LCWR and Avg. Len. Increasing the strength of For. KL regularization only brings slight reduction in average length.}
\label{tab:LLaMA_for}
\begin{tabular}{lcc}
\hline
\textbf{Method} & \textbf{AlpacaEval LCWR} & \textbf{Avg. Len} \\
\hline
SPPO-Iter1 & 31.73 & 1962 \\
SPPO-Iter2 & 35.15 & 2021 \\
SPPO-Iter3 & 38.77 & 2066 \\
\hline
RSPO (IS-For. KL, \textbf{0.01}, clip 10)-Iter1 & 30.87 & 1939 \\
RSPO (IS-For. KL, \textbf{0.01}, clip 10)-Iter2 & 33.58 & 1962 \\
RSPO (IS-For. KL, \textbf{0.01}, clip 10)-Iter3 & 36.80 & 1998 \\
\hline
RSPO (IS-For. KL, \textbf{0.1}, clip 10)-Iter1 & 26.72 & 1920 \\
RSPO (IS-For. KL, \textbf{0.1}, clip 10)-Iter2 & 32.22 & 1973 \\
RSPO (IS-For. KL, \textbf{0.1}, clip 10)-Iter3 & 36.64 & 1948 \\
\hline
\end{tabular}
\end{table}

\begin{table}[h!]
\centering
\caption{\textbf{Regularization Effect of RSPO on LLaMA3-8B-Instruct:} Due to only slight length reduction in length as shown in \Cref{tab:LLaMA_for}, the combination regularization of For. and Rev. KL shows less improvement compared to only regularized with Rev. KL. However, the performance is still superior compared to SPPO.}
\label{tab:LLaMA_all}
\begin{tabular}{lcc}
\hline
\textbf{Method} & \textbf{AlpacaEval LCWR} & \textbf{Avg. Len} \\
\hline
SPPO-Iter1 & 31.73 & 1962 \\
SPPO-Iter2 & 35.15 & 2021 \\
SPPO-Iter3 & 38.77 & 2066 \\
\hline
RSPO (Rev. KL, 0.5)-Iter1 & 31.04 & 2100 \\
RSPO (Rev. KL, 0.5)-Iter2 & 35.89 & 2289 \\
\textbf{RSPO (Rev. KL, 0.5)-Iter3} & \textbf{43.66} & \textbf{2504} \\
\hline
RSPO (IS-For. + Rev. KL)-Iter1 & 30.89 & 2100 \\
RSPO (IS-For. + Rev. KL)-Iter2 & 34.94 & 2308 \\
RSPO (IS-For. + Rev. KL)-Iter3 & 41.84 & 2465 \\
\hline
\end{tabular}

\end{table}

\begin{table}[t!]
\centering
\caption{\textbf{ArmoRM Evaluation.} Evaluation of diverse response quality aspects of fine-tuned Mistral-7B model on Ultrafeedback validation set using ArmoRM \cite{wang2024interpretable}. The combined application of forward and reverse KL regularization leads to  superior performance compared to either form of regularization applied independently.}
\resizebox{\textwidth}{!}{\begin{tabular}{@{}lccccc@{}}
\toprule
\textbf{Methods} & \textbf{Overall Score} & \textbf{Instruction Following} & \textbf{Truthfulness} & \textbf{Honesty} & \textbf{Helpfulness} \\
\midrule
Snorkel & 0.706 & 0.781 & 0.796 & 0.821 & 0.760 \\
SPPO & 0.716 & 0.798 & 0.812 & \cellcolor{gray!30} \textbf{0.836} & 0.771 \\
\midrule
RSPO ($\chi^2$, $\lambda = 0.1$) & 0.713 & 0.793 & 0.805 & 0.827 & 0.769 \\
RSPO (Rev. $\lambda = 0.5$) & 0.718 & 0.798 & 0.805 & 0.831 & \cellcolor{gray!30} \textbf{0.773} \\
RSPO (Rev. $\lambda = 1$) & 0.715 & 0.798 & 0.807 & 0.826 & 0.769 \\
RSPO (For. $\lambda = 0.1$) & 0.711 & 0.795 & 0.809 & 0.824 & 0.760 \\
RSPO (For. $\lambda = 0.5$) & 0.713 & 0.793 & 0.815 & 0.826 & 0.749 \\
\textbf{RSPO (For.+Rev.)} & \cellcolor{gray!30} \textbf{0.719} & \cellcolor{gray!30} \textbf{0.805} & \cellcolor{gray!30}\textbf{0.816} & 0.833 & 0.768 \\
\bottomrule
\end{tabular}}
\vspace{.3cm}
\label{table:armorm}
\vspace{-.2cm}
\end{table}

\section{Others}
In this section, we provide other details including additional related works, compute resources, societal impacts and limitations.

\subsection{Compute Resources}
\label{sec:compute}
We conduct experiments on $8 \times$A100 80GB for training and single A100 80GB for evaluation.

\subsection{Societal Impacts}
\label{sec:impacts}
This study introduces a novel framework for fine-tuning large language models through self-play, incorporating regularization toward a reference model. Ethical considerations may emerge if the reference model exhibits harmful behaviors, or if the preference model used for policy updates inadvertently assigns higher ratings to harmful outputs. However, drawing on prior research, we find no evidence that the proposed approach poses direct negative societal impacts.

\subsection{Limitations}
\label{append:limitations}
% The primary theoretical limitations of this work stem from underlying assumptions. The first concerns the on-policy reinforcement learning (RL) condition. As detailed in Assumption \ref{assumption:on-policy}, it is standard across self-play alignment methods to assume $y \sim \pi_\theta \Leftrightarrow y \sim \pi_t$, implying equivalence between the data-generating policy and the current policy. However, in practice, to better utilize sampled data, training typically involves multiple gradient updates using an effectively equivalent loss function—thereby introducing off-policy behavior. This deviation from theoretical assumptions can affect convergence guarantees to a Nash equilibrium. While this limitation is shared by all self-play alignment approaches, our method nevertheless demonstrates strong empirical performance.

A theoretical limitation lies in the nature of the regularization term $R$ which is required to be relatively convex with respect to entropy (Assumption \ref{assumption:reg}). Both reverse KL divergence and $\chi^2$ divergence satisfy this property, whereas forward KL divergence does not. This discrepancy is evident in performance metrics such as raw win rates. Interestingly, forward KL has a beneficial side effect of reducing response length. To leverage the length reduction and reconcile the decreasing win rate, we adopt a linear combination of forward and reverse KL divergences—an approach that not only satisfies the relative convexity condition but also exploits the complementary strengths of each to achieve improved control over response length while maintaining theoretical soundness.

\subsection{Related Work}
Incorporating regularization into equilibrium-finding procedures has been extensively studied and shown to improve convergence guarantees across a range of algorithmic game-theoretic settings \citep{facchinei2003finite,hennes2019neural,cen2021fast,perolat2021poincare,abe2022last,pattathil2023symmetric,abe2023adaptively,abe2024boosting}, including methods tailored to extensive-form games \citep{liu2022power}. However, these approaches have primarily been evaluated in small-scale problem settings, leaving limited evidence that they can be directly applied to or effectively scaled for large games such as those arising in large language model (LLM) preference optimization. Prior work that does scale to large games \citep{perolat2022mastering,jacob2022modeling} remains restricted to reverse-KL regularization. In contrast, our work scales reinforcement-learning–based preference optimization to large games while accommodating general forms of regularization. Moreover, closely related methods such as \citep{abe2023adaptively,abe2024boosting} implicitly constrain the regularizer to be linear in the policy, whereas our approach—grounded in the theory by \citet{sokota2022unified}—supports general relatively-convex regularizers.

\end{document}